\documentclass{article} 
\usepackage{iclr2027_conference,times}

\usepackage{amsmath,amsfonts,bm}

\def\eqref#1{equation~\ref{#1}}

\def\1{\bm{1}}

\DeclareMathAlphabet{\mathsfit}{\encodingdefault}{\sfdefault}{m}{sl}
\SetMathAlphabet{\mathsfit}{bold}{\encodingdefault}{\sfdefault}{bx}{n}

\usepackage{hyperref}
\usepackage{url}
\usepackage{multirow}
\usepackage{graphicx}
\usepackage{adjustbox}
\usepackage{adjustbox}
\usepackage[utf8]{inputenc}
\usepackage{newunicodechar}
\usepackage{algorithm}
\usepackage{algorithmic}
\usepackage{booktabs}
\usepackage{amsthm}
\usepackage{caption}
\usepackage{etoc}

\newtheorem{theorem}{Theorem}
\newtheorem{lemma}{Lemma}

\title{Learning Perturbation-Specific Representations through Conditional Energy Prior for Cellular Response Prediction}

\author{Yuchen Yan, Yatao Bian\thanks{Corresponding author.}  \\
Department of Computer Science, National University of Singapore \\
\texttt{yuchen.yan@u.nus.edu}, \texttt{ybian@nus.edu.sg} \\
}

\iclrfinalcopy
\begin{document}

\maketitle

\begin{abstract}
AI-based cellular perturbation prediction is essential to accelerate drug discovery and therapeutic identification. However, perturbation-specific signals are generally obscured by cellular complexity, making it difficult for existing methods to correctly capture perturbation effects. To address this problem, we propose PerturbEnergy, which introduces a conditional energy-based prior over perturbation representations. It assigns lower energy (higher probability) to representations compatible with the given perturbation and higher energy to others, thereby explicitly tilting the latent space toward perturbation-preferred regions. Besides, PerturbEnergy can be integrated into autoencoder architectures as we find its Evidence Lower Bound (ELBO) is equivalent to the ELBO of the complete autoencoder. We further prove that optimizing PerturbEnergy increases the lower bound on the conditional mutual information between representation and perturbation. Extensive experiments on four benchmark datasets demonstrate the effectiveness of our PerturbEnergy: it improves perturbation recognition especially in latent space with improvement of up to 83.2\%; It also achieves state-of-the-art performance for predicting top-50/100/200 differentially expressed gene(DEG) response and remaining competitive on highly variable gene(HVG) prediction. The code is provided: https://github.com/BlueWhaleLab/PerturbEnergy
\end{abstract}
\section{Introduction}
\label{sec:introduction}
Cellular perturbation prediction aims to forecast cell responses to drug and genetic interventions, supporting therapeutic target identification, drug discovery, and personalized medicine. While modern high-throughput CRISPR technologies~\cite{jinek2012programmable,gilbert2014genome,dixit2016perturb,replogle2020combinatorial} have dramatically increased the scale of single-cell transcriptomic profiling, experimentally screening the vast space of perturbations across diverse cellular contexts remains costly and labor-intensive. To address this challenge, AI-based virtual cells (AIVCs)~\cite{bunne2024build,resasco2012virtual} have emerged as an efficient approach for predicting cellular responses using deep learning models.

Existing methods generally follow a representation learning paradigm~\cite{jiang2026makes,russkikh2020style} where the perturbation effect is captured within the latent representation space. Some approaches use foundation models~\cite{cui2024scgpt,hao2024large,wang2024modeling} or specialized encoders~\cite{lotfollahi2019scgen,russkikh2020style,jiang2026makes,roohani2024predicting} to encode cells into latent representations and subsequently predict perturbation responses through generative decoding. Other studies develop more powerful architectures to directly model perturbation-induced transitions between control and perturbed cell representations. For example, STATE~\cite{adduri2026predicting} employs a transition Transformer to model perturbation-induced shifts between cell representations; CellFlow~\cite{klein2025cellflow} uses optimal transport to learn mappings between control and perturbed cell distributions; Squidiff~\cite{he2026squidiff} and PerturbDiff~\cite{yuan2026perturbdiff} characterize such representation transitions through diffusion processes.

\begin{figure*}[h!]
\centering
\includegraphics[width=0.95\textwidth]{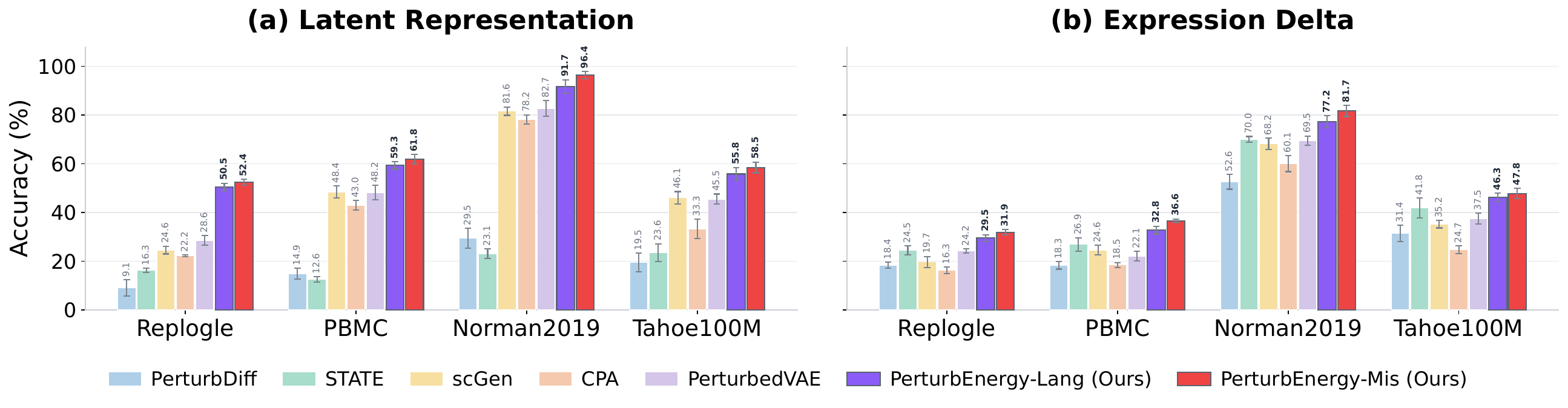}
\vspace{-1.2em}
\caption{Perturbation classification accuracy. Existing methods show limited perturbation recognition rate both in representation and in expression.}
\label{fig:perturbation_study}
\vspace{-1.6em}
\end{figure*}

Although these methods have demonstrated effectiveness across multiple datasets, they overlook an important issue that may substantially limit their performance. Accurate perturbation prediction depends on two key factors: (i) whether the perturbation effect is well described within the representation space, and (ii) whether the model is powerful enough to capture such perturbation effects. Existing methods mainly focus on the second factor where they develop sophisticated models and assume that such models can effectively extract perturbation effects from the latent space. However, they pay considerably less attention to the more fundamental question: \emph{how much perturbation-specific information is actually encoded in the learned representations for the model to exploit?} Addressing this issue is non-trivial due to the inherent complexity of cellular systems. Gene expression is influenced by a wide range of context-dependent factors, including cell type, donor background, and culture conditions. Moreover, biological heterogeneity~\cite{altschuler2010cellular} further complicates the problem: even cells under identical experimental conditions may differ in their physiological states or microenvironments, resulting in substantially different expression profiles. Collectively, all these variations can overwhelm perturbation-induced signals in the representation space, making it difficult for models to distinguish true perturbation effects from spurious and non-generalizable correlations. Although several approaches, such as stVAE~\cite{russkikh2020style} and PerturbedVAE~\cite{jiang2026makes}, claim to learn perturbation-specific representations, they remain largely heuristic: These methods typically rely on contrastive or adversarial objectives to implicitly regularize representation learning, providing only soft constraints without a theoretical guarantee that perturbation-specific signals can be well preserved. We conduct an empirical analysis examining such concern, as shown in Table~\ref{tab:perturbation_study} and Figure~\ref{fig:perturbation_study}. We find that existing methods, including those specifically designed to learn perturbation representations, exhibit limited perturbation distinctiveness both in the latent space and in their predicted expression profiles. These observations motivate the central question of this work: \emph{how can we directly enrich cellular representations with perturbation-specific information to improve cellular response prediction?}

In this paper, we propose a novel method named \textbf{PerturbEnergy}, which leverages an energy-based model to address this issue. The philosophy behind our method is that an energy function provides an effective mechanism for reshaping the distribution landscape~\cite{pang2021latent,zhou2022model}. Specifically, to enhance perturbation-specific information in the representation space, we explicitly formulate the latent prior in a normalized energy-based form. The energy model measures the compatibility between latent representations and perturbations, assigning lower energy (higher probability mass) to compatible pairs while assigning higher energy to others. Unlike existing methods that rely on implicit constraints, this design addresses the problem more directly by tilting the latent prior toward perturbation-preferred regions, thereby encouraging the learned representations to contain richer perturbation-specific information. Moreover, the proposed energy model can be easily integrated into autoencoder architectures. We show that the Evidence Lower Bound (ELBO) associated with the energy model can be optimized separately from the remaining architecture while remaining equivalent to the ELBO of the complete autoencoder. Our energy formulation also offers an additional computational advantage: the gradient of the normalized energy function can be derived analytically, while the samples required for optimization can be efficiently obtained through negative sampling or Langevin dynamics. Together, these properties allow the proposed energy model to be optimized independently of specific model architecture. Finally, we theoretically analyze the effectiveness of PerturbEnergy and prove that \emph{it can increase a lower bound on the mutual information between perturbations and latent representations.} To validate our method, we conduct comprehensive experiments on several benchmark datasets. The results show that PerturbEnergy substantially enhances perturbation-specific information, particularly in the latent representation space, with improvements of up to 83.2\%. It also consistently achieves state-of-the-art performance on top-50/100/200 differentially expressed gene (DEG) prediction and competitive performance on highly variable gene (HVG) prediction.

\section{Related Work}
\label{sec:related_work}
We summarize existing methods for cellular perturbation prediction from a representation learning perspective. Linear models~\cite{ahlmann2024deep} perform regression in the expression space or on representations from foundation models such as scGPT~\cite{cui2024scgpt} to predict responses. VAE-based methods model perturbation effects in latent space. For example, scGen~\cite{lotfollahi2019scgen} and Biolord~\cite{piran2024disentanglement} represent perturbation effects as additive latent factors; SAMS-VAE~\cite{bereket2023modelling} models perturbations using sparse latent representations. Other approaches design more advanced architectures to capture perturbation effects in the representation space. GEARS~\cite{roohani2024predicting} learns genetic perturbation representations using Graph Neural Networks (GNNs); STATE~\cite{adduri2026predicting} uses a Transformer to model perturbation-induced latent shifts; and diffusion or flow-based methods~\cite{he2026squidiff,yuan2026perturbdiff,klein2025cellflow,bunne2023learning} learn transitions between control and perturbed representations through diffusion or flow matching process. All these methods focus on learning perturbation effects from the representation space, but ignore whether the latent space contains sufficient perturbation-specific information. Although some studies claim to learn perturbation-awared representations, they remain intuitive: PerturbedVAE~\cite{jiang2026makes} employs a Contrastive object to disentangle perturbation representations from other factors; stVAE~\cite{russkikh2020style} and CPA~\cite{lotfollahi2023predicting} use adversarial training to promote such disentanglement. These auxiliary objectives provide only implicit constraints without a theoretical guarantee that perturbation-specific information is well preserved in the latent space.
\section{Method}
\subsection{Preliminary}

Before introducing our method, we first formalize the basic modeling framework and then explain why perturbation-specific information is important. Cellular perturbation prediction aims to predict gene expression induced by genetic or drug perturbations and we study it from a probabilistic perspective, assuming that each response is jointly determined by the cellular context and the applied perturbation. Formally, Let $x \in \mathbb{R}^{G}$ denotes the gene-expression profile of a perturbed cell, $a \in \mathcal{A}$ denotes its perturbation, and $c \in \mathcal{C}$ denotes its cellular context.\footnote{Here, $x$ denotes the perturbed expression profile and should rigorously be written as $x_{\mathrm{pert}}$. For simplicity, we use $x$ throughout the remainder of the paper unless otherwise specified.} Given observations $\mathcal{D}=\{(x_i,c_i,a_i)\}_{i=1}^{N}$, the goal is to learn a distribution $p(x|c,a)$ that predicts the cellular expression profile. To learn perturbation-specific effect, we decompose the latent representation into two components $z=(z_b,z_a)$, where $z_b$ captures the basal cellular state that is invariant to perturbation and $z_a$ captures perturbation-specific variation. The resulting joint distribution factorizes as:
\begin{equation}
p(x,z_b,z_a|c,a) = p(x|z_b,z_a) p(z_b,z_a|c,a),
\label{method_eq:join_dist}
\end{equation}
where $p(x|z_b,z_a)$ is the perturbation expression distribution conditioned on latent representations and can be parameterized by a generative decoder $p_{\theta}(x|z_b,z_a)$ with parameter $\theta$; $p(z_b,z_a|c,a)$ is the latent prior distribution, we adopt the variational inference like prior work~\cite{jiang2026makes} to approximate it as:
\begin{equation}
q_{\phi}(z_b,z_a|x_{ctr},c,a) = q_{\phi}(z_b|x_{ctr},c) q_{\phi}(z_a|x_{ctr},a,c),
\label{method_eq:variational_inference}
\end{equation}
where $x_{ctr}$ is the controlled (unperturbed) cell expression under the same context $c$. To model such a joint distribution, existing approaches~\cite{jiang2026makes,yuan2026perturbdiff,adduri2026predicting} either use VAE to learn perturbation-aware representations via implicit constraints or employ powerful decoders $p_{\theta}(x|z_b,z_a)$ to learn perturbation effects from the latent space. However, these methods do not directly address the concern: \textbf{does the latent space contain sufficient perturbation-specific information for the model to exploit?} Actually, the amount of perturbation-specific information retained in the latent space directly affects how easily the model can learn perturbation effects. More informative representations make generalizable perturbation patterns easier to identify and can therefore improve cellular response prediction. Our preliminary analysis supports this. By comparing CPA and PerturbedVAE in terms of perturbation classification accuracy and root mean square error (RMSE) on differentially expressed genes (DEGs), as shown in Table~\ref{tab:perturbation_study} and Table~\ref{tab:deg-rmse}, we observe that richer perturbation-specific information is associated with better perturbation prediction performance (CPA vs PerturbedVAE; Acc: 22.2\% vs.\ 28.6\%; Top50 RMSE: 0.3462 vs.\ 0.2130). However, we also find existing methods still capture limited perturbation-specific information within their latent space (less than 50\% Acc on Replogle, PBMC and Tahoe100M, as shown in Figure~\ref{fig:perturbation_study}). These observations motivate the central question of this work: \emph{how can we explicitly enrich latent representations with perturbation-specific information to improve cellular response prediction?}

\begin{figure*}[t]
\centering
\includegraphics[width=\textwidth]{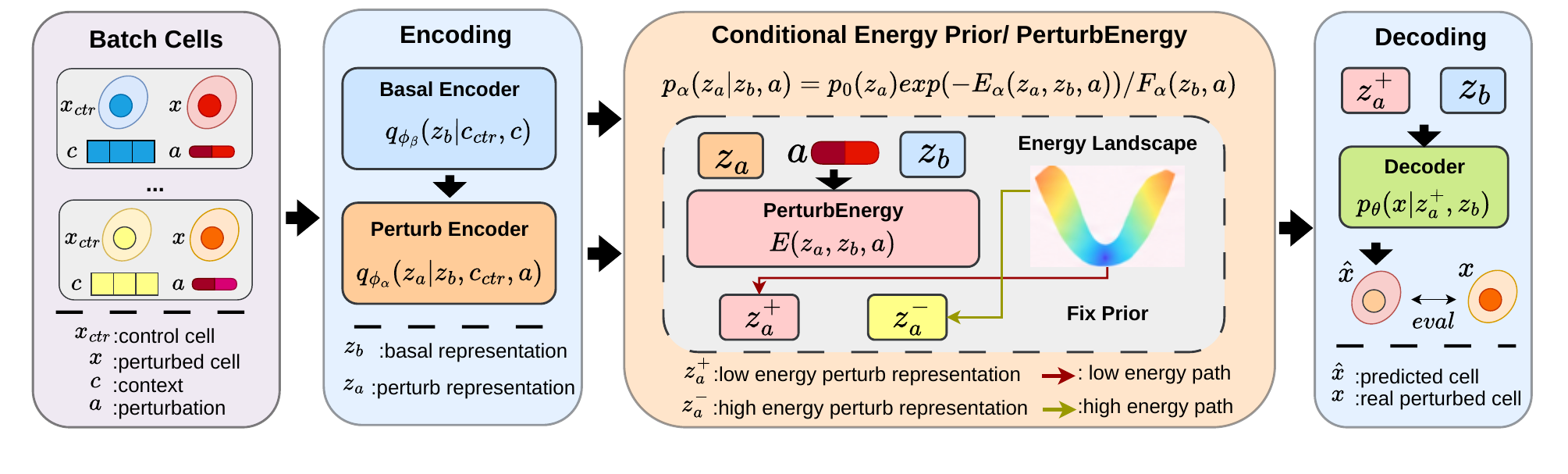}
\vspace{-1.6em}
\caption{Our PerturbEnergy framework. It reshapes the latent prior to perturbation preferred regions by assigning low energy to perturbation-compatible representations. Thus we can sample the representation with more perturbation-specific information for the decoder.}
\label{fig:framework}
\vspace{-1.2em}
\end{figure*}

\subsection{Conditional Energy-Based Perturbation Prior and Energy-ELBO}
To address the above concern, we propose \textbf{PerturbEnergy}, which introduces a conditional energy-based prior over the perturbation latent space. The framework is shown in Figure~\ref{fig:framework}. Our design is motivated by the observation that an energy function can serve as a flexible mechanism for revising a target distribution~\cite{pang2021latent,zhou2022model}. Specifically, we factorize the joint latent prior in Eq.~\eqref{method_eq:join_dist} as:
\begin{equation}
p(z_b,z_a|c,a) = p_{\beta}(z_b|c) p_{\alpha}(z_a|z_b,a),
\label{method_eq:latent_variable_model}
\end{equation}
where $p_{\beta}(z_b|c)$ denotes the marginal prior over the basal representation, and $p_{\alpha}(z_a|z_b,a)$ denotes the conditional prior over the perturbation representation. To enrich the latent space with perturbation-specific information, we formulate $p_{\alpha}(z_a|z_b,a)$ as a conditional energy-based distribution:
\begin{equation}
p_{\alpha}(z_a|z_b,a) = p_0(z_a)\exp\left[-E_{\alpha}(z_a,z_b,a)\right]/F_{\alpha}(z_b,a),
\label{method_eq:ebm_prior}
\end{equation}
where $p_0(z_a)$ denotes the prior without perturbation constraints (i.e, the prior used in existing works for perturbation space). $E_{\alpha}(z_a,z_b,a)$ is the energy-based model that revises the prior according to the perturbation: for a given basal state $z_b$ and perturbation $a$, it assigns lower energy (higher probability mass) to the perturbation representation $z_a$ that are more compatible with this condition and assigns higher energy to others. Unlike existing approaches~\cite{yuan2026perturbdiff,russkikh2020style} that encourage perturbation-specific information through implicit constraints, our method directly modifies the \emph{latent prior} of the perturbation space, explicitly tilting it toward the perturbation-preferred region. Here, $E_{\alpha}(z_a,z_b,a)$ can be implemented as a scalar-valued neural network parameterized by $\alpha$. $F_{\alpha}(z_b,a)$ is the corresponding partition function ensuring the normalization of the distribution, which is defined as:
\begin{equation}
F_{\alpha}(z_b,a)=\int p_0(z_a)\exp\left(-E_{\alpha}(z_a,z_b,a)\right)\mathrm{d}z_a.
\label{method_eq:partition_function}
\end{equation}
Based on this formulation, we then consider how to optimize such a model. Our target is to fit the real perturbed cell expression distribution, formally to maximize the $log \, p_{\Theta}(x|c,a)$. However, given $log \, p_{\Theta}(x|c,a) = log \int p_{\theta}(x|z_b,z_a)p(z_b,z_a|c,a) dz_bdz_a$, it has two challenges: (i) the integral over latent space is intractable, and (ii) the joint latent prior $p(z_b,z_a|c,a)$ is difficult to model. Therefore we adopt the variational inference in Eq.~\eqref{method_eq:variational_inference} to estimate the energy-formed prior as:
\begin{equation}
q_{\phi}(z_b,z_a|x_{ctr},c,a) = q_{\phi_\beta}(z_b|x_{ctr},c)q_{\phi_\alpha}(z_a|z_b,x_{ctr},a),
\label{method_eq:vae}
\end{equation}
where $q_{\phi_{\beta}}(z_b|x_{ctr},c)$ is the basal state encoder estimating the marginal prior $p_{\beta}(z_b|c)$ and $q_{\phi_{\alpha}}(z_a|z_b,x_{ctr},a)$ is the perturbation encoder estimating the conditional energy prior $p_{\alpha}(z_a|z_b,a)$. This variational formulation leads to the Evidence Lower Bound (ELBO) of $log\, p_{\Theta}(x|c,a)$ as:
\begin{equation}
\mathcal{L}_{\mathrm{ELBO}} = \mathbb{E}_{q_{\phi}(z_b,z_a|x_{ctr},c,a)}\left[log \frac{p_{\Theta}(x,z_b,z_a|c,a)}{q_{\phi}(z_b,z_a|x_{ctr},c,a)} \right]
\label{method_eq:whole_elbo}
\end{equation}
The optimization of the ELBO includes the complete model parameters $\{\Theta,\phi\}=\{\theta, \alpha, \phi_\beta, \phi_\alpha\}$, where $\theta$ is the parameter of generative decoder, $\phi_\beta$ is the parameter of basal state encoder, $\phi_\beta$ is the parameter of perturbation encoder, and $\alpha$ is our energy-based model parameter. We focus on the optimization corresponding to energy-based model, then we will prove that this optimization is equivalent to the optimization of entire ELBO. The Kullback-Leibler (KL) divergence between $q_{\phi_{\alpha}}(z_a |z_b,x_{ctr},a)$ and $p_{\alpha}(z_a |z_b,a)$ is derived as:
\begin{equation}
\begin{aligned}
D_{\mathrm{KL}}\left(q_{\phi_{\alpha}}||p_{\alpha}\right)
& = \mathbb{E}_{q_{\phi_{\alpha}}}\left[log \, q_{\phi_{\alpha}}(z_a|z_b,x_{ctr},a) - log \, p_{0}(z_a) + E_{\alpha}(z_a,z_b,a) + \log F_{\alpha}(z_b,a)\right] \\
& = D_{\mathrm{KL}}\left(q_{\phi_{\alpha}}(z_a |z_b,x_{ctr},a)||p_0(z_a)\right) + \mathbb{E}_{q_{\phi_{\alpha}}}\left[E_{\alpha}(z_a,z_b,a)\right] + \log F_{\alpha}(z_b,a).
\end{aligned}
\label{method_eq:energy_kl}
\end{equation}
Eq.~\eqref{method_eq:energy_kl} provides an intuitive understanding of PerturbEnergy: the KL divergence is not only bounded by regularization to the conventional prior (i.e., $D_{\mathrm{KL}}\left(q_{\phi_{\alpha}}(z_a |z_b,x_{ctr},a)||p_0(z_a)\right)$), it also considers the compatibility between the inferred perturbation representation and the condition $(z_b,a)$. Since $\log F_{\alpha}(z_b,a)$ is independent of $z_a$, optimizing Eq.~\eqref{method_eq:energy_kl} encourages $q_{\phi_{\alpha}}(z_a |z_b,x_{ctr},a)$ to prefer $z_a$ with more perturbation-specific information. In this paper, we term the \textbf{negative KL} $-D_{\mathrm{KL}}\left(q_{\phi_{\alpha}}||p_{\alpha}\right)$ in Eq.~\eqref{method_eq:energy_kl} as \textbf{Energy-ELBO} and its relationship with the full objective $\mathcal{L}_{\mathrm{ELBO}}$ is formalized in Theorem~\ref{method_theory:theorem_1}. The detailed proof is provided in Appendix~\ref{app:theory_elbo_equivalence}.
\begin{theorem}
Energy-ELBO $-D_{\mathrm{KL}}\left(q_{\phi_{\alpha}}||p_{\alpha}\right)$ is a part of the $log\, p_{\Theta}(x|c,a)$ ELBO, specifically we have:
\begin{equation}
\mathcal{L}_{\mathrm{ELBO}} = \mathbb{E}_{q_{\phi}}\left[log \, p_{\theta}(x|z_b,z_a)\right] - D_{\mathrm{KL}}(q_{\phi_\beta}||p_{\beta}) + \mathbb{E}_{q_{\phi_\beta}}\left[-D_{\mathrm{KL}}\left(q_{\phi_{\alpha}}||p_{\alpha}\right)\right].
\label{method_eq:entire_elbo}
\end{equation}
\label{method_theory:theorem_1}
\end{theorem}
\vspace{-10pt}
Therefore, given specific $(z_b,a)$ the optimization of Energy-ELBO is exactly equivalent to optimizing $\mathcal{L}_{\mathrm{ELBO}}$ at the complete model.
\subsection{Optimization for PerturbEnergy}
\label{method_subsec:energy_optimization}
Although we have derived an explicit form of the Energy-ELBO, optimizing PerturbEnergy still face a problem that the partition function $F_{\alpha}(z_b,a)$ in Eq.~\eqref{method_eq:partition_function} requires integration over the entire perturbation latent space and is generally intractable. Fortunately, its gradient admits a tractable form, which is described in Lemma~\ref{method_theory:lemma_1}. The proof is provided in Appendix~\ref{app:theory_partition_function_gradient}.
\begin{lemma}
Given the partition function is defined as Eq.~\eqref{method_eq:partition_function}, the gradient with respect to the energy parameters $\alpha$ is:
\begin{equation}
\nabla_{\alpha} \log F_{\alpha}(z_b,a) = - \mathbb{E}_{p_{\alpha}(z_a |z_b,a)}
\left[
\nabla_{\alpha}
E_{\alpha}(z_a,z_b,a)
\right].
\end{equation}
\label{method_theory:lemma_1}
\end{lemma}
\vspace{-1.8em}
Therefore, the gradient of $\mathcal{L}_{\mathrm{ELBO}}$ with respect to the energy parameters $\alpha$ is as follows, the derivation is given in Appendix~\ref{app:derivation_of_energy_gradient}.
\begin{equation}
\begin{aligned}
\nabla_{\alpha} \mathcal{L}_{\mathrm{ELBO}} = \mathbb{E}_{q_{\phi_\beta}} \{-\mathbb{E}_{q_{\phi_{\alpha}}}\left[\nabla_{\alpha} E_{\alpha}(z_a,z_b,a)\right] + \mathbb{E}_{p_{\alpha}}\left[\nabla_{\alpha}E_{\alpha}(z_a,z_b,a)\right]\}.
\end{aligned}
\label{method_eq:elbo_gradient}
\end{equation}
Since the perturbation representation is sampled based on the fixed condition $(z_b,a)$, and our objective is to \emph{maximize} the EBLO, we consider the corresponding \emph{negative gradient}:
\begin{equation}
\begin{aligned}
\nabla_{\alpha} \mathcal{J}_{\mathrm{E}}(\alpha) = \underbrace{\mathbb{E}_{q_{\phi_{\alpha}}}\left[\nabla_{\alpha} E_{\alpha}(z_a,z_b,a)\right]}_{\text{positive phase}} - \underbrace{\mathbb{E}_{p_{\alpha}}\left[\nabla_{\alpha}E_{\alpha}(z_a,z_b,a)\right]}_{\text{negative phase}}.
\end{aligned}
\label{method_eq:j_gradient}
\end{equation}
This gradient is exactly the same as maximizing the Energy-ELBO in Eq.~\eqref{method_eq:energy_kl}, which again shows the equivalence of optimizing the Energy-ELBO and the complete ELBO. The \emph{positive phase} for this gradient is easy to calculate since positive samples can be obtained from the real cell data via $z_a^{+} \sim q_{\phi_{\alpha}}(z_a |z_b,x_{ctr},a)$. In contrast, the \emph{negative phase} requires sampling $z_a$ from $p_{\alpha}(z_a|z_b,a)$, which is intractable as it involves the partition function. We provide two practical solutions. The first is based on physical function of PerturbEnergy, which assigns lower energy to compatible representations and higher energy to incompatible ones. Accordingly, we can sample the \emph{mismatched} perturbation representations for the negative phase, formally via $z_a^{-} \sim q_{\phi_{\alpha}}(z_a |z_b,x_{ctr},a^{-})$ where $a^{-}$ denotes a perturbation different from the true $a$. The second solution exploits an additional property of our energy-based prior: \emph{its gradient with respect to $z_a$ can be calculated directly}. This enables us to use Langevin dynamics to approximately sample from $p_{\alpha}(z_a|z_b,a)$ for the negative phase, which is described as follows and the algorithm can be found in Appendix~\ref{app_alg:langevin}.

\noindent \textbf{Langevin Sampling}: Since the partition function $F_{\alpha}(z_b,a)$ is independent of $z_a$, the gradient of $log \, p_{\alpha}(z_a|z_b,a)$ with respect to $z_a$ can be written as:
\begin{equation}
    \nabla_{z_a} \log \, p_{\alpha}(z_a |z_b,a) = \nabla_{z_a}\log \, p_0(z_a) - \nabla_{z_a}E_{\alpha}(z_a,z_b,a),
\label{method_eq:lagevin_sample_gradient}
\end{equation}
We adopt the Euler-Maruyama discretized Langevin, which updates the sample as: $z_a^{(k+1)} = z_a^{(k)} + \frac{\eta}{2} \nabla_{z_a} \log \, p_{\alpha}(z_a^{(k)} |z_b,a) + \sqrt{\eta}\epsilon^{(k)}$, where $z_a^{(k)}$ denotes the sample at the $k$-th step, $\eta>0$ is the Langevin step size, and $\epsilon^{(k)} \sim \mathcal{N}(0,I)$ is the random noise term corresponding to Brownian motion. Assuming a Gaussian prior $p_0(z_a)$ and combining it with Eq.~\eqref{method_eq:lagevin_sample_gradient}, we obtain:
\begin{equation}
z_a^{(k+1)} = z_a^{(k)} - \frac{\eta}{2}\left[\nabla_{z_a}E_{\alpha}(z_a^{(k)},z_b,a) + z_a^{(k)}\right] + \sqrt{\eta}\epsilon^{(k)}.
\label{method_eq:lagevin_update}
\end{equation}
After $K$ steps, we use $z_a^{(K)}$ as the negative sample. Intuitively, the positive phase lowers the energy of latent representations inferred from real perturbation and cell observations, whereas the negative phase raises the energy of representations favored by the current model but not supported by the data. In other word, the negative phase corrects representations to which the model has incorrectly assigned low energy. Using either of the two solutions, we can obtain negative samples for optimization. Together with the positive samples, the gradient for optimizing PerturbEnergy can then be calculated via Eq.~\eqref{method_eq:j_gradient}.

\noindent \textbf{Optimization:} We implement the optimization of the entire model, including the varitional encoder $q_{\phi_\alpha}$, $q_{\phi_\beta}$, generator $p_{\theta}$, and the PerturbEnergy $E_{\alpha}$. The target is to \emph{maximize} the ELBO in Eq.~\eqref{method_eq:entire_elbo}, which can be derived as:
\begin{equation}
\mathcal{L}_{\mathrm{ELBO}} = \mathbb{E}_{q_{\phi}}\left[log \, p_{\theta}\right] - D_{\mathrm{KL}}(q_{\phi_\beta}||p_{\beta}) - \mathbb{E}_{q_{\phi_\beta}}\left[D_{KL}(q_{\phi_\alpha}||p_{0})\right] - \mathbb{E}_{q_{\phi_\beta}}\left[\mathbb{E}_{q_{\phi_\alpha}} E_{\alpha} + log F_{\alpha} \right].
\label{method_eq:elbo_derivation}
\end{equation}
We use Maximum Mean Discrepancy (MMD) as the surrogate to reconstruction term $\mathbb{E}_{q_{\phi}}\left[log \,p_{\theta}\right]$. For the two KL-divergence terms, we assume the Gaussian prior for $p_{0}$ and $p_{\beta}$. The gradient of $E_{\alpha}$ with respect to $\phi_{\alpha}$ and $\phi_{\beta}$ is easy to calculate since the gradient $\nabla_{z_a/z_b} E_{\alpha}$ is direct and such gradient can be backpropagated to $\phi_{\alpha}/\phi_{\beta}$ via Jacobian $\frac{\partial z_a}{\partial \phi_{\alpha}}/\frac{\partial z_b}{\partial \phi_{\beta}}$. The last one is the gradient of $log F_{\alpha}$ with respect to $\phi_{\beta}$. According to Lemma~\ref{method_theory:lemma_1}, its gradient with respect to $z_b$ takes a similar form: $- \mathbb{E}_{p_{\alpha}}\left[\nabla_{z_b}E_{\alpha}\right]$, which can be further backpropagated to $\phi_{\beta}$ via Jacobian $\frac{\partial z_b}{\partial \phi_{\beta}}$. Together with $\nabla_{\alpha} \mathcal{J}_{\mathrm{E}}(\alpha)$ in Eq.~\eqref{method_eq:j_gradient}, the whole model can be optimized as: Given batch data $(x_{ctr},x)$ under the same context $c$ with perturbation $a$, first calculate the basal representations $z_{b}$ via $q_{\phi_\beta}$, then generate $N$ positive batch samples $\{z_{a,j}^{+}\}_{j=1}^{N}$ via $q_{\phi_\alpha}$ and $N$ negative batch samples $\{z_{a,j}^{-}\}_{j=1}^{N}$ via Langevin or Mismatching. Thus $\alpha$ is updated by $\nabla_{\alpha} \mathcal{J}_{\mathrm{E}}(\alpha)$ whereas $\{\theta,\phi_\beta, \phi_\alpha\}$ are updated by $\{\nabla_{\theta}\mathcal{L}_{\mathrm{ELBO}}, \nabla_{\phi_\beta}\mathcal{L}_{\mathrm{ELBO}},\nabla_{\phi_\alpha}\mathcal{L}_{\mathrm{ELBO}}\}$. The Algorithm is provided in Appendix~\ref{app_alg:optimization}

\noindent \textbf{Inference:} We adopt the importance sampling in the inference stage. For unperturbed cell $(x_{ctr}, c)$ and perturbation $a$: (i) calculate the basal state $z_b$ via $q_{\phi_\beta}$; (ii) sample multiple perturbation representations $\{z_{a,i}\}_{i=0}^{M}$ via $q_{\phi_\alpha}$; (iii) select the $z_a$ with the lowest energy $E_{\alpha}$; (vi) generate cellular response $x$ via $p_{\theta}$ based on the $z_b$ and $z_a$. The algorithm is provided in Appendix~\ref{app_alg:inference}.

\subsection{Discussion}
We analyze the effectiveness of PerturbEnergy in enhancing perturbation-specific information. Our theoretical finding is formalized as Theorem~\ref{method_theory:theorem_2}, whose proof is provided in Appendix~\ref{app:theory_conditional_mutual_information_lower_bound}.
\begin{theorem}
Optimizing PerturbEnergy leading to the increase of the variational mutual information lower bound between the perturbation and the latent representation. Specifically, the relation is:
\begin{equation}
\begin{aligned}
I(A;Z_a|Z_b,C) & \geq H(A|Z_b,C) + \mathbb{E}\left[\log r_{\alpha}(A |Z_a,Z_b,C)\right],
\end{aligned}
\label{method_eq:mutual_info_bound}
\end{equation}
\begin{equation}
r_{\alpha}(a|z_a,z_b,c) = \frac{\pi(a|z_b,c) \exp (-E_{\alpha}(z_a,z_b,a))}{\sum_{a' \in \mathcal{A}} \pi(a'|z_b,c) \exp (-E_{\alpha}(z_a,z_b,a'))}
\end{equation}
where $\pi(a|z_b,c)$ is the true conditional prior of perturbation, $I(.;.|Z_b,C)$ denotes the conditional mutual information and $H(.|Z_b, C)$ is the conditional entropy.
\label{method_theory:theorem_2}
\end{theorem}
\vspace{-10pt}

\section{Experiments}
\label{sec:experiments}

\subsection{Experimental Settings}
\noindent \textbf{Baselines:} We compare with following baselines: (i) Linear methods: we use Linear model on cell expression and scGPT~\cite{cui2024scgpt} representations; (ii) VAE-based methods: we use scGen~\cite{lotfollahi2019scgen}, CPA~\cite{lotfollahi2023predicting}, stVAE~\cite{russkikh2020style}, and PerturbedVAE~\cite{jiang2026makes}; iii) Other architecture: we use PerturbDiff~\cite{yuan2026perturbdiff}, Squidiff~\cite{he2026squidiff}, CellFlow~\cite{klein2025cellflow} and STATE~\cite{adduri2026predicting}. The detailed settings of all these baselines are provided in Appendix~\ref{app_exp_part:baseline_settings}.

\noindent
\begin{minipage}[t]{0.55\textwidth}
\vspace{-5pt}
\textbf{Datasets:}
We use four widely used benchmark datasets: Replogle~\cite{replogle2022mapping}, PBMC~\cite{yuan2026perturbdiff}, Norman2019~\cite{norman2019exploring}, and Tahoe100M~\cite{zhang2025tahoe}. For a fair comparison, we follow the standardized experimental settings used in STATE~\cite{adduri2026predicting} and VC benchmark~\cite{mao2026benchmarking}. For all datasets, we use top 2,000 highly variable genes (HVGs). The basic statistics of these datasets are summarized in Table~\ref{tab:data_stats}, and detailed descriptions are provided in Appendix~\ref{app_exp_part:dataset_description}.
\end{minipage}
\hfill
\begin{minipage}[t]{0.42\textwidth}
\vspace{-5pt}
\centering
\scriptsize

\captionof{table}{\textbf{Data statistics.}
Number of cells(\#Cells), perturbations(\#Pert.), context(\#CT.) and batches(\#Batches).}
\label{tab:data_stats}

\vspace{0.3em}

\resizebox{\linewidth}{!}{
\begin{tabular}{lrrrr}
\toprule
Dataset    & \#Cells & \#Pert. & \#CT. & \#Batches \\
\midrule
Replogle   & 643.3K  & 2,023   & 4     & 56     \\
PBMC       & 9.7M    & 90      & 18    & 12     \\
Norman2019 & 111.4K  & 236     & 1     & 8      \\
Tahoe100M  & 101.2M  & 1,137   & 50    & 14     \\
\bottomrule
\end{tabular}
}
\end{minipage}


\noindent \textbf{Metrics:} We evaluate all the methods from two perspectives: (i) how well a model reproduces the highly variable gene (HVG) expression (\textbf{HVG group}); (ii) how well a model predict the differential expression gene (DEG) expression influenced by the perturbation. (\textbf{DEG group}). For the HVG group, we adopt Cell-Eval~\cite{adduri2026predicting} including: (a) correlation measures: Pearson correlation (PCC), Pearson delta correlation (PDCorr), Coefficient of Determination (R$^2$), log-fold-change Spearman (LFCSpear), log-fold-change cosine (LogFC\_cos), direction agreement (DirAgr) and Perturbation discrimination (PDS-$\ell_1$/$\ell_2$/cos); (b) DE measures: precison (DEPrec), recall (DERec), Jaccard index (DEOverlap), AUROC and AUPRC between the prediction and real top-50 DEGs; (c) distance measures: Root Mean Square Error (RMSE), Mean Square Error (MSE), Mean Absolute Error (MAE), Energy Distance (EDist) and Sinkhorn Divergence (Sinkhorn); For the DEG group, we calculate RMSE, MAE, PCC and R$^2$ over the top-50/100/200 DEGs respectively. The full definitions of all these metrics are given in Appendix~\ref{app_exp:metrics_description}.

\noindent \textbf{Implementation:} PerturbEnergy is model-agnostic and and we implement it with a standard VAE backbone. Both encoders use two-layer MLPs with a latent dimension of 512, while context $c$ and perturbation $a$ are represented by learnable embedding tables. $E_{\alpha}(z_a,z_b,a)$ is implemented as a simple neural network: we concatenate $z_a$ and $z_b$ and feed it to a two-layer MLP and then output the cosine similarity score with the embedding of $a$. A two-layer MLP decoder predicts the expression of 2,000 HVGs from the concatenated latent representations. We train with Adam~\cite{kingma2014adam} at a learning rate of $10^{-3}$, using batch sizes and maximum training steps of (256, 20,000) for Norman2019, (512, 40,000) for Replogle, (512, 60,000) for PBMC and (512, 160,000) Tahoe100M. By default, PerturbEnergy-Langevin(Lang) uses $K=20$ and $\eta=10^{-4}$, while PerturbEnergy-Mismatch(Mis) randomly samples embeddings from other perturbations as negatives. During inference, we set the candidate number $M=8$ by default. We apply the same checkpoint selection strategy to all methods, as detailed in Appendix~\ref{app_exp_part:baseline_settings}.

\subsection{Experimental Results and Discussions}
\begin{figure*}[h!]
\centering
\includegraphics[width=\textwidth]{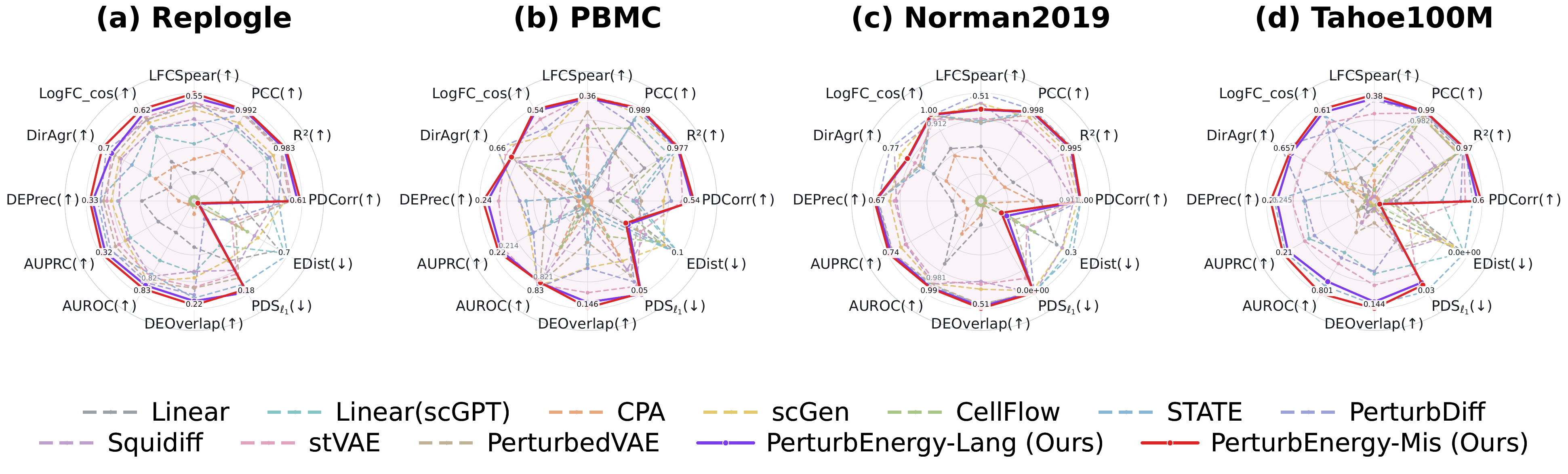}
\vspace{-1.5em}
\caption{Performance on parts of the HVG group metrics. $\uparrow/\downarrow$ means the higher/lower is better.}
\label{fig:radar_hvg}
\vspace{-0.5em}
\end{figure*}

\noindent \textbf{HVG Group.} Figure~\ref{fig:radar_hvg} summarizes the performance of all methods on the HVG group, with detailed results provided in Appendix~\ref{app_exp:hvg_group_res}. For correlation- and distance-based metrics, including PDCorr, LFCSpear, R$^2$, and EDist, PerturbEnergy generally achieves competitive performance across all four datasets, demonstrating its ability to capture overall expression patterns and distributional similarities. Moreover, our method consistently outperforms the baselines on differential expression (DE) metrics, including DEPrec, AUROC, AUPRC, and DEOverlap. This improvement highlights the effectiveness of explicitly modeling perturbation-specific information, which is particularly important for identifying differentially expressed genes (DEGs) associated with perturbations.

\noindent \textbf{DEG Group.} Unlike the HVG group which evaluates general gene expression patterns, the DEG group metrics focus on predicting perturbation-specific responses. We conduct experiments for the DEG group metrics and report the RMSE results in Table~\ref{tab:deg-rmse}, with the remaining metrics and corresponding analysis provided in the Appendix~\ref{app_exp:deg_group_res}. We have the following key observations: (i) Although linear methods can capture general expression patterns in the HVG group, they struggle to predict perturbation-induced changes, suggesting that simple regression models may focus on house keeping gene patterns rather than perturbation effects. (ii) CellFlow, Squidiff, and PerturbDiff model perturbation effects through diffusion or flow-matching processes, but their performance varies across datasets. For example, PerturbDiff achieves state-of-the-art performance on Norman2019 but struggles on Replogle. This variation suggests that relying primarily on model expressiveness may not consistently capture perturbation effects when they are obscured by other sources of cellular variation. Notably, STATE achieves competitive performance, potentially benefiting from batch-level self-attention that mitigates cellular heterogeneity and its large model capacity. (iii) Among representation learning-based methods, PerturbedVAE achieves better performance than other baselines, highlighting the potential benefits of explicitly modeling perturbation-specific information. (iv) Despite using a simple MLP backbone, PerturbEnergy achieves the best performance on Replogle, PBMC, and Tahoe100M and the second-best performance on Norman2019. These results demonstrate the effectiveness of our design in facilitate capturing perturbation-specific effects.

\newcommand{\std}[1]{{\scriptsize$\pm$#1}}

\begin{table}[t]
\centering
\scriptsize
\caption{RMSE ($\downarrow$) over the true top-$K$ DEGs. Values are the mean over held-out perturbation groups $\pm$ the s.e.m. across those groups. Best in \textbf{bold}, second best \underline{underlined}.}
\vspace{-1em}
\label{tab:deg-rmse}
\setlength{\tabcolsep}{4pt}
\resizebox{\textwidth}{!}{%
\begin{tabular}{lcccccc}
\toprule
\multirow{2}{*}{Method} & \multicolumn{3}{c}{Replogle} & \multicolumn{3}{c}{Tahoe100M} \\
\cmidrule(lr){2-4} \cmidrule(lr){5-7}
 & Top-50 & Top-100 & Top-200 & Top-50 & Top-100 & Top-200 \\
\midrule
Linear & 0.2738\,\std{0.0086} & 0.2390\,\std{0.0072} & 0.2045\,\std{0.0058} & 0.4536\,\std{0.0072} & 0.3866\,\std{0.0053} & 0.3032\,\std{0.0040} \\
Linear (scGPT) & 0.2449\,\std{0.0081} & 0.2081\,\std{0.0066} & 0.1738\,\std{0.0052} & 0.3748\,\std{0.0060} & 0.3257\,\std{0.0044} & 0.2588\,\std{0.0034} \\
\midrule
CellFlow & 0.2174\,\std{0.0067} & 0.1850\,\std{0.0055} & 0.1536\,\std{0.0043} & 0.3860\,\std{0.0062} & 0.3318\,\std{0.0044} & 0.2617\,\std{0.0034} \\
Squidiff & 0.2173\,\std{0.0055} & 0.1907\,\std{0.0044} & 0.1658\,\std{0.0035} & 0.4192\,\std{0.0071} & 0.3889\,\std{0.0051} & 0.3231\,\std{0.0039} \\
PerturbDiff & 0.2612\,\std{0.0096} & 0.2198\,\std{0.0079} & 0.1808\,\std{0.0063} & 0.3558\,\std{0.0070} & 0.3079\,\std{0.0051} & 0.2462\,\std{0.0040} \\
STATE & 0.2024\,\std{0.0064} & 0.1727\,\std{0.0052} & 0.1447\,\std{0.0041} & 0.3356\,\std{0.0041} & 0.2935\,\std{0.0030} & 0.2342\,\std{0.0023} \\
\midrule
scGen & 0.2336\,\std{0.0078} & 0.1992\,\std{0.0065} & 0.1661\,\std{0.0051} & 0.4462\,\std{0.0081} & 0.3818\,\std{0.0059} & 0.3018\,\std{0.0046} \\
CPA & 0.3462\,\std{0.0132} & 0.2905\,\std{0.0110} & 0.2363\,\std{0.0088} & 0.7070\,\std{0.0103} & 0.6457\,\std{0.0075} & 0.5304\,\std{0.0056} \\
stVAE & 0.2147\,\std{0.0071} & 0.1795\,\std{0.0057} & 0.1484\,\std{0.0044} & 0.3293\,\std{0.0057} & 0.2883\,\std{0.0041} & 0.2315\,\std{0.0031} \\
PerturbedVAE & 0.2130\,\std{0.0071} & 0.1817\,\std{0.0058} & 0.1514\,\std{0.0047} & 0.3984\,\std{0.0075} & 0.3476\,\std{0.0055} & 0.2791\,\std{0.0042} \\
\midrule
PerturbEnergy-Lang & \underline{0.1870}\,\std{0.0061} & \underline{0.1591}\,\std{0.0049} & \underline{0.1335}\,\std{0.0039} & \textbf{0.3114}\,\std{0.0051} & \textbf{0.2719}\,\std{0.0037} & \textbf{0.2179}\,\std{0.0029} \\
PerturbEnergy-Mis & \textbf{0.1780}\,\std{0.0057} & \textbf{0.1522}\,\std{0.0046} & \textbf{0.1285}\,\std{0.0036} & \underline{0.3123}\,\std{0.0051} & \underline{0.2723}\,\std{0.0037} & \underline{0.2179}\,\std{0.0029} \\
\bottomrule
\addlinespace[6pt]
\toprule
\multirow{2}{*}{Method} & \multicolumn{3}{c}{PBMC} & \multicolumn{3}{c}{Norman2019} \\
\cmidrule(lr){2-4} \cmidrule(lr){5-7}
 & Top-50 & Top-100 & Top-200 & Top-50 & Top-100 & Top-200 \\
\midrule
Linear & 0.1065\,\std{0.0046} & 0.0901\,\std{0.0037} & 0.0740\,\std{0.0029} & 0.4105\,\std{0.0264} & 0.3136\,\std{0.0208} & 0.2313\,\std{0.0155} \\
Linear (scGPT) & 0.0979\,\std{0.0045} & 0.0816\,\std{0.0036} & 0.0656\,\std{0.0029} & 0.2810\,\std{0.0090} & 0.2422\,\std{0.0073} & 0.1837\,\std{0.0055} \\
\midrule
CellFlow & 0.0821\,\std{0.0024} & 0.0678\,\std{0.0019} & 0.0541\,\std{0.0015} & 0.1603\,\std{0.0145} & 0.1292\,\std{0.0118} & 0.1007\,\std{0.0089} \\
Squidiff & 0.1271\,\std{0.0051} & 0.1106\,\std{0.0042} & 0.0931\,\std{0.0033} & 0.2041\,\std{0.0106} & 0.1743\,\std{0.0084} & 0.1447\,\std{0.0068} \\
PerturbDiff & 0.0829\,\std{0.0031} & 0.0718\,\std{0.0025} & 0.0602\,\std{0.0020} & \textbf{0.1248}\,\std{0.0088} & \textbf{0.1035}\,\std{0.0067} & \textbf{0.0836}\,\std{0.0051} \\
STATE & 0.0935\,\std{0.0033} & 0.0774\,\std{0.0027} & 0.0619\,\std{0.0022} & 0.1431\,\std{0.0117} & 0.1168\,\std{0.0091} & 0.0922\,\std{0.0065} \\
\midrule
scGen & 0.0960\,\std{0.0049} & 0.0787\,\std{0.0040} & 0.0625\,\std{0.0031} & 0.1729\,\std{0.0102} & 0.1404\,\std{0.0082} & 0.1108\,\std{0.0065} \\
CPA & 0.1581\,\std{0.0069} & 0.1317\,\std{0.0054} & 0.1062\,\std{0.0041} & 0.4419\,\std{0.0282} & 0.3370\,\std{0.0219} & 0.2477\,\std{0.0162} \\
stVAE & 0.0729\,\std{0.0022} & 0.0613\,\std{0.0018} & 0.0500\,\std{0.0014} & 0.2110\,\std{0.0146} & 0.1656\,\std{0.0114} & 0.1260\,\std{0.0084} \\
PerturbedVAE & 0.1071\,\std{0.0032} & 0.0925\,\std{0.0025} & 0.0777\,\std{0.0020} & 0.1400\,\std{0.0123} & 0.1186\,\std{0.0100} & 0.0953\,\std{0.0075} \\
\midrule
PerturbEnergy-Lang & \textbf{0.0637}\,\std{0.0020} & \textbf{0.0542}\,\std{0.0016} & \textbf{0.0448}\,\std{0.0013} & 0.1434\,\std{0.0113} & 0.1176\,\std{0.0086} & 0.0927\,\std{0.0062} \\
PerturbEnergy-Mis & \underline{0.0654}\,\std{0.0020} & \underline{0.0557}\,\std{0.0016} & \underline{0.0462}\,\std{0.0013} & \underline{0.1376}\,\std{0.0110} & \underline{0.1135}\,\std{0.0083} & \underline{0.0898}\,\std{0.0061} \\
\bottomrule
\end{tabular}}
\vspace{-2em}
\end{table}

\begin{table}[h!]
\centering
\tiny
\caption{Perturbation classification accuracy. Best in \textbf{bold}, second best \underline{underlined}.}
\vspace{-1em}
\label{tab:perturbation_study}
\begin{tabular}{l|cccc|cccc}
\toprule
\multirow{2}{*}{Method} & \multicolumn{4}{c|}{Latent Representation} & \multicolumn{4}{c}{Expression Delta} \\
\cmidrule(lr){2-5} \cmidrule(lr){6-9}
 & Replogle & PBMC & Norman2019 & Tahoe100M & Replogle & PBMC & Norman2019 & Tahoe100M \\
\midrule
scGPT & 18.5$\pm$1.1 & 33.1$\pm$1.0 & 74.4$\pm$2.9 & 38.2$\pm$1.9 & 21.5$\pm$2.5 & 24.8$\pm$1.4 & 73.1$\pm$2.7 & 39.9$\pm$2.3 \\
PerturbDiff & 9.1$\pm$3.3 & 14.9$\pm$2.3 & 29.5$\pm$4.1 & 19.5$\pm$3.8 & 18.4$\pm$1.2 & 18.3$\pm$1.5 & 52.6$\pm$3.0 & 31.4$\pm$3.4 \\
STATE & 16.3$\pm$0.9 & 12.6$\pm$1.1 & 23.1$\pm$2.0 & 23.6$\pm$3.6 & 24.5$\pm$1.8 & 26.9$\pm$2.8 & 70.0$\pm$1.2 & 41.8$\pm$4.1 \\
\midrule
scGen & 24.6$\pm$1.6 & 48.4$\pm$2.5 & 81.6$\pm$1.7 & 46.1$\pm$2.5 & 19.7$\pm$2.3 & 24.6$\pm$1.9 & 68.2$\pm$2.3 & 35.2$\pm$1.5 \\
CPA & 22.2$\pm$0.4 & 43.0$\pm$1.9 & 78.2$\pm$1.9 & 33.3$\pm$4.0 & 16.3$\pm$1.4 & 18.5$\pm$1.0 & 60.1$\pm$3.3 & 24.7$\pm$1.6 \\
PerturbedVAE & 28.6$\pm$2.0 & 48.2$\pm$2.9 & 82.7$\pm$3.2 & 45.5$\pm$2.1 & 24.2$\pm$0.9 & 22.1$\pm$2.0 & 69.5$\pm$1.9 & 37.5$\pm$2.3 \\
\midrule
PerturbEnergy-Lang & \underline{50.5$\pm$1.5} & \underline{59.3$\pm$1.5} & \underline{91.7$\pm$2.8} & \underline{55.8$\pm$2.7} & \underline{29.5$\pm$1.4} & \underline{32.8$\pm$1.6} & \underline{77.2$\pm$2.5} & \underline{46.3$\pm$1.7} \\
PerturbEnergy-Mis & \textbf{52.4$\pm$1.2} & \textbf{61.8$\pm$2.0} & \textbf{96.4$\pm$1.5} & \textbf{58.5$\pm$2.2} & \textbf{31.9$\pm$1.1} & \textbf{36.6$\pm$0.8} & \textbf{81.7$\pm$2.3} & \textbf{47.8$\pm$2.1} \\
\bottomrule
\end{tabular}
\vspace{-2em}
\end{table}

\noindent \textbf{Perturbation Study.} To evaluate how effectively different methods capture perturbation-specific information, we conduct an analysis using two types of inputs: (i) latent representations and (ii) expression differences between predicted perturbed cells and observed control cells. For each setting, we train a two-layer MLP classifier to predict perturbation labels. The results are reported in Table~\ref{tab:perturbation_study} and Figure~\ref{fig:perturbation_study}. For latent representations, scGPT captures limited perturbation effect as it's likely to be subtilized; STATE and PerturbDiff achieve limited classification accuracy, as their representations primarily capture basal cellular states; VAE-based methods perform better by incorporating perturbation information into the latent space. Besides, richer perturbation-specific information generally leads to better perturbation prediction performance. For example, PerturbedVAE achieves higher accuracy than CPA and outperforms CPA in Top-50/100/200 RMSE. (See Table~\ref{tab:deg-rmse}). However, their accuracy remains below 50\% on Replogle, PBMC, and Tahoe100M, suggesting that existing representation learning strategies may not sufficiently preserve perturbation-specific information. For expression differences, STATE achieves competitive performance, whereas PerturbDiff struggles compared with other baselines. This variation suggests that relying primarily on model may not consistently capture perturbation effects. VAE-based methods also exhibit degraded performance, potentially because cellular complexity obscures perturbation-induced changes in the expression. In contrast, PerturbEnergy achieves the best performance in both settings, with particularly substantial improvements in latent representation classification (83.2\% on Replogle and 28.2\% on PBMC). These results demonstrate the effectiveness of our method in enriching latent representations with perturbation-specific information.

\begin{figure*}[h!]
\centering
\includegraphics[width=\textwidth]{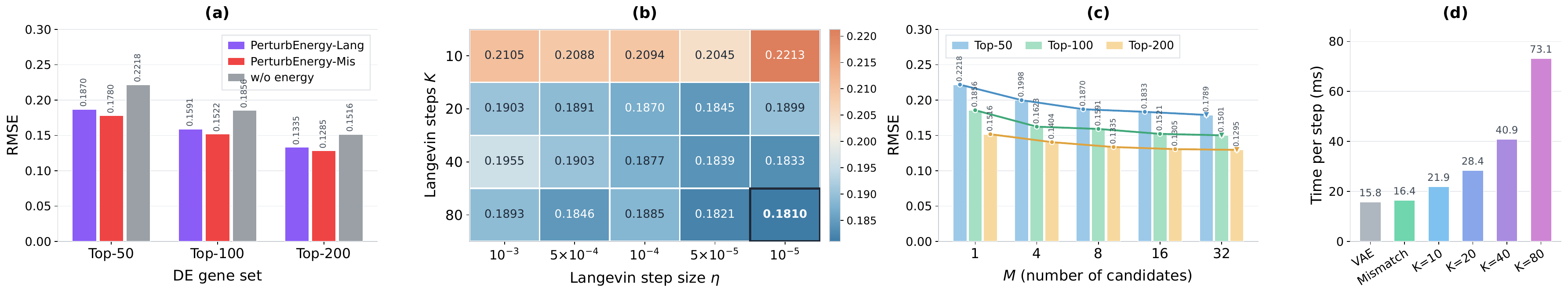}
\vspace{-1.5em}
\caption{Ablation and other study on Replogle. (a) is the ablation; (b) is the hyperparameter study about Langevin steps/size; (c) is the influence of candidate number $M$. (d) is the training efficiency.}
\label{fig:ablation_and_hyperparameter}
\vspace{-1em}
\end{figure*}
\noindent \textbf{Ablation and Other Study.} We conduct ablation, hyperparameter and efficiency studies on Replogle, as shown in Figure~\ref{fig:ablation_and_hyperparameter}. Removing the energy function consistently degrades performance, where the model reduces to a standard VAE without explicitly modeling perturbation-specific information. We then examine the effects of Langevin steps $K$ and step size $\eta$ using RMSE on the top-50 DEGs. Larger $K$ and smaller $\eta$ generally improve performance, whereas insufficient Langevin steps lead to degraded results. We also find that Mismatch generally has the same training efficiency as VAE (16.4 vs 15.6 ms/step) while Langevin with larger $K$ significantly decreases the training speed. As a balance, we set $K=20$ and $\eta=10^{-4}$ by default, which generally make the sample region with a flat gradient. Finally, we examine the effect of the candidate number $M$ during inference. Performance deteriorates with too few candidates ($M\leq 4$), while increasing $M$ to 16 or 32 yields only marginal improvements. We therefore set $M=8$ as a balance.

\noindent \textbf{Other Discussion.} Due to space limitations, we provide an analysis of the training scale/model dimension in the Appendix~\ref{app_exp_part:training_model_scale}. We also demonstrate the compatibility of PerturbEnergy with different VAE backbones in Appendix~\ref{app_exp_part:other_vae_backbone}. The visualizations are provided in Appendix~\ref{app_exp_part:visualization}.

\section{Conclusion}
\label{sec:conclusion}
In this paper, we introduce PerturbEnergy, a framework designed to enhance perturbation-specific information for cellular perturbation prediction. PerturbEnergy uses an energy function to reshape the latent prior, directing it toward perturbation preferred regions. We theoretically demonstrate that PerturbEnergy increases the mutual information between the representations and perturbations. Extensive experiments on four datasets validate its effectiveness.
\section{AI use statement}
In this work, we used generative AI tools only for polishing the paper writing, including fixing grammar mistakes, making expressions more easy to understand. We have not used generative AI tools for developing our research idea and experiments. The applicable and not applicable details are as follows:

\noindent \textbf{Applicable AI usage for this work:} (i) Grammar, spelling mistake finding and fixing; (ii) Expression and writing polishing, AI is used to paraphrase some of my sentences to make it clear and easy to understand, while exactly keeping the main content of the original sentences. (iii) Give suggestions on the paper formatting, including the position of our experimental Tables, Figures.

\noindent \textbf{Not Applicable AI usage for this work:} We do not use AI for our research idea developing, including: (i) AI is not used for reviewing related works; (ii) AI is not used for developing methodology and theory of this work; (iii) AI is not used for collecting experimental data; (iv) AI is not used for implementing our method codes;

\section{Ethics statement}
Our work is focused on AI-based prediction of biological cellular perturbation response, aiming to enhance the accuracy of cell prediction for common perturbations. This work is the application of the universal AI methodology in the context of cellular prediction, aiming to solve the practical problems in this scenario. The content of this research work does not involve any ethical issues, there is no risk data release, no harmful content, and no conflicts of interest. However, it should be noted that this method is used to improve the computational prediction accuracy of cell perturbations. Therefore, this work needs to be properly applied and should not be used to assist in harmful biomedical research.

\section{Reproducibility statement}
We declare that all the work presented in this paper is reproducible. Firstly, in terms of methodology, we have provided rigorous theoretical proofs for each theorem/lemma. Secondly, in terms of the software code, we have made the relevant code of our work public via anonymous link. Furthermore, we adopt the publicly available datasets, data processing methods, and evaluation metrics. This means that the experimental and theoretical results of this work can all be reproduced.

\bibliography{iclr2027_conference}
\bibliographystyle{iclr2027_conference}

\clearpage
\appendix
\section{Appendix}

\begingroup
\etocsetnexttocdepth{subsubsection}
\etocsettocstyle{\subsection*{Appendix Contents}}{}
\localtableofcontents
\endgroup
\clearpage

\subsection{Theoretical supports}
\label{app:theory}

\subsubsection{Proof of Theorem 1: Equivalence of the Energy-ELBO}
\label{app:theory_elbo_equivalence}
Here we give a formal proof to the Theorem~\ref{method_theory:theorem_1} as follows:

\begin{proof}

We begin from with the ELBO of $log\, p_{\Theta}(x|c,a)$ which can be derived as:
\begin{equation}
\label{app_eq:app_original_elbo}
\begin{aligned}
\mathcal{L}_{\mathrm{ELBO}}
& = \mathbb{E}_{q_{\phi}(z_b,z_a|x_{ctr},c,a)}
\Big[ \log p_{\theta}(x|z_b,z_a) + \log p(z_b,z_a|c,a) - log q_{\phi}(z_b,z_a|x_{ctr},c,a)\Big] \\
& = \mathbb{E}_{q_{\phi}}\left[\log p_{\theta}(x|z_b,z_a)\right] - D_{KL}(q_{\phi}(z_b,z_a|x_{ctr},c,a)||p(z_b,z_a|c,a)).
\end{aligned}
\end{equation}
According to our joint prior assumption in Eq.~\eqref{method_eq:latent_variable_model} and the variational inference formation in Eq.~\eqref{method_eq:vae}, we have:
\begin{equation}
p(z_b,z_a|c,a) = p_{\beta}(z_b|c) p_{\alpha}(z_a|z_b,a),
\label{app_eq:latent_variable_model}
\end{equation}
and:
\begin{equation}
q_{\phi}(z_b,z_a|x_{ctr},c,a) = q_{\phi_\beta}(z_b|x_{ctr},c)q_{\phi_\alpha}(z_a|z_b,x_{ctr},a).
\label{app_eq:vae}
\end{equation}
Besides, according to the chain rule of Kullback-Leibler (KL) divergence, where:
\begin{equation}
D_{KL}(P(X, Y)||Q(X, Y)) = D_{KL}(P(X)||Q(X)) + D_{KL}(P(Y|X)||Q(Y|X)).
\end{equation}
The $\mathcal{L}_{\mathrm{ELBO}}$ then becomes:
\begin{equation}
\begin{aligned}
\mathcal{L}_{\mathrm{ELBO}}
=&
\mathbb{E}_{q_{\phi}}\left[\log p_{\theta}(x|z_b,z_a)\right]\\
& + \mathbb{E}_{q_{\phi_\beta}}\left[\log p_{\beta}(z_b|c)-\log q_{\phi_\beta}(z_b|x_{ctr},c)\right]\\
& + \mathbb{E}_{q_{\phi_\beta}}\{\mathbb{E}_{q_{\phi_\alpha}}\left[\log p_{\alpha}(z_a|z_b,a) - \log q_{\phi_\alpha}(z_a|z_b,x_{ctr},a)\right]\}.
\end{aligned}
\end{equation}
The second term $\mathbb{E}_{q_{\phi_\beta}}\left[\log p_{\beta}(z_b|c)-\log q_{\phi_\beta}(z_b|x_{ctr},c)\right]$ is:
\begin{equation}
- D_{\mathrm{KL}}\left(q_{\phi_\beta}||p_{\beta}\right).
\end{equation}
For the $\mathbb{E}_{q_{\phi_\alpha}}\left[\log p_{\alpha}(z_a|z_b,a) - \log q_{\phi_\alpha}(z_a|z_b,x_{ctr},a)\right]$ in the third term, it actually is:
\begin{equation}
-D_{\mathrm{KL}}\left(q_{\phi_{\alpha}}||p_{\alpha}\right).
\end{equation}

Thus, the complete ELBO is:
\begin{equation}
\mathcal{L}_{\mathrm{ELBO}} = \mathbb{E}_{q_{\phi}}\left[log \, p_{\theta}(x|z_b,z_a)\right] - D_{\mathrm{KL}}(q_{\phi_\beta}||p_{\beta}) + \mathbb{E}_{q_{\phi_\beta}}\left[-D_{\mathrm{KL}}\left(q_{\phi_{\alpha}}||p_{\alpha}\right)\right].
\end{equation}

Until now, we have the proved the formation in Theorem~\ref{method_theory:theorem_1}.

In addition, given the definition of our conditional EBM prior in Eq.~\ref{method_eq:ebm_prior}, the complete ELBO can be further derived as:
\begin{equation}
\begin{aligned}
\mathcal{L}_{\mathrm{ELBO}} & = \mathbb{E}_{q_{\phi}}\left[log \, p_{\theta}(x|z_b,z_a)\right] -D_{\mathrm{KL}}(q_{\phi_\beta}||p_{\beta}) \\
& - \mathbb{E}_{q_{\phi_\beta}}\big\{\mathbb{E}_{q_{\phi_\alpha}}[\log q_{\phi_\alpha}(z_a|z_b,x_{ctr},a) - \log p_0(z_a) + E_{\alpha}(z_a,z_b,a) + \log F_{\alpha}(z_b,a)]\big\}.
\end{aligned}
\end{equation}

Since $\log F_{\alpha}(z_b,a)$ is independent of $z_a$ for fixed $(z_b,a)$, thus:
\begin{equation}
\begin{aligned}
\mathcal{L}_{\mathrm{ELBO}} & = \mathbb{E}_{q_{\phi}}\left[log \, p_{\theta}(x|z_b,z_a)\right] -D_{\mathrm{KL}}(q_{\phi_\beta}||p_{\beta}) \\
& - \mathbb{E}_{q_{\phi_\beta}}\big\{D_{\mathrm{KL}}\left(q_{\phi_{\alpha}}(z_a |z_b,x_{ctr},a)||p_0(z_a)\right) + \mathbb{E}_{q_{\phi_{\alpha}}}\left[E_{\alpha}(z_a,z_b,a)\right] + \log F_{\alpha}(z_b,a)\big\},
\end{aligned}
\end{equation}
which gives the exact form as Energy-ELBO described in Eq.~\eqref{method_eq:energy_kl}. Therefore, we fully prove Theorem~\ref{method_theory:theorem_1}.
\end{proof}

\subsubsection{Proof of Lemma 1: Gradient of the partition function}
\label{app:theory_partition_function_gradient}
Here we give a proof to Lemma~\ref{method_theory:lemma_1} as follows:
\begin{proof}
\begin{equation}
\begin{aligned}
\nabla_{\alpha} \log F_{\alpha}(z_b,a) = \frac{1}{F_{\alpha}(z_b,a)} \nabla_{\alpha} F_{\alpha}(z_b,a)
\end{aligned}
\end{equation}
According to the definition in Eq.~\eqref{method_eq:partition_function}, the partition function is:
\begin{equation}
F_{\alpha}(z_b,a)=\int p_0(z_a)\exp\left(-E_{\alpha}(z_a,z_b,a)\right)\mathrm{d}z_a.
\end{equation}
Thus, we have:
\begin{equation}
\begin{aligned}
\nabla_{\alpha} \log F_{\alpha}(z_b,a) & = \frac{1}{F_{\alpha}(z_b,a)} \nabla_{\alpha} \int p_0(z_a)\exp\left(-E_{\alpha}(z_a,z_b,a)\right)\mathrm{d}z_a \\
& = \frac{1}{F_{\alpha}(z_b,a)} \int p_0(z_a) \nabla_{\alpha} \exp\left(-E_{\alpha}(z_a,z_b,a)\right)\mathrm{d}z_a \\
& = \frac{1}{F_{\alpha}(z_b,a)} \int p_0(z_a) \exp\left(-E_{\alpha}(z_a,z_b,a)\right) [-\nabla_{\alpha}E_{\alpha}(z_a,z_b,a)]\mathrm{d}z_a \\
& = \int \frac{p_0(z_a)\exp\left(-E_{\alpha}(z_a,z_b,a)\right)}{F_{\alpha}(z_b,a)} [-\nabla_{\alpha}E_{\alpha}(z_a,z_b,a)]\mathrm{d}z_a.
\end{aligned}
\end{equation}

According to the our conditional energy form described in Eq.~\eqref{method_eq:ebm_prior}, which is:
\begin{equation}
p_{\alpha}(z_a|z_b,a) = \frac{p_0(z_a)\exp\left[-E_{\alpha}(z_a,z_b,a)\right]}{F_{\alpha}(z_b,a)}.
\end{equation}
Therefore, we have:
\begin{equation}
\begin{aligned}
\nabla_{\alpha} \log F_{\alpha}(z_b,a) & = \int p_{\alpha}(z_a|z_b,a) [-\nabla_{\alpha}E_{\alpha}(z_a,z_b,a)]\mathrm{d}z_a \\
& = - \mathbb{E}_{p_{\alpha}(z_a |z_b,a)}\left[\nabla_{\alpha}E_{\alpha}(z_a,z_b,a)\right].
\end{aligned}
\end{equation}

Therefore, we finish the proof of Lemma~\ref{method_theory:lemma_1}.
\end{proof}

\subsubsection{Derivation of the gradient to energy parameters}
\label{app:derivation_of_energy_gradient}
Here we derive the gradient of the entire ELBO corresponding to energy parameters. 
\begin{proof}

According to the formation of $log\, p_{\Theta}(x|c,a)$ ELBO described in Eq.~\eqref{method_eq:entire_elbo}, we have:
\begin{equation}
\begin{aligned}
\nabla_{\alpha} \mathcal{L}_{\mathrm{ELBO}} & = \nabla_{\alpha} \mathbb{E}_{q_{\phi}}\left[log \, p_{\theta}(x|z_b,z_a)\right] - \nabla_{\alpha} D_{\mathrm{KL}}(q_{\phi_\beta}||p_{\beta}) + \nabla_{\alpha} \mathbb{E}_{q_{\phi_\beta}}\left[-D_{\mathrm{KL}}\left(q_{\phi_{\alpha}}||p_{\alpha}\right)\right] \\
& = \mathbb{E}_{q_{\phi_\beta}} \left[-\nabla_{\alpha} D_{\mathrm{KL}}\left(q_{\phi_{\alpha}}||p_{\alpha}\right)\right].
\end{aligned}
\end{equation}

According to the Energy-ELBO definition in Eq.~\eqref{method_eq:energy_kl}, we then have:
\begin{equation}
\begin{aligned}
\nabla_{\alpha} \mathcal{L}_{\mathrm{ELBO}} & = \mathbb{E}_{q_{\phi_\beta}} \left[-\nabla_{\alpha} D_{\mathrm{KL}}\left(q_{\phi_{\alpha}}||p_{\alpha}\right)\right] \\
& = \mathbb{E}_{q_{\phi_\beta}} \{ -\nabla_{\alpha} D_{\mathrm{KL}}\left(q_{\phi_{\alpha}}(z_a |z_b,x_{ctr},a)||p_0(z_a)\right) - \mathbb{E}_{q_{\phi_{\alpha}}}\left[ \nabla_{\alpha} E_{\alpha}(z_a,z_b,a)\right] -  \nabla_{\alpha} \log F_{\alpha}(z_b,a) \}.
\end{aligned}
\end{equation}
Since $D_{\mathrm{KL}}\left(q_{\phi_{\alpha}}(z_a |z_b,x_{ctr},a)||p_0(z_a)\right)$ is unrelated to $\alpha$, together with Lemma~\ref{method_theory:lemma_1}, where:
\begin{equation}
\begin{aligned}
\nabla_{\alpha} \log F_{\alpha}(z_b,a) & = - \mathbb{E}_{p_{\alpha}(z_a |z_b,a)}\left[\nabla_{\alpha}E_{\alpha}(z_a,z_b,a)\right],
\end{aligned}
\end{equation}
we have:
\begin{equation}
\begin{aligned}
\nabla_{\alpha} \mathcal{L}_{\mathrm{ELBO}} = \mathbb{E}_{q_{\phi_\beta}} \{-\mathbb{E}_{q_{\phi_{\alpha}}}\left[\nabla_{\alpha} E_{\alpha}(z_a,z_b,a)\right] + \mathbb{E}_{p_{\alpha}}\left[\nabla_{\alpha}E_{\alpha}(z_a,z_b,a)\right]\}.
\end{aligned}
\end{equation}
Therefore, we finish the derivation of the gradient corresponding to energy parameters $\alpha$.
\end{proof}

\subsubsection{Proof of Theorem 2: Conditional Mutual-Information Lower Bound}
\label{app:theory_conditional_mutual_information_lower_bound}
Here we give a proof to Theorem~\ref{method_theory:theorem_2} as follows:

\begin{proof}
According to our definition of conditional energy-based prior in Eq.~\eqref{method_eq:ebm_prior}, we have:
\begin{equation}
p_{\alpha}(z_a|z_b,a) \propto p_0(z_a)\exp\left[-E_{\alpha}(z_a,z_b,a)\right].
\label{app_eq:proof_4_eq_1}
\end{equation}
We denote the real perturbation prior as: $\pi(a|z_b,c)$, according to Bayes rule, we obtain:
\begin{equation}
p_{\alpha}(a|z_a,z_b) \propto \pi(a|z_b,c) p_{\alpha}(z_a|z_b,a).
\label{app_eq:proof_4_eq_2}
\end{equation}
Combining Eq.~\eqref{app_eq:proof_4_eq_1} and Eq.~\eqref{app_eq:proof_4_eq_2}, and eliminating $p_0(z_a)$, we obtain:
\begin{equation}
r_{\alpha}(a|z_a,z_b,c) = \frac{\pi(a|z_b,c) \exp (-E_{\alpha}(z_a,z_b,a))}{\sum_{a' \in \mathcal{A}} \pi(a'|z_b,c) \exp (-E_{\alpha}(z_a,z_b,a'))}.
\label{app_eq:proof_thery_2_r_definition}
\end{equation}
From the perspective of mutual information, we consider conditional mutual information $I(A;Z_a|Z_b,C)$, having:
\begin{equation}
\begin{aligned}
I(A;Z_a|Z_b,C) = H(A|Z_b,C) - H(A|Z_a,Z_b,C),
\end{aligned}
\end{equation}
where $H(.|Z_b,C)$ and $H(.|Z_a,Z_b,C)$ denote the conditional entropy. According to the definition of entropy, we have:
\begin{equation}
 - H(A|Z_a,Z_b,C) = \mathbb{E}\left[ log \, p(a|z_a,z_b,c) \right],
\end{equation}
where $p(a|z_a,z_b,c)$ is the real perturbation posterior distribution. Therefore we have:
\begin{equation}
\begin{aligned}
I(A;Z_a|Z_b,C) = H(A|Z_b,C) + \mathbb{E}\left[ log \, p(a|z_a,z_b,c) \right].
\end{aligned}
\end{equation}
We use $r_{\alpha}(a|z_a,z_b,c)$ to approximate $p(a|z_a,z_b,c)$, and because the property of KL-divergence that $D_{KL}(p||q) \geq 0$, we have:
\begin{equation}
\begin{aligned}
D_{KL}(p(a|z_a,z_b,c)||r_{\alpha}(a|z_a,z_b,c)) \geq 0,
\end{aligned}
\end{equation}
Therefore, we obtain:
\begin{equation}
\begin{aligned}
I(A;Z_a|Z_b,C) & \geq H(A|Z_b,C) + \mathbb{E}\left[\log r_{\alpha}(A |Z_a,Z_b,C)\right],
\end{aligned}
\end{equation}
where the equality holds when $r_{\alpha}(a|z_a,z_b,c)$ correctly approximates the real distribution $p(a|z_a,z_b,c)$.

Thus, we have proved the formation of Theorem~\ref{method_theory:theorem_2}.

Next, we will talk about the relationship between the energy optimization and the conditional mutual information. 

According to our optimization realized in Subsection~\ref{method_subsec:energy_optimization}, the energy function $E_{\alpha}(z_a,z_b,a)$ assigns lower energy (higher probability mass) to the compatible perturbation representations and higher energy to incompatible ones (i.e., For the definition of $r_{\alpha}(a|z_a,z_b,c)$ in Eq.~\eqref{app_eq:proof_thery_2_r_definition}, $E_{\alpha}(z_a,z_b,a)$ becomes lower while $E_{\alpha}(z_a,z_b,a')$ becomes higher). Therefore, the optimization of energy function leads to the increase of $\mathbb{E}\left[\log r_{\alpha}(A |Z_a,Z_b,C)\right]$, which means \textbf{Optimizing PerturbEnergy leading to the increase of the conditional mutual information lower bound between the perturbation and the latent representation.}

Until now, We have fully proved Theorem~\ref{method_theory:theorem_2}.
\end{proof}

\subsection{Algorithm Description}

\subsubsection{Langevin Sampling Algorithm}
\label{app_alg:langevin}
We provide the Langevin sampling in Algorithm~\ref{app_alg:langevin_sampling_alg_procedure}.
\begin{algorithm}[h!]
\caption{Langevin Sampling}
\label{app_alg:langevin_sampling_alg_procedure}
\begin{algorithmic}[1]
\REQUIRE Cellular expression $x_{ctr}$, basal cellular representation $z_b$, perturbation $a$, energy function $E_{\alpha}(z_a,z_b,a)$, initial model $q_{\phi_{\alpha}}(z_a|z_b,x_{ctr},a)$, Langevin steps $K$, step size $\eta$
\ENSURE Negative sample $z_a^{-}$

\STATE Initialize $z_a^{(0)} \sim q_{\phi_{\alpha}}(z_a|z_b,x_{ctr},a)$

\FOR{$k = 0,1,\dots,K-1$}
    \STATE Sample Gaussian noise $\epsilon^{(k)} \sim \mathcal{N}(0,I)$
    \STATE Update $z_a^{(k)}$ via Langevin dynamics:
    \[
    z_a^{(k+1)}
    =
    z_a^{(k)}
    -
    \eta \nabla_{z_a}
    E_{\alpha}\left(z_a^{(k)},z_b,a\right)
    +
    \sqrt{2\eta}\,\epsilon^{(k)}
    \]
\ENDFOR

\STATE Set $z_a^{-} = z_a^{(K)}$

\RETURN $z_a^{-}$
\end{algorithmic}
\end{algorithm}

\subsubsection{Optimization Algorithm}
\label{app_alg:optimization}
We provide the optimization process in Algorithm~\ref{app_alg:optimization_alg_procedure}
\begin{algorithm}[h!]
\caption{Optimization Algorithm}
\label{app_alg:optimization_alg_procedure}
\begin{algorithmic}[1]
\REQUIRE Batch cell data $(x_{ctr}, x)$ under the same context $c$ with perturbation $a$, model parameters $\{\theta, \alpha, \phi_\beta, \phi_\alpha\}$, sample size $N$, $\alpha$ learning rate $\zeta$ , and $\{\theta,\phi_\beta, \phi_\alpha\}$ learning rate $\xi$
\ENSURE Updated model parameters $\{\theta, \alpha, \phi_\beta, \phi_\alpha\}$

\STATE Sample batch basal representations $z_{b}$ from $q_{\phi_{\beta}}(z_b|x_{ctr},c)$

\STATE Generate $N$ positive batch samples $\{z_{a,i}^{+}\}_{i=1}^{N}$ via $q_{\phi_{\alpha}}(z_a| z_b,x,a)$

\STATE Generate $N$ negative batch samples $\{z_{a,i}^{-}\}_{i=1}^{N}$ via Langevin sampling in Algorithm~\ref{app_alg:langevin_sampling_alg_procedure} or via $q_{\phi_{\alpha}}(z_a |z_b,x_{ctr},a^{-})$ with random other $a^{-}$

\STATE Calculating $\nabla_{\alpha} \mathcal{J}_{\mathrm{E}}(\alpha)$ via Eq.~\eqref{method_eq:j_gradient}

\STATE Calculating $\{\nabla_{\theta}\mathcal{L}_{\mathrm{ELBO}}, \nabla_{\phi_\beta}\mathcal{L}_{\mathrm{ELBO}},\nabla_{\phi_\alpha}\mathcal{L}_{\mathrm{ELBO}}\}$ via Eq.~\eqref{method_eq:elbo_derivation}

\STATE Update $\alpha$ via: $\alpha \leftarrow \alpha - \zeta \nabla_{\alpha} \mathcal{J}_{\mathrm{E}}(\alpha)$

\STATE Update $\{\theta, \phi_\beta, \phi_\alpha\}$ via: $\theta/\phi_\beta/\phi_\alpha  \leftarrow \theta/\phi_\beta/\phi_\alpha + \xi \nabla_{\theta/\phi_\beta/\phi_\alpha} \mathcal{L}_{\mathrm{ELBO}}$
\end{algorithmic}
\end{algorithm}

\subsubsection{Inference Algorithm}
\label{app_alg:inference}
We provide the inference process in Algorithm~\ref{app_alg:inference_alg_procedure}

\begin{algorithm}[h!]
\caption{Inference Algorithm}
\label{app_alg:inference_alg_procedure}
\begin{algorithmic}[1]
\REQUIRE Unperturbed cell $(x_{ctr},c)$, perturbation $a$, model parameters $\{\theta,\alpha,\phi_\beta,\phi_\alpha\}$, number of samples $M$
\ENSURE Predicted cellular response $x$

\STATE Sample basal representation $z_b$ from $q_{\phi_\beta}(z_b|x_{ctr},c)$

\STATE Sample $M$ perturbation representations $\{z_{a,i}\}_{i=1}^{M}$ from $q_{\phi_\alpha}(z_a|z_b,a)$

\STATE Compute energy scores $\{E_{\alpha}(z_{a,i},z_b,a)\}_{i=1}^{M}$

\STATE Select
$
z_a^{*}
=
\arg\min_{z_{a,i}}
E_{\alpha}(z_{a,i},z_b,a)
$

\STATE Generate predicted cellular response $x$ from $p_{\theta}(x|z_b,z_a^{*})$

\RETURN $x$
\end{algorithmic}
\end{algorithm}

\subsection{Experimental Appendix}

\subsubsection{Dataset Descriptions and Settings}
\label{app_exp_part:dataset_description}

\begin{table}[h!]
\centering
\caption{\textit{(a)} Perturbation datasets. Statistics include counts of perturbed cell
across train/valid/test splits and counts of control cells, total gene vocabulary size and
HVGs, and dataset diversity in terms of perturbations, cell types, and experimental batches.}
\label{app_tab:dataset_info}
\resizebox{\textwidth}{!}{%
\begin{tabular}{lllllllllll}
\toprule
\multirow{2}{*}{Dataset} & \multirow{2}{*}{Perturbation Type} & \multicolumn{3}{c}{\#Perturbed Cells} & \multirow{2}{*}{\#Control Cells} & \multirow{2}{*}{\#Genes} & \multirow{2}{*}{\#HVGs} & \multirow{2}{*}{\#Perturbations} & \multirow{2}{*}{\#Cell Types} & \multirow{2}{*}{\#Batches} \\
\cmidrule(lr){3-5}
 & & Train & Valid & Test & & & & & & \\
\midrule
Replogle    & gene     & 572,545    & 4,825   & 26,878    & 39,165    & 6,642  & 2,000 & 2,023 & 4  & 56 \\
PBMC        & cytokine & 6,657,336  & 150,484 & 2,260,453 & 629,701   & 62,710 & 2,000 & 90    & 18 & 12 \\
Norman      & gene     & 85,568     & 1,754   & 12,268    & 11,855    & 33,694 & 2,000 & 236   & 1  & 8  \\
Tahoe100M   & drug     & 89,495,239 & 553,076 & 8,790,272 & 2,330,156 & 40,352 & 2,000 & 1,137 & 50 & 14 \\
\bottomrule
\end{tabular}%
}
\end{table}

We evaluate on four widely used single-cell perturbation benchmarks that together span
signaling, chemical and genetic interventions: Replogle~\cite{replogle2022mapping}, PBMC~\cite{yuan2026perturbdiff}, Norman2019~\cite{norman2019exploring} and Tahoe100M~\cite{zhang2025tahoe}. Following
PerturbDiff~\cite{yuan2026perturbdiff}, STATE~\cite{adduri2026predicting} and VC Benchmark~\cite{mao2026benchmarking}, we adopt the same data sources, preprocessing pipeline and held-out splits. Summary statistics are given in Table.~\ref{app_tab:dataset_info}.

\noindent \textbf{Replogle (genetic perturbation).}
The Replogle dataset~\cite{replogle2022mapping} is a landmark resource for studying
single-cell genetic perturbation effects. It integrates multiple pooled CRISPR interference
(CRISPRi)~\cite{gilbert2013crispr,nadig2025transcriptome} experiments across four human cell lines: K562, RPE1,
Jurkat and HepG2. It systematically repressing thousands of essential genes while measuring
transcriptome-wide responses. Because the targeted genes are essential, a single knockdown often
induces broad transcriptional changes that affect hundreds of downstream genes.

\noindent \textbf{PBMC (cytokine signaling).}
The PBMC dataset~\cite{yuan2026perturbdiff} is a large-scale signaling perturbation atlas of
peripheral blood mononuclear cells. It profiles approximately 10 million human cells collected
from 12 healthy donors, covering 18 immune cell types and 90 distinct cytokine treatments
alongside a PBS control condition. After quality control, about 9.7 million cells are retained,
with a mean sequencing depth of roughly 31,000 reads per cell. Many cytokines elicit strong
transcriptional responses, typically upregulating more than 50 genes, which makes this dataset a
representative testbed for studying perturbation-induced shifts at the population level.

\noindent \textbf{Tahoe100M (drug perturbation).}
Tahoe100M~\citep{zhang2025tahoe} is the largest single-cell perturbation atlas released to date.
It comprises over 100 million single-cell expression profiles measured across 50 human cancer
cell lines exposed to more than 1,100 combinations of small-molecule drug and dosage. The atlas adopts a multiplexed cell-village design during sequencing, which substantially reduces batch effects and yields high data quality. Its scale and the breadth of its cellular contexts make it a demanding benchmark for context-dependent drug response prediction.

\noindent \textbf{Norman2019 (combinatorial genetic perturbation).}
The Norman2019 dataset~\cite{norman2019exploring} is a Perturb-seq screen of combinatorial CRISPR activation (CRISPRa) in K562 cells. Its value lies not in the diversity of cellular contexts but in the structure of its perturbation space: it comprises 105 single-gene activations and 131 two-gene combinations, and both constituent genes of every combination are also activated individually elsewhere in the screen. This makes it the standard benchmark for combinatorial generalization, meaning predicting the transcriptional outcome of a gene pair from the effects of its individual members. Such predictions are far from trivial, because gene pairs frequently interact non-additively: two genes whose individual effects are mild may act synergistically, and a naive additive model that simply sums their single-gene effects fails to recover the resulting phenotype. We use the filtered version of the dataset as VC Benchmark~\cite{mao2026benchmarking}, which contains 111,445 cells.

\paragraph{Preprocessing.} We apply the same preprocessing pipeline as PerturbDiff~\cite{yuan2026perturbdiff}, STATE~\cite{adduri2026predicting} and VC Benchmark~\cite{mao2026benchmarking}. For PBMC and Tahoe100M, raw counts are first normalized per cell to a fixed library size($10^4$), followed by a log1p transformation to rescale the gene expression values into range $[0,1)$. For Replogle, we directly adopt the processed data released by STATE~\cite{adduri2026predicting}, which first filters cells and perturbations by on-target knockdown efficacy to improve the signal-to-noise ratio of perturbation effects, and
then applies the same library-size normalization and log transformation. After normalization, we
follow standard practice~\cite{wolf2018scanpy} and select the top 2,000 highly variable genes (HVGs) for each dataset
from its original transcriptome vocabulary using \texttt{scanpy.pp.highly\_variable\_genes}.
Excluding low-variance genes improves the signal-to-noise ratio without sacrificing biological
fidelity, and all models are trained and evaluated in this 2,000-dimensional HVG space. Norman2019 is distributed as raw counts and no curated gene panel is available for it, so we apply the same normalization pipeline and then select the top 2,000 HVGs from its 33,694-gene vocabulary.

\paragraph{Train/validation/test splits.}
For Replogle, PBMC and Tahoe100M, we follow the hold-out protocol of PerturbDiff~\cite{yuan2026perturbdiff} and STATE~\cite{adduri2026predicting}, which measures generalization to unseen cellular contexts. For each dataset, a subset of contexts is held out for testing, and 30\% of the perturbations measured in those held-out contexts are moved back into the training set. This partial-coverage design ensures that every perturbation is observed somewhere during training, so that the test set isolates
adaptation to a new cellular context rather than extrapolation to an unknown perturbation label.
Concretely:
\begin{itemize}
  \item \textbf{PBMC.} Four of the twelve donors are reserved for testing. 30\% of the
        perturbations from these held-out donors are moved into training, leaving the remaining
        70\% for testing. Note that the model conditions on cell type (18 categories), while
        the hold-out is taken over donors; at test time the control cells are therefore drawn
        from the held-out donors themselves.
  \item \textbf{Tahoe100M.} PCA is first performed on the pseudobulk expression profiles of the
        50 cell lines to identify distinct phenotypic clusters, from which five cell lines are
        selected as the held-out test set, again with 30\% of their perturbations returned to
        training.
  \item \textbf{Replogle.} One cell line (HepG2) is held out for testing, and the model is
        trained on the remaining three cell lines together with 30\% of the perturbations
        randomly sampled from HepG2.
\end{itemize}

\textbf{Norman2019} contains a single cell line, so holding out a cellular context is degenerate. We
therefore split over the perturbations themselves, following the VC benchmark~\cite{mao2026benchmarking}: 30\% of the two-gene combinations (39 of 131) are held out as the test set, while all 105 single-gene perturbations remain in training. This is the ``0 of 2 unseen'' generalization setting. Each held-out combination is never observed during training, yet both of its constituent genes have been perturbed individually, which isolates the ability to compose known single-gene effects into an unseen combination, rather than confounding it with extrapolation to genes for which no perturbation data exist at all. It is also what keeps the held-out combinations representable by our model: we embed a perturbation as the sum of the embeddings of its target genes, so an unseen pair is encoded from the two single-gene rows that training has already fitted. A per-label lookup table, by contrast, would have no entry for such a pair and could not make a prediction at all.

Control cells remain in the training pool under all three protocols: they are the unperturbed
reference population that the model is given at inference time, and they are therefore counted
separately from the train/validation/test splits. Finally, we hold out 2\% of the training cells
uniformly at random as a validation set. This set is used only for checkpoint selection; all
reported results come from the held-out contexts in the test split.

\subsubsection{Baseline Settings}
\label{app_exp_part:baseline_settings}

\subsubsubsection{Baseline models}

\noindent \textbf{Linear methods.}
The linear baseline~\cite{ahlmann2024deep} predicts a perturbed cell as the control cell plus the mean
expression shift of that perturbation measured on the training split,
$\hat{x} = x_{ctrl} + \Delta_a$, with the shift $\Delta_a$ taken against the
controls of the contexts in which the perturbation was observed. Despite its simplicity,
this predictor is a strong baseline on average-expression metrics and serves as the linear
reference point of the comparison. \textbf{scGPT}~\citep{cui2024scgpt} replaces the raw
expression profile with a pretrained cell representation: we use the released
\texttt{scGPT\_human} checkpoint (12 layers, $d_{\text{model}}{=}512$, 8 heads, 51.9M
parameters, pretrained on 33M cells), keep it frozen in evaluation mode, and train
only a lightweight head on top, $\hat{x} = x_{\mathrm{ctrl}} + \mathrm{MLP}([u-z])$, where
$z$ is the \texttt{<cls>} embedding of the control cell and $u$ the perturbation
embedding. Cells are tokenised exactly as during scGPT's pretraining: gene symbols mapped
through its 60{,}697-entry vocabulary, nonzero genes only, per-cell quantile binning into
51 bins, and a prepended \texttt{<cls>} token. The head's output layer is zero-initialised,
so training starts from ``predict no change''. This row answers a narrower question than
the others: whether a representation pretrained on external data already encodes how a cell
responds to a perturbation.

\noindent \textbf{VAE-based methods.} All the VAE-based methods are implemented via their official codes. \textbf{scGen}~\cite{lotfollahi2019scgen} trains an unconditional VAE on all cells. After training, the perturbation vector $\Delta_a = \mathrm{mean}\{z_a\} - \mathrm{mean}\{z_{ctr}\}$ is measured on the training split, where $z_a$ is the representation of perturbed cells and $z_{ctr}$ is the representations of controlled cells. And prediction is $\hat{x} = \mathrm{Dec}(\mathrm{Enc}(x_{\mathrm{ctrl}}) + \Delta_a)$.
\textbf{CPA}~\citep{lotfollahi2023predicting} decomposes a cell into a basal state plus additive perturbation and covariate embeddings, $z = \mathrm{Enc}(x) + u_a + u_c$, and enforces the decomposition with adversarial discriminators trained to disentangle the perturbation and the covariates of the basal latent. Prediction swaps the perturbation embedding to the corresponding one.
\textbf{stVAE}~\citep{russkikh2020style} is a style-transfer VAE: expression is decoupled into a perturbation-independent content latent and a perturbation-specific style vector, and prediction decodes the sum of the control cell's content and the target perturbation's style.
\textbf{PerturbedVAE}~\citep{jiang2026makes} splits the latent into a
perturbation-invariant part $z_i$ and a perturbation-responsive part $z_v$, and makes the
split persist with a contrastive term that anchors $z_i$ to the corresponding control cell.
The prediction mechanism is a learned conditional prior $p(z_v \mid z_i, u)$, which supplies
the response at inference time. As their official implementation suggests, CPA, stVAE, and PerturbedVAE are trained on observed cells and scGen is trained unconditionally. For a fair comparison, we adopt their official VAE architecture but unify the latent embedding size the same as ours: 512.

\noindent \textbf{Diffusion-based methods.}
\textbf{PerturbDiff}~\citep{yuan2026perturbdiff} follows its official implementation, it runs a denoising diffusion process over
the cell population rather than over individual cells. The denoiser is an MM-DiT
stack (4 blocks, 8 heads, latent dimension is setted as 512) with joint attention over the control and perturbed token
streams, a single shared timestep noises the whole population so that the diffused object
remains a distribution, and the denoising objective is the MMD between the denoised and the
real population. Inference uses 50 DDIM steps as it general leads to excellent performance.
\textbf{Squidiff}~\citep{he2026squidiff} runs a conditional DDIM in a VAE semantic latent
(dimension of $z$ is adjusted to 512 as other models) with an MLP denoiser and sinusoidal timestep embeddings, and decodes
the sampled latent. Its autoencoder is warmed-up trained alone for the first 4,000 steps, after
which the encoder is frozen so that the latent diffusion trains against a fixed target.

\noindent \textbf{Flow-based methods.}
\textbf{CellFlow}~\cite{klein2025cellflow} learns a conditional velocity field that
transports the control population to the perturbed one. Training pairs the two minibatches
by exact minibatch optimal transport under a squared-Euclidean cost, interpolates
$x_t = (1-t)x_0 + t x_1$ with $t \sim \mathcal{U}(0,1)$, and regresses
$v_\theta(x_t, t, \text{cond})$ onto $x_1 - x_0$, where the conditioning combines the
perturbation and context embeddings with the pooled control population. Inference
integrates the field from $t{=}0$ to $t{=}1$ with a Heun solver in 50 steps.

\noindent \textbf{Transformer-based method.}
\textbf{STATE}~\cite{adduri2026predicting} is a bidirectional transformer over sets of cells, where the perturbation embedding is added to every token and the model predicts the transition from each control cell. We reproduce the State Embedding (SE), State Transition (ST), and adapt the embedding size to 512, where SE uses ESM2 embedding and project into 512 dimension. For Replogle, PBMC and Tahoe100M, we use the the same training dataset for SE: CellxGene~\cite{czi2025cz} with the same reconstruction target. The learning rate is set as $1\times 10^{-5}$, weight decay is set as 0.01, batch size is set as 128 and epoch number is 12. As for Norman2019, SE is pretrained on its training set with the same training settings but 5 epochs.

\subsubsubsection{Training settings}

All the baselines follow the same data preprocess and split(train/validation/test) settings as described in Appendix~\ref{app_exp_part:dataset_description}. We use the same condition sampler and random seed for all the methods: A batch is $n$ blocks of \texttt{set\_size} cells, where each block pairs control cells of a context with real perturbed cells of one (context, perturbation) condition. Every model therefore sees the same 512 cells per step (256 on Norman2019). As for the maximum training step budget, we set 20,000 steps on Norman2019, 40,000 on Replogle, 60,000 on PBMC and 160,000 on Tahoe100M for all the methods. All models are optimised with Adam (weight decay: $10^{-6}$) under their own best learning rate. The loss function and corresponding settings of each baseline follows its official suggestion. A cosine learning-rate decay is used for PerturbDiff, Squidiff, CellFlow and STATE, and a constant rate otherwise. For the two diffusion baselines and CellFlow we additionally keep an exponential moving average (EMA) of the weights (decay 0.999) and validate, checkpoint and evaluate that average, since EMA is part of how diffusion and flow-matching models are trained in practice and their reported sample quality is not what the raw iterate gives. Model-specific settings are listed in Table.~\ref{app_tab:baseline_info}.

\begin{table}[t]
\caption{\textbf{Model-specific training settings.} All other settings including batch shape, step
budget, optimiser family, gradient clipping, checkpoint selection and evaluation are
shared across models and identical to those used for PerturbEnergy.}
\label{app_tab:baseline_info}
\vskip 0.15in
\begin{center}
\begin{small}
\begin{tabular}{llccl}
\toprule
Model & Family & LR & EMA & Model-specific settings \\
\midrule
Linear       & linear        & ---                & ---   & closed form; no gradient training \\
scGPT        & foundation    & $1\!\times\!10^{-3}$ & ---   & encoder frozen; head $512{+}512\!\to\!1024\!\to\!G$ \\
\midrule
scGen        & VAE           & $1\!\times\!10^{-3}$ & ---   & $z=512$, hidden $800{\times}2$, no conditioning \\
CPA          & VAE           & $1\!\times\!10^{-3}$ & ---   & $z=512$, adv.\ weight 50, 3 adversary steps \\
stVAE        & VAE           & $1\!\times\!10^{-3}$ & ---   & $z=512$, latent loss 1.0, VAE loss 0.1 \\
PerturbedVAE & VAE           & $1\!\times\!10^{-3}$ & ---   & $z_i=z_v=512$, contrastive weight 1.0 \\
\midrule
PerturbDiff  & diffusion     & $3\!\times\!10^{-4}$ & 0.999 & $z=512$, 4 MM-DiT blocks, 8 heads, $T=1000$ \\
Squidiff     & diffusion     & $1\!\times\!10^{-3}$ & 0.999 & $z=512$, $T=1000$, 4k-step AE warm-up \\
CellFlow     & flow matching & $1\!\times\!10^{-3}$ & 0.999 & Solver 50 steps, depth 4, OT coupling \\
STATE        & transformer   & $3\!\times\!10^{-4}$ & ---   & $z=512$, depth 4, 8 heads, set size 128 \\
\bottomrule
\end{tabular}
\end{small}
\end{center}
\vskip -0.1in
\end{table}

\subsubsubsection{Validation and checkpoint selection}

Validation is identical for every baseline model and PerturbEnergy. Every 2,500 steps we validate on the validation batches where the score is the Mean Squared Error (MSE) of the predicted perturbation mean against the real one. The checkpoint with the lowest validation score is the one evaluated on the test split; for models with EMA, the averaged weights are the ones both validated and saved.

\subsubsection{Metrics}
\label{app_exp:metrics_description}

\subsubsubsection{Metric Definitions}

\noindent \textbf{Preliminary Notation for Metrics Description.}
A \emph{condition} $\lambda=(c,a)$ pairs a cellular context $c$ with a perturbation $a$;
the test split contains $M$ such conditions, and each metric below is computed per
condition and then averaged over the $M$ of them. For condition $\lambda$, we note
$X^{\lambda}\in\mathbb{R}^{N\times G}$ for the real perturbed cells,
$\widetilde{X}^{\lambda}\in\mathbb{R}^{n\times G}$ for the control cells of the same
context $c$, and $\hat{X}^{\lambda}\in\mathbb{R}^{n\times G}$ for the model's perturbation
prediction, over $G=2{,}000$ highly variable genes indexed by $g\in\{1,\dots,G\}$. A
subscript $i$ on a matrix denotes its $i$-th row, i.e.\ a single cell in
$\mathbb{R}^{G}$. Inference maps control cells to perturbed ones, so the prediction
contains the same number of cells as the control population, $n$. Pseudobulk profiles are
the column means, which are vectors in $\mathbb{R}^{G}$:
\begin{equation}
\mu=\tfrac{1}{N}\textstyle\sum_{i}X^{\lambda}_{i},\qquad
\widetilde{\mu}=\tfrac{1}{n}\textstyle\sum_{i}\widetilde{X}^{\lambda}_{i},\qquad
\hat{\mu}=\tfrac{1}{n}\textstyle\sum_{i}\hat{X}^{\lambda}_{i},
\end{equation}
the ground true perturbation effect is:
\begin{equation}
\Delta=\mu-\widetilde{\mu},
\end{equation}
the model predicted perturbation effect is:
\begin{equation}
\hat{\Delta}=\hat{\mu}-\widetilde{\mu}.
\end{equation}
Since the data is log1p-normalised, $\Delta$ and $\hat{\Delta}$ are log fold
changes.

\noindent \textbf{Differentially expressed gene (DEG).}
The DEG set is derived from the real data so all methods are scored on the same genes. Each gene is tested with a two-sided Wilcoxon rank-sum (Mann--Whitney $U$) test between $X^{\lambda}$ and $\widetilde{X}^{\lambda}$;
genes are ordered by ascending $p$-value with $|\Delta_g|$, and $\mathcal{S}_k$ denotes the top-$k$ genes.

\subsubsubsection{HVG Group Metrics}

\noindent \textbf{Correlation-based metrics.} These metrics ask how closely the predicted expression tracks the ground truth (correlation), which has two forms: measured on the expression level ($\hat{\mu}$ against $\mu$) corresponding to the overall transcriptional state, and measured on the perturbation effect ($\hat{\Delta}$ against $\Delta$). The metrics we use are defined as follows:

\textbf{PCC:} Pearson correlation between the predicted and real pseudobulk profiles;
measures agreement in overall expression level. Let $\mathrm{corr}(\cdot,\cdot)$ denote the Pearson correlation:
\begin{equation}
\mathrm{PCC}=\mathrm{corr}\!\left(\hat{\mu},\mu\right)
=\frac{\sum_{g}(\hat{\mu}_g-\bar{\hat{\mu}})(\mu_g-\bar{\mu})}
{\sqrt{\sum_{g}(\hat{\mu}_g-\bar{\hat{\mu}})^{2}}\sqrt{\sum_{g}(\mu_g-\bar{\mu})^{2}}}
\end{equation}

\textbf{PDCorr:} the same Pearson correlation taken on the perturbation effect rather than the level, so an accurate baseline expression cannot inflate it. Here $\bar{\Delta}$ and $\bar{\hat{\Delta}}$ are the scalar means of the two effect vectors.
\begin{equation}
\mathrm{PDCorr}=\mathrm{corr}\!\left(\hat{\Delta},\Delta\right)
=\frac{\sum_{g}(\hat{\Delta}_g-\bar{\hat{\Delta}})(\Delta_g-\bar{\Delta})}
{\sqrt{\sum_{g}(\hat{\Delta}_g-\bar{\hat{\Delta}})^{2}}
 \sqrt{\sum_{g}(\Delta_g-\bar{\Delta})^{2}}}
\end{equation}

\textbf{R$^2$:} fraction of variance in the real profile explained by the prediction,
relative to the empirical mean baseline. 
\begin{equation}
R^{2}=1-\frac{\sum_{g}(\mu_g-\hat{\mu}_g)^{2}}{\sum_{g}(\mu_g-\bar{\mu})^{2}}
\end{equation}

\textbf{LFCSpear:} Spearman rank correlation $\rho_{s}$ of the log fold changes, that
is the Pearson correlation computed on the ranks of the entries instead of their values,
where $\operatorname{rank}(v)\in\mathbb{R}^{G}$ assigns each gene its position when $v$ is
sorted. Unlike PDCorr it depends only on the ordering of genes by effect, not on the
magnitudes.
\begin{equation}
\mathrm{LFCSpear}=\rho_{s}\!\left(\hat{\Delta},\Delta\right)
=\mathrm{corr}\!\left(\operatorname{rank}(\hat{\Delta}),\operatorname{rank}(\Delta)\right)
\end{equation}

\textbf{LogFC\_cos:} cosine similarity of the two effect vectors, i.e.\ whether the
response points in the right direction irrespective of its scale.
\begin{equation}
\mathrm{LogFC\_cos}=\frac{\langle\hat{\Delta},\Delta\rangle}
{\|\hat{\Delta}\|_{2}\,\|\Delta\|_{2}}
=\frac{\sum_{g}\hat{\Delta}_g\Delta_g}
{\sqrt{\sum_{g}\hat{\Delta}_g^{2}}\sqrt{\sum_{g}\Delta_g^{2}}}
\end{equation}

\textbf{DirAgr:} fraction of genes whose direction of change is predicted correctly.
The average is taken over $\mathcal{G}^{+}=\{g:\Delta_g\neq 0\}$: with log-normalised
counts many genes have a true effect of exactly zero in a small population, and those can
only be "agreed" with by predicting exactly zero, which would pull every model towards
the do-nothing baseline. Let $\operatorname{sign}(\cdot)\in\{-1,0,+1\}$ denotes the sign of a scalar and
$\mathbb{1}[\cdot]$ denotes the indicator function equal to one when its argument holds and zero
otherwise.
\begin{equation}
\mathrm{DirAgr}=\frac{1}{|\mathcal{G}^{+}|}\sum_{g\in\mathcal{G}^{+}}
\mathbb{1}\!\left[\operatorname{sign}(\hat{\Delta}_g)=\operatorname{sign}(\Delta_g)\right]
\end{equation}

\textbf{PDS-$\ell_1$/$\ell_2$/cos:} perturbation discrimination score. A prediction should resemble its own ground truth more than any other condition's. Let $\hat{\Delta}_{\lambda}$ be the predicted effect for condition $\lambda$ and $\Delta_{p}$ the real effect of condition $p$. The rank $r_{\lambda}\in\{0,\dots,M-1\}$ counts how many of the other $M-1$ ground truths lie closer to $\hat{\Delta}_{\lambda}$ than its own does, and dividing by $M-1$ normalises it to $[0,1]$:
\begin{equation}
\mathrm{PDS}=\frac{1}{M}\sum_{\lambda=1}^{M}\frac{r_{\lambda}}{M-1},
\qquad
r_{\lambda}=\sum_{p\neq\lambda}\mathbb{1}\!\left[d(\hat{\Delta}_{\lambda},\Delta_{p})
<d(\hat{\Delta}_{\lambda},\Delta_{\lambda})\right],
\end{equation}
with distance the $\ell_1$, $\ell_2$ or cosine distance. \textbf{Lower is better}: $0$ means every prediction is nearest its own ground truth and $0.5$ is chance. (Works reporting $1-\mathrm{PDS}$ invert the
direction.)

\noindent \textbf{DEG distinction metrics.} These metrics ask a different question from the correlation measures: whether the model can identify the correct DEGs. Let $\mathcal{S}_k$ be the true top-$k$ DE genes and $\hat{\mathcal{S}}_k$ the $k$ genes with the largest predicted $|\hat{\Delta}_g|$; we use $k=50$. These metrics include: DEPrec, DERec, DEOverlap, AUROC and AUPRC.

\textbf{DEPrec, DERec, DEOverlap:} precision, recall and Jaccard index of the retrieved DEG set, where $|\cdot|$ is set cardinality and $\cap,\cup$ are intersection and union. Because the two sets are built with the same size ($|\hat{\mathcal{S}}_k|=|\mathcal{S}_k|=k$), precision and recall coincide here.
\begin{equation}
\mathrm{DEPrec}=\frac{|\mathcal{S}_k\cap\hat{\mathcal{S}}_k|}{|\hat{\mathcal{S}}_k|},
\quad
\mathrm{DERec}=\frac{|\mathcal{S}_k\cap\hat{\mathcal{S}}_k|}{|\mathcal{S}_k|},
\quad
\mathrm{DEOverlap}=\frac{|\mathcal{S}_k\cap\hat{\mathcal{S}}_k|}
{|\mathcal{S}_k\cup\hat{\mathcal{S}}_k|}
\end{equation}

\textbf{AUROC, AUPRC:} Genes in $\mathcal{S}_k$ are labelled positive, all remaining $G-k$ genes negative, and every gene is scored by $|\hat{\Delta}_g|$; we report the area under the ROC curve and the average precision.

\noindent \textbf{Distance-based metrics.} These metrics evaluate the distance between the prediction and real expression. We use Root Mean Square Error (RMSE), Mean Square Error (MSE), Mean Absolute Error (MAE), Energy Distance (EDist) and Sinkhorn Divergence (Sinkhorn).

\textbf{RMSE, MSE, MAE:} error between the predicted and real pseudobulk.
\begin{equation}
\mathrm{MSE}=\frac{1}{G}\sum_{g}(\hat{\mu}_g-\mu_g)^{2},\quad
\mathrm{RMSE}=\sqrt{\mathrm{MSE}},\quad
\mathrm{MAE}=\frac{1}{G}\sum_{g}|\hat{\mu}_g-\mu_g|
\end{equation}

\textbf{EDist:} energy distance between the two populations, zero if and only if the
distributions coincide.
\begin{equation}
\mathrm{EDist}=\frac{2}{nN}\sum_{i,j}\|\hat{x}_i-x_j\|_{2}
-\frac{1}{n^{2}}\sum_{i,j}\|\hat{x}_i-\hat{x}_j\|_{2}
-\frac{1}{N^{2}}\sum_{i,j}\|x_i-x_j\|_{2}
\end{equation}

\textbf{Sinkhorn:} entropic optimal-transport divergence, debiased so that
$\mathrm{S}(X,X)=0$:
\begin{equation}
\mathrm{S}_{\varepsilon}(\hat{X},X)=\mathrm{OT}_{\varepsilon}(\hat{X},X)
-\tfrac{1}{2}\mathrm{OT}_{\varepsilon}(\hat{X},\hat{X})
-\tfrac{1}{2}\mathrm{OT}_{\varepsilon}(X,X),
\end{equation}
\begin{equation}
\mathrm{OT}_{\varepsilon}(A,B)=\min_{P\in\Pi(a,b)}\langle P,C\rangle
+\varepsilon\,\mathrm{KL}(P\,\|\,a b^{\top}),
\qquad C_{ij}=\|A_i-B_j\|_{2}^{2}/s,
\end{equation}
with uniform marginals, $\varepsilon=0.05$, and $s$ the mean squared pairwise distance
within the control population, which makes the value comparable across datasets. Both
distances are evaluated on a fixed random subsample of the populations (512 cells for
EDist, 256 for Sinkhorn) to bound their quadratic cost.

\subsubsubsection{DEG Group Metrics}

This metric group evaluates how well a model predict the biologically meaningful gene expression specifically influenced by the perturbation. Different from the HVG group describing the general expression pattern, the DEG group is directly related to the perturbation. Since perturbation leaves most of the transcriptome nearly untouched, so genome-wide averages are dominated by genes that carry no signal and a model can score well while missing the
response entirely. The DEG group therefore restricts the error and correlation measures to
the true top-$k$ DE genes, $k\in\{50,100,200\}$. With $\mathcal{S}_k$ fixed in advance from
the real data, this group no longer asks whether the responding genes were found. We calculate the RMSE, MAE, PCC and R$^2$ on the DEGs.
\begin{equation}
\mathrm{MSE}^{(k)}=\frac{1}{k}\sum_{g\in\mathcal{S}_k}(\hat{\mu}_g-\mu_g)^{2},
\quad
\mathrm{RMSE}^{(k)}=\sqrt{\mathrm{MSE}^{(k)}},
\qquad
\mathrm{MAE}^{(k)}=\frac{1}{k}\sum_{g\in\mathcal{S}_k}|\hat{\mu}_g-\mu_g|,
\end{equation}
\begin{equation}
\mathrm{PCC}^{(k)}=\mathrm{corr}\!\left(\hat{\mu}_{\mathcal{S}_k},
\mu_{\mathcal{S}_k}\right),
\qquad
R^{2(k)}=1-\frac{\sum_{g\in\mathcal{S}_k}(\mu_g-\hat{\mu}_g)^{2}}
{\sum_{g\in\mathcal{S}_k}(\mu_g-\bar{\mu}_{\mathcal{S}_k})^{2}}.
\end{equation}

\subsubsection{Other VAE Backbone Integration}
\label{app_exp_part:other_vae_backbone}
Since our PerturbEnergy is backbone-agnostic, we integrate it with two additional VAE backbones, scGen and PerturbedVAE, with the results reported in Table~\ref{app_exp_part:tab_other_backbone_rmse}. Since scGen does not explicitly learn a separate perturbation representation, we adapt the energy function from $E_{\alpha}(z_a,z_b,a)$ to $E_{\alpha}(z,a)$, which computes the cosine similarity between the latent representation and the perturbation embedding. For PerturbedVAE, we directly apply $E_{\alpha}(z_a,z_b,a)$ to its basal and perturbation representations. Both integrations preserve the original architectures and training objectives, with only the energy-based optimization term in Eq.~\eqref{method_eq:j_gradient} added. The results show that integrating PerturbEnergy consistently improves prediction performance, particularly for scGen. Specifically, PerturbEnergy reduces the top-50 DEG RMSE of scGen by 18.6\% on Replogle (from 0.2336 to 0.1902) and 24.1\% on Tahoe100M (from 0.4462 to 0.3387). These substantial improvements suggest that explicitly incorporating perturbation-specific information is particularly beneficial for models without a dedicated perturbation representation. PerturbEnergy also improves PerturbedVAE, although the gains are relatively smaller, potentially because PerturbedVAE already captures some perturbation-specific information. Collectively, these results demonstrate the flexibility of PerturbEnergy across different VAE backbones and its effectiveness in improving cellular response prediction.

\begin{table}[h!]
\centering
\small
\caption{Backbone transfer: RMSE ($\downarrow$) over the true top-$K$ DEGs when the PerturbEnergy objective is applied to two other backbones. Values are the mean over held-out perturbation groups $\pm$ the s.e.m. across those groups. Within each backbone block, the best entry per column is in \textbf{bold}.}
\label{app_exp_part:tab_other_backbone_rmse}
\setlength{\tabcolsep}{4pt}
\resizebox{\textwidth}{!}{%
\begin{tabular}{lcccccc}
\toprule
\multirow{2}{*}{Method} & \multicolumn{3}{c}{Replogle} & \multicolumn{3}{c}{Tahoe100M} \\
\cmidrule(lr){2-4} \cmidrule(lr){5-7}
 & Top-50 & Top-100 & Top-200 & Top-50 & Top-100 & Top-200 \\
\midrule
scGen & 0.2336\,\std{0.0078} & 0.1992\,\std{0.0065} & 0.1661\,\std{0.0051} & 0.4462\,\std{0.0081} & 0.3818\,\std{0.0059} & 0.3018\,\std{0.0046} \\
\quad + PerturbEnergy (Lang) & 0.1931\,\std{0.0064} & 0.1631\,\std{0.0052} & 0.1356\,\std{0.0041} & 0.3457\,\std{0.0061} & 0.3023\,\std{0.0044} & 0.2421\,\std{0.0034} \\
\quad + PerturbEnergy (Mis) & \textbf{0.1902}\,\std{0.0063} & \textbf{0.1607}\,\std{0.0051} & \textbf{0.1341}\,\std{0.0040} & \textbf{0.3387}\,\std{0.0059} & \textbf{0.2966}\,\std{0.0043} & \textbf{0.2379}\,\std{0.0033} \\
\midrule
PerturbedVAE & 0.2130\,\std{0.0071} & 0.1817\,\std{0.0058} & 0.1514\,\std{0.0047} & 0.3984\,\std{0.0075} & 0.3476\,\std{0.0055} & 0.2791\,\std{0.0042} \\
\quad + PerturbEnergy (Lang) & \textbf{0.1966}\,\std{0.0058} & \textbf{0.1692}\,\std{0.0048} & 0.1424\,\std{0.0038} & 0.3892\,\std{0.0073} & \textbf{0.3324}\,\std{0.0053} & \textbf{0.2653}\,\std{0.0040} \\
\quad + PerturbEnergy (Mis) & 0.1969\,\std{0.0059} & 0.1694\,\std{0.0049} & \textbf{0.1422}\,\std{0.0039} & \textbf{0.3759}\,\std{0.0068} & 0.3327\,\std{0.0050} & 0.2692\,\std{0.0039} \\
\bottomrule
\addlinespace[6pt]
\toprule
\multirow{2}{*}{Method} & \multicolumn{3}{c}{PBMC} & \multicolumn{3}{c}{Norman2019} \\
\cmidrule(lr){2-4} \cmidrule(lr){5-7}
 & Top-50 & Top-100 & Top-200 & Top-50 & Top-100 & Top-200 \\
\midrule
scGen & 0.0960\,\std{0.0049} & 0.0787\,\std{0.0040} & 0.0625\,\std{0.0031} & 0.1729\,\std{0.0102} & 0.1404\,\std{0.0082} & 0.1108\,\std{0.0065} \\
\quad + PerturbEnergy (Lang) & 0.0706\,\std{0.0021} & 0.0600\,\std{0.0017} & 0.0495\,\std{0.0014} & \textbf{0.1560}\,\std{0.0127} & \textbf{0.1265}\,\std{0.0099} & \textbf{0.0990}\,\std{0.0073} \\
\quad + PerturbEnergy (Mis) & \textbf{0.0703}\,\std{0.0022} & \textbf{0.0597}\,\std{0.0018} & \textbf{0.0493}\,\std{0.0015} & 0.1645\,\std{0.0137} & 0.1329\,\std{0.0107} & 0.1030\,\std{0.0079} \\
\midrule
PerturbedVAE & 0.1071\,\std{0.0032} & 0.0925\,\std{0.0025} & 0.0777\,\std{0.0020} & 0.1400\,\std{0.0123} & 0.1186\,\std{0.0100} & 0.0953\,\std{0.0075} \\
\quad + PerturbEnergy (Lang) & \textbf{0.0987}\,\std{0.0031} & \textbf{0.0849}\,\std{0.0024} & \textbf{0.0710}\,\std{0.0019} & \textbf{0.1312}\,\std{0.0115} & 0.1109\,\std{0.0093} & 0.0919\,\std{0.0071} \\
\quad + PerturbEnergy (Mis) & 0.1007\,\std{0.0028} & 0.0880\,\std{0.0022} & 0.0740\,\std{0.0017} & 0.1379\,\std{0.0108} & \textbf{0.1072}\,\std{0.0087} & \textbf{0.0897}\,\std{0.0066} \\
\bottomrule
\end{tabular}}
\end{table}

\subsubsection{Influence of Training Scale and Model Dimension}
\label{app_exp_part:training_model_scale}

\begin{figure*}[h!]
\centering
\includegraphics[width=\textwidth]{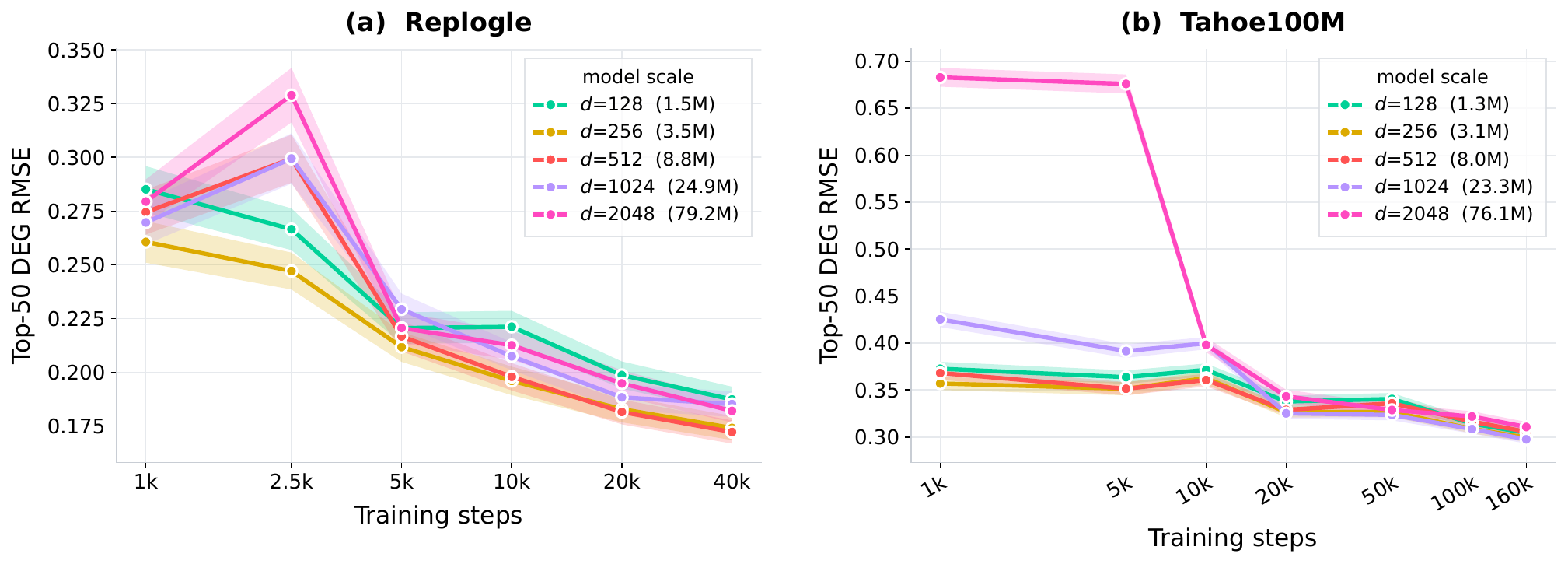}
\caption{Influence of training steps and model dimension on Replogle and Tahoe100M.}
\label{app_fig:traning_model_scale}
\end{figure*}

We investigate the effects of training scale and latent dimension on Replogle and Tahoe100M, using RMSE on the top-50 DEGs as the evaluation metric. The results are shown in Figure~\ref{app_fig:traning_model_scale}. We observe that insufficient training steps consistently degrade performance across all latent dimensions and datasets, suggesting that learning generalizable perturbation effects requires adequate training. Moreover, increasing the latent dimension does not necessarily improve performance, with the best results generally achieved when $d$ ranges from 256 to 1024. For example, $d=512$ generally achieves the best performance on Replogle. One possible explanation is that larger latent dimensions increase model capacity but do not necessarily improve the ability to capture perturbation-specific effects. These findings also suggest that relying on model expressiveness to capture perturbation effect is unstable.

\subsubsection{Visualization}
\label{app_exp_part:visualization}
We visualize the latent representations learned by PerturbEnergy and two VAE-based methods, scGen and PerturbedVAE, alongside the original gene expression profiles. We use the test set of Tahoe100M to visualize and the UMAP visualization results are shown in Figure~\ref{app_fig:visualization}. We observe that perturbation-specific information is largely obscured in the original expression space. Although scGen captures some perturbation effects, its representations remain largely indistinguishable across perturbations. PerturbedVAE improves the separation of perturbation-specific representations, but some perturbations, such as Betamethasone dipropionate and Triamcinolone, remain mixed. In contrast, PerturbEnergy exhibits clearer separation between perturbations, demonstrating its ability to preserve perturbation-specific information in the latent space.

\begin{figure*}[h!]
\centering
\includegraphics[width=\textwidth]{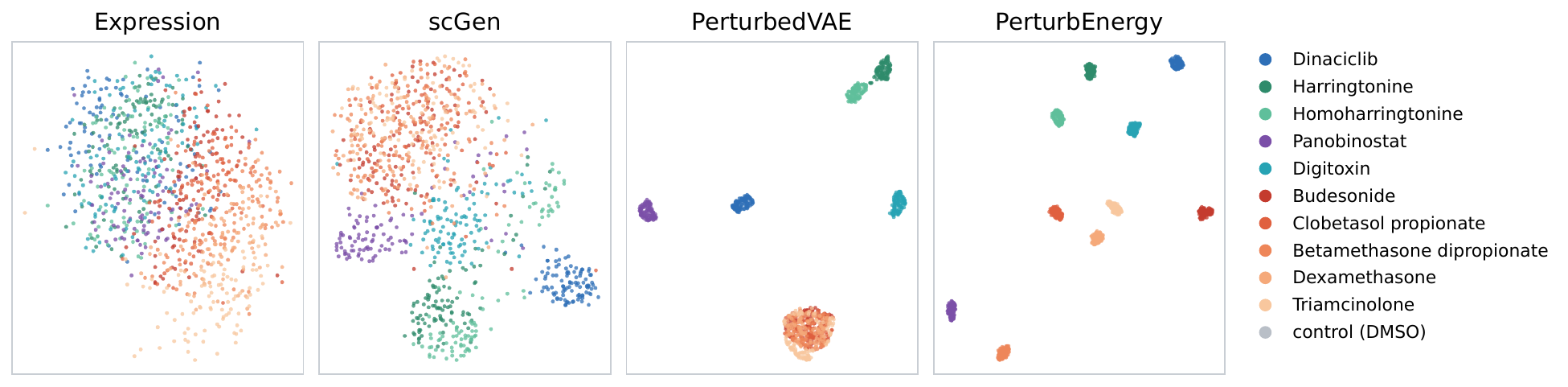}
\caption{Visualization of representation in Tahoe100M.}
\label{app_fig:visualization}
\end{figure*}

\subsubsection{Other Metrics on DEG Group}
\label{app_exp:deg_group_res}

We report the MAE, PCC and R$^2$ on Top-50/100/200 differential expression genes (DEGs) in Table~\ref{app_exp_tab:deg_mae}, Table~\ref{app_exp_tab:deg_pcc} and Table~\ref{app_exp_tab:deg_r2} respectively. Firstly, about our method, we have the consistent observation with the RMSE where PerturbEnergy achieves the top-2 performance on every dataset except Norman2019 even under simple MLP implementations. Specifically, in terms of MAE, the margins over the strongest baseline are largest. For example, PerturbEnergy-Mis reduces MAE from 0.1675 (Squidiff) to 0.1434 on Replogle (16.8\% improvement); from 0.2648 (stVAE) to 0.2495 on Tahoe100M; on PBMC, PerturbEnergy-Lang reduces it from 0.0511 (stVAE) to 0.0453 (11.4\% improvement). As for PCC and $R^2$, we have the same ordering observations. Secondly, for the remaining baselines, we find that: (i) Linear methods still struggle on the DEG-related prediction, for example, in MAE it achives 0.2114 compared to the best 0.1434. This again shows the regression is difficult to capture the perturbation effects; (ii) For STATE, PerturbDiff, CellFlow and Squidiff, their performance also vary a lot. In terms of MAE, PerturbDiff achieves best on Norman2019 but struggles on Replogle; Squidiff achieves competitive performance on Replogle but struggles on Tahoe100M; Cellflow perform well on PBMC and Norman2019 but get declined performance on Tahoe100M. These observations show that relying on model expressiveness to capture perturbation effect is unstable. (iii) For the representation learning baselines, especially these learn the perturbation-aware representations, we find they generally can achieve competitive performance. For example in top-50 MAE, PerturbedVAE gets 0.1812 on Replogle and 0.0703 on PBMC, which is the best among other baselines. This prove the effectiveness of modeling perturbation information.

\begin{table}[h!]
\centering
\small
\caption{MAE ($\downarrow$) over the true top-$K$ DEGs. Values are the mean over held-out perturbation groups $\pm$ the s.e.m. across those groups. Best in \textbf{bold}, second best \underline{underlined}.}
\label{tab:deg-mae}
\setlength{\tabcolsep}{4pt}
\resizebox{\textwidth}{!}{%
\begin{tabular}{lcccccc}
\toprule
\multirow{2}{*}{Method} & \multicolumn{3}{c}{Replogle} & \multicolumn{3}{c}{Tahoe100M} \\
\cmidrule(lr){2-4} \cmidrule(lr){5-7}
 & Top-50 & Top-100 & Top-200 & Top-50 & Top-100 & Top-200 \\
\midrule
Linear & 0.2282\,\std{0.0076} & 0.1963\,\std{0.0062} & 0.1634\,\std{0.0049} & 0.3788\,\std{0.0069} & 0.3169\,\std{0.0046} & 0.2246\,\std{0.0034} \\
Linear (scGPT) & 0.2114\,\std{0.0076} & 0.1757\,\std{0.0060} & 0.1418\,\std{0.0046} & 0.3092\,\std{0.0057} & 0.2624\,\std{0.0039} & 0.1884\,\std{0.0029} \\
\midrule
CellFlow & 0.1925\,\std{0.0063} & 0.1615\,\std{0.0051} & 0.1306\,\std{0.0039} & 0.3280\,\std{0.0059} & 0.2767\,\std{0.0039} & 0.1975\,\std{0.0029} \\
Squidiff & 0.1675\,\std{0.0047} & 0.1433\,\std{0.0037} & 0.1205\,\std{0.0029} & 0.3077\,\std{0.0061} & 0.2785\,\std{0.0042} & 0.2116\,\std{0.0030} \\
PerturbDiff & 0.2304\,\std{0.0092} & 0.1904\,\std{0.0075} & 0.1519\,\std{0.0059} & 0.2878\,\std{0.0062} & 0.2415\,\std{0.0042} & 0.1748\,\std{0.0031} \\
STATE & 0.1704\,\std{0.0058} & 0.1424\,\std{0.0046} & 0.1162\,\std{0.0035} & 0.2787\,\std{0.0037} & 0.2394\,\std{0.0025} & 0.1734\,\std{0.0019} \\
\midrule
scGen & 0.1985\,\std{0.0071} & 0.1665\,\std{0.0058} & 0.1347\,\std{0.0044} & 0.3654\,\std{0.0076} & 0.3040\,\std{0.0051} & 0.2165\,\std{0.0038} \\
CPA & 0.3164\,\std{0.0126} & 0.2615\,\std{0.0105} & 0.2067\,\std{0.0081} & 0.4564\,\std{0.0072} & 0.4047\,\std{0.0047} & 0.2977\,\std{0.0037} \\
stVAE & 0.1797\,\std{0.0065} & 0.1455\,\std{0.0049} & 0.1159\,\std{0.0036} & 0.2648\,\std{0.0051} & 0.2263\,\std{0.0033} & 0.1646\,\std{0.0025} \\
PerturbedVAE & 0.1812\,\std{0.0066} & 0.1509\,\std{0.0054} & 0.1216\,\std{0.0042} & 0.3160\,\std{0.0068} & 0.2666\,\std{0.0047} & 0.1918\,\std{0.0034} \\
\midrule
PerturbEnergy-Lang & \underline{0.1541}\,\std{0.0055} & \underline{0.1281}\,\std{0.0044} & \underline{0.1044}\,\std{0.0033} & \underline{0.2503}\,\std{0.0045} & \underline{0.2131}\,\std{0.0030} & \underline{0.1550}\,\std{0.0022} \\
PerturbEnergy-Mis & \textbf{0.1434}\,\std{0.0051} & \textbf{0.1196}\,\std{0.0040} & \textbf{0.0979}\,\std{0.0030} & \textbf{0.2495}\,\std{0.0045} & \textbf{0.2115}\,\std{0.0030} & \textbf{0.1531}\,\std{0.0022} \\
\bottomrule
\addlinespace[6pt]
\toprule
\multirow{2}{*}{Method} & \multicolumn{3}{c}{PBMC} & \multicolumn{3}{c}{Norman2019} \\
\cmidrule(lr){2-4} \cmidrule(lr){5-7}
 & Top-50 & Top-100 & Top-200 & Top-50 & Top-100 & Top-200 \\
\midrule
Linear & 0.0836\,\std{0.0042} & 0.0672\,\std{0.0033} & 0.0515\,\std{0.0025} & 0.3285\,\std{0.0244} & 0.2335\,\std{0.0189} & 0.1506\,\std{0.0127} \\
Linear (scGPT) & 0.0780\,\std{0.0041} & 0.0619\,\std{0.0032} & 0.0466\,\std{0.0024} & 0.2325\,\std{0.0072} & 0.1461\,\std{0.0057} & 0.1086\,\std{0.0039} \\
\midrule
CellFlow & 0.0654\,\std{0.0021} & 0.0511\,\std{0.0015} & 0.0376\,\std{0.0011} & 0.1264\,\std{0.0127} & 0.0977\,\std{0.0102} & 0.0699\,\std{0.0071} \\
Squidiff & 0.0855\,\std{0.0044} & 0.0692\,\std{0.0034} & 0.0532\,\std{0.0026} & 0.1513\,\std{0.0076} & 0.1257\,\std{0.0062} & 0.0966\,\std{0.0048} \\
PerturbDiff & 0.0568\,\std{0.0025} & 0.0455\,\std{0.0018} & 0.0349\,\std{0.0014} & \textbf{0.0950}\,\std{0.0068} & \textbf{0.0754}\,\std{0.0051} & \textbf{0.0563}\,\std{0.0036} \\
STATE & 0.0738\,\std{0.0028} & 0.0580\,\std{0.0022} & 0.0431\,\std{0.0017} & 0.1138\,\std{0.0100} & 0.0885\,\std{0.0074} & 0.0644\,\std{0.0048} \\
\midrule
scGen & 0.0712\,\std{0.0045} & 0.0546\,\std{0.0036} & 0.0396\,\std{0.0027} & 0.1312\,\std{0.0080} & 0.1019\,\std{0.0063} & 0.0738\,\std{0.0046} \\
CPA & 0.1009\,\std{0.0057} & 0.0769\,\std{0.0043} & 0.0559\,\std{0.0031} & 0.3629\,\std{0.0250} & 0.2580\,\std{0.0191} & 0.1654\,\std{0.0128} \\
stVAE & 0.0511\,\std{0.0018} & 0.0403\,\std{0.0014} & 0.0306\,\std{0.0010} & 0.1664\,\std{0.0134} & 0.1228\,\std{0.0101} & 0.0842\,\std{0.0066} \\
PerturbedVAE & 0.0703\,\std{0.0026} & 0.0559\,\std{0.0019} & 0.0429\,\std{0.0014} & 0.1126\,\std{0.0108} & 0.0923\,\std{0.0085} & 0.0690\,\std{0.0058} \\
\midrule
PerturbEnergy-Lang & \textbf{0.0453}\,\std{0.0016} & \textbf{0.0362}\,\std{0.0012} & \textbf{0.0282}\,\std{0.0009} & 0.1129\,\std{0.0095} & 0.0876\,\std{0.0068} & 0.0634\,\std{0.0045} \\
PerturbEnergy-Mis & \underline{0.0467}\,\std{0.0016} & \underline{0.0376}\,\std{0.0012} & \underline{0.0292}\,\std{0.0009} & \underline{0.1075}\,\std{0.0091} & \underline{0.0842}\,\std{0.0065} & \underline{0.0614}\,\std{0.0044} \\
\bottomrule
\end{tabular}}
\label{app_exp_tab:deg_mae}
\end{table}

\begin{table}[h!]
\centering
\small
\caption{PCC ($\uparrow$) over the true top-$K$ DEGs. Values are the mean over held-out perturbation groups $\pm$ the s.e.m. across those groups. Best in \textbf{bold}, second best \underline{underlined}.}
\label{tab:deg-pcc}
\setlength{\tabcolsep}{4pt}
\resizebox{\textwidth}{!}{%
\begin{tabular}{lcccccc}
\toprule
\multirow{2}{*}{Method} & \multicolumn{3}{c}{Replogle} & \multicolumn{3}{c}{Tahoe100M} \\
\cmidrule(lr){2-4} \cmidrule(lr){5-7}
 & Top-50 & Top-100 & Top-200 & Top-50 & Top-100 & Top-200 \\
\midrule
Linear & 0.9611\,\std{0.0024} & 0.9676\,\std{0.0018} & 0.9712\,\std{0.0015} & 0.9202\,\std{0.0031} & 0.9323\,\std{0.0022} & 0.9456\,\std{0.0015} \\
Linear (scGPT) & 0.9692\,\std{0.0020} & 0.9753\,\std{0.0015} & 0.9789\,\std{0.0012} & 0.9406\,\std{0.0026} & 0.9480\,\std{0.0019} & 0.9573\,\std{0.0013} \\
\midrule
CellFlow & 0.9759\,\std{0.0017} & 0.9809\,\std{0.0012} & 0.9840\,\std{0.0009} & 0.9372\,\std{0.0030} & 0.9460\,\std{0.0021} & 0.9563\,\std{0.0015} \\
Squidiff & 0.9760\,\std{0.0013} & 0.9795\,\std{0.0009} & 0.9813\,\std{0.0008} & 0.9116\,\std{0.0034} & 0.9145\,\std{0.0025} & 0.9250\,\std{0.0019} \\
PerturbDiff & 0.9644\,\std{0.0025} & 0.9720\,\std{0.0019} & 0.9768\,\std{0.0015} & 0.9443\,\std{0.0036} & 0.9521\,\std{0.0026} & 0.9610\,\std{0.0019} \\
STATE & 0.9780\,\std{0.0015} & 0.9824\,\std{0.0011} & 0.9853\,\std{0.0009} & 0.9516\,\std{0.0017} & 0.9569\,\std{0.0013} & 0.9642\,\std{0.0009} \\
\midrule
scGen & 0.9713\,\std{0.0019} & 0.9772\,\std{0.0014} & 0.9808\,\std{0.0011} & 0.9219\,\std{0.0035} & 0.9333\,\std{0.0025} & 0.9455\,\std{0.0018} \\
CPA & 0.9405\,\std{0.0040} & 0.9534\,\std{0.0030} & 0.9616\,\std{0.0025} & 0.7536\,\std{0.0066} & 0.7615\,\std{0.0048} & 0.7917\,\std{0.0037} \\
stVAE & 0.9751\,\std{0.0018} & 0.9809\,\std{0.0013} & 0.9843\,\std{0.0010} & 0.9507\,\std{0.0025} & 0.9566\,\std{0.0018} & 0.9640\,\std{0.0013} \\
PerturbedVAE & 0.9766\,\std{0.0017} & 0.9815\,\std{0.0012} & 0.9843\,\std{0.0010} & 0.9334\,\std{0.0032} & 0.9409\,\std{0.0024} & 0.9506\,\std{0.0017} \\
\midrule
PerturbEnergy-Lang & \underline{0.9805}\,\std{0.0016} & \underline{0.9847}\,\std{0.0011} & \underline{0.9872}\,\std{0.0008} & \underline{0.9566}\,\std{0.0019} & \underline{0.9618}\,\std{0.0014} & \underline{0.9683}\,\std{0.0010} \\
PerturbEnergy-Mis & \textbf{0.9824}\,\std{0.0013} & \textbf{0.9860}\,\std{0.0009} & \textbf{0.9881}\,\std{0.0007} & \textbf{0.9567}\,\std{0.0018} & \textbf{0.9619}\,\std{0.0013} & \textbf{0.9684}\,\std{0.0010} \\
\bottomrule
\addlinespace[6pt]
\toprule
\multirow{2}{*}{Method} & \multicolumn{3}{c}{PBMC} & \multicolumn{3}{c}{Norman2019} \\
\cmidrule(lr){2-4} \cmidrule(lr){5-7}
 & Top-50 & Top-100 & Top-200 & Top-50 & Top-100 & Top-200 \\
\midrule
Linear & 0.9394\,\std{0.0039} & 0.9490\,\std{0.0029} & 0.9562\,\std{0.0021} & 0.9427\,\std{0.0075} & 0.9586\,\std{0.0052} & 0.9692\,\std{0.0046} \\
Linear (scGPT) & 0.9534\,\std{0.0032} & 0.9619\,\std{0.0024} & 0.9684\,\std{0.0018} & 0.9548\,\std{0.0007} & 0.9655\,\std{0.0006} & 0.9711\,\std{0.0005} \\
\midrule
CellFlow & 0.9685\,\std{0.0018} & 0.9742\,\std{0.0011} & 0.9787\,\std{0.0008} & 0.9906\,\std{0.0020} & 0.9927\,\std{0.0014} & 0.9938\,\std{0.0014} \\
Squidiff & 0.9289\,\std{0.0044} & 0.9403\,\std{0.0031} & 0.9479\,\std{0.0023} & 0.9850\,\std{0.0015} & 0.9869\,\std{0.0012} & 0.9885\,\std{0.0010} \\
PerturbDiff & 0.9668\,\std{0.0023} & 0.9715\,\std{0.0016} & 0.9752\,\std{0.0011} & \textbf{0.9949}\,\std{0.0007} & \textbf{0.9957}\,\std{0.0006} & \textbf{0.9964}\,\std{0.0005} \\
STATE & 0.9594\,\std{0.0023} & 0.9673\,\std{0.0015} & 0.9730\,\std{0.0011} & 0.9930\,\std{0.0011} & 0.9942\,\std{0.0008} & 0.9953\,\std{0.0007} \\
\midrule
scGen & 0.9552\,\std{0.0036} & 0.9648\,\std{0.0027} & 0.9720\,\std{0.0020} & 0.9903\,\std{0.0011} & 0.9919\,\std{0.0009} & 0.9934\,\std{0.0008} \\
CPA & 0.8896\,\std{0.0057} & 0.9067\,\std{0.0043} & 0.9206\,\std{0.0033} & 0.9363\,\std{0.0076} & 0.9536\,\std{0.0052} & 0.9654\,\std{0.0046} \\
stVAE & 0.9763\,\std{0.0012} & 0.9802\,\std{0.0007} & 0.9829\,\std{0.0005} & 0.9861\,\std{0.0018} & 0.9890\,\std{0.0014} & 0.9913\,\std{0.0011} \\
PerturbedVAE & 0.9403\,\std{0.0031} & 0.9485\,\std{0.0022} & 0.9547\,\std{0.0017} & 0.9928\,\std{0.0013} & 0.9939\,\std{0.0010} & 0.9947\,\std{0.0010} \\
\midrule
PerturbEnergy-Lang & \textbf{0.9816}\,\std{0.0009} & \textbf{0.9843}\,\std{0.0006} & \textbf{0.9863}\,\std{0.0005} & 0.9931\,\std{0.0010} & 0.9942\,\std{0.0008} & 0.9954\,\std{0.0006} \\
PerturbEnergy-Mis & \underline{0.9792}\,\std{0.0010} & \underline{0.9823}\,\std{0.0007} & \underline{0.9844}\,\std{0.0005} & \underline{0.9935}\,\std{0.0009} & \underline{0.9945}\,\std{0.0007} & \underline{0.9956}\,\std{0.0006} \\
\bottomrule
\end{tabular}}
\label{app_exp_tab:deg_pcc}
\end{table}

\begin{table}[h!]
\centering
\small
\caption{$R^2$ ($\uparrow$) over the true top-$K$ DEGs. Values are the mean over held-out perturbation groups $\pm$ the s.e.m. across those groups. Best in \textbf{bold}, second best \underline{underlined}.}
\label{tab:deg-r2}
\setlength{\tabcolsep}{4pt}
\resizebox{\textwidth}{!}{%
\begin{tabular}{lcccccc}
\toprule
\multirow{2}{*}{Method} & \multicolumn{3}{c}{Replogle} & \multicolumn{3}{c}{Tahoe100M} \\
\cmidrule(lr){2-4} \cmidrule(lr){5-7}
 & Top-50 & Top-100 & Top-200 & Top-50 & Top-100 & Top-200 \\
\midrule
Linear & 0.9142\,\std{0.0050} & 0.9297\,\std{0.0037} & 0.9383\,\std{0.0031} & 0.7456\,\std{0.0108} & 0.8009\,\std{0.0078} & 0.8497\,\std{0.0051} \\
Linear (scGPT) & 0.9295\,\std{0.0044} & 0.9456\,\std{0.0031} & 0.9546\,\std{0.0025} & 0.8303\,\std{0.0069} & 0.8621\,\std{0.0051} & 0.8930\,\std{0.0034} \\
\midrule
CellFlow & 0.9444\,\std{0.0035} & 0.9569\,\std{0.0025} & 0.9646\,\std{0.0020} & 0.8185\,\std{0.0074} & 0.8556\,\std{0.0053} & 0.8896\,\std{0.0035} \\
Squidiff & 0.9466\,\std{0.0028} & 0.9560\,\std{0.0019} & 0.9608\,\std{0.0016} & 0.7862\,\std{0.0090} & 0.8028\,\std{0.0072} & 0.8323\,\std{0.0049} \\
PerturbDiff & 0.9177\,\std{0.0054} & 0.9374\,\std{0.0039} & 0.9488\,\std{0.0032} & 0.8482\,\std{0.0071} & 0.8767\,\std{0.0052} & 0.9032\,\std{0.0036} \\
STATE & 0.9492\,\std{0.0035} & 0.9607\,\std{0.0024} & 0.9677\,\std{0.0019} & 0.8642\,\std{0.0045} & 0.8888\,\std{0.0031} & 0.9126\,\std{0.0021} \\
\midrule
scGen & 0.9356\,\std{0.0040} & 0.9495\,\std{0.0030} & 0.9583\,\std{0.0023} & 0.7520\,\std{0.0119} & 0.8041\,\std{0.0088} & 0.8503\,\std{0.0058} \\
CPA & 0.8544\,\std{0.0092} & 0.8895\,\std{0.0068} & 0.9107\,\std{0.0056} & 0.4470\,\std{0.0115} & 0.4937\,\std{0.0080} & 0.5710\,\std{0.0062} \\
stVAE & 0.9436\,\std{0.0040} & 0.9584\,\std{0.0026} & 0.9664\,\std{0.0020} & 0.8690\,\std{0.0055} & 0.8918\,\std{0.0040} & 0.9143\,\std{0.0027} \\
PerturbedVAE & 0.9449\,\std{0.0039} & 0.9572\,\std{0.0027} & 0.9644\,\std{0.0023} & 0.8025\,\std{0.0092} & 0.8378\,\std{0.0068} & 0.8718\,\std{0.0047} \\
\midrule
PerturbEnergy-Lang & \underline{0.9564}\,\std{0.0033} & \underline{0.9667}\,\std{0.0022} & \underline{0.9725}\,\std{0.0017} & \textbf{0.8826}\,\std{0.0046} & \textbf{0.9035}\,\std{0.0034} & \textbf{0.9237}\,\std{0.0023} \\
PerturbEnergy-Mis & \textbf{0.9606}\,\std{0.0028} & \textbf{0.9697}\,\std{0.0019} & \textbf{0.9747}\,\std{0.0015} & \underline{0.8819}\,\std{0.0046} & \underline{0.9033}\,\std{0.0033} & \underline{0.9237}\,\std{0.0023} \\
\bottomrule
\addlinespace[6pt]
\toprule
\multirow{2}{*}{Method} & \multicolumn{3}{c}{PBMC} & \multicolumn{3}{c}{Norman2019} \\
\cmidrule(lr){2-4} \cmidrule(lr){5-7}
 & Top-50 & Top-100 & Top-200 & Top-50 & Top-100 & Top-200 \\
\midrule
Linear & 0.8485\,\std{0.0085} & 0.8804\,\std{0.0059} & 0.9009\,\std{0.0043} & 0.8375\,\std{0.0289} & 0.8975\,\std{0.0166} & 0.9321\,\std{0.0113} \\
Linear (scGPT) & 0.8757\,\std{0.0074} & 0.9067\,\std{0.0050} & 0.9265\,\std{0.0036} & 0.8852\,\std{0.0029} & 0.9288\,\std{0.0019} & 0.9711\,\std{0.0014} \\
\midrule
CellFlow & 0.9029\,\std{0.0056} & 0.9296\,\std{0.0032} & 0.9462\,\std{0.0021} & 0.9706\,\std{0.0090} & 0.9802\,\std{0.0053} & 0.9855\,\std{0.0038} \\
Squidiff & 0.8061\,\std{0.0097} & 0.8399\,\std{0.0066} & 0.8631\,\std{0.0048} & 0.9647\,\std{0.0037} & 0.9714\,\std{0.0027} & 0.9760\,\std{0.0022} \\
PerturbDiff & 0.8972\,\std{0.0086} & 0.9183\,\std{0.0061} & 0.9344\,\std{0.0035} & \textbf{0.9863}\,\std{0.0019} & \textbf{0.9895}\,\std{0.0014} & \textbf{0.9918}\,\std{0.0010} \\
STATE & 0.8833\,\std{0.0061} & 0.9145\,\std{0.0038} & 0.9339\,\std{0.0026} & 0.9803\,\std{0.0038} & 0.9855\,\std{0.0025} & 0.9892\,\std{0.0017} \\
\midrule
scGen & 0.8874\,\std{0.0074} & 0.9163\,\std{0.0054} & 0.9359\,\std{0.0039} & 0.9735\,\std{0.0038} & 0.9805\,\std{0.0026} & 0.9852\,\std{0.0020} \\
CPA & 0.7335\,\std{0.0122} & 0.7857\,\std{0.0089} & 0.8224\,\std{0.0067} & 0.8174\,\std{0.0296} & 0.8845\,\std{0.0169} & 0.9235\,\std{0.0115} \\
stVAE & 0.9302\,\std{0.0035} & 0.9470\,\std{0.0021} & 0.9576\,\std{0.0014} & 0.9592\,\std{0.0066} & 0.9716\,\std{0.0045} & 0.9800\,\std{0.0029} \\
PerturbedVAE & 0.8503\,\std{0.0072} & 0.8791\,\std{0.0047} & 0.8957\,\std{0.0035} & 0.9785\,\std{0.0056} & 0.9840\,\std{0.0035} & 0.9877\,\std{0.0026} \\
\midrule
PerturbEnergy-Lang & \textbf{0.9434}\,\std{0.0036} & \textbf{0.9570}\,\std{0.0020} & \textbf{0.9647}\,\std{0.0015} & 0.9803\,\std{0.0033} & 0.9854\,\std{0.0023} & 0.9893\,\std{0.0015} \\
PerturbEnergy-Mis & \underline{0.9396}\,\std{0.0035} & \underline{0.9536}\,\std{0.0022} & \underline{0.9615}\,\std{0.0015} & \underline{0.9818}\,\std{0.0030} & \underline{0.9864}\,\std{0.0021} & \underline{0.9899}\,\std{0.0014} \\
\bottomrule
\end{tabular}}
\label{app_exp_tab:deg_r2}
\end{table}

\subsubsection{Results on HVG Group}
\label{app_exp:hvg_group_res}

\subsubsubsection{Results on Replogle Dataset}

The HVG group results on Replogle is shown in Table~\ref{app_exp_tab:hvg_replogle_corr}, Table~\ref{app_exp_tab:hvg_replogle_deg} and Table~\ref{app_exp_tab:hvg_replogle_dis}. Table~\ref{app_exp_tab:hvg_replogle_corr} summarizes the correlation based measures, we find that our PerturbEnergy generally achive the best performance, which shows its effectiveness of capturing the whole transcriptomic pattern of cellular response. Table~\ref{app_exp_tab:hvg_replogle_deg} is the results of differential expression gene (DEG) measures. DEG recognition is directly related to perturbation. We find that our PerturbEnergy consistently outperform other baselines with a clear margin, which validates the effectiveness of modeling the perturbation-specific information. Table~\ref{app_exp_tab:hvg_replogle_dis} summarize the distance-based measures. Our PerturbEnergy is also competitive with other baselines.

\begin{table}[h!]
\centering
\small
\caption{HVG-group correlation-based measures on Replogle, over all 2{,}000 highly
  variable genes. Values are the mean over held-out perturbation groups $\pm$ the
  s.e.m.\ across those groups; the perturbation-discrimination scores are computed once
  over the whole test set and carry no such spread. Best in \textbf{bold}, second best
  \underline{underlined}.}
\label{tab:hvg-replogle-corr}
\setlength{\tabcolsep}{4pt}
\resizebox{\textwidth}{!}{%
\begin{tabular}{lccccc}
\toprule
Method & PCC\,$\uparrow$ & $R^2$\,$\uparrow$ & PDCorr\,$\uparrow$ & LFCSpear\,$\uparrow$ & LogFC-cos\,$\uparrow$ \\
\midrule
Linear & 0.9772\,\std{0.0011} & 0.9529\,\std{0.0022} & 0.3323\,\std{0.0129} & 0.2519\,\std{0.0117} & 0.3368\,\std{0.0131} \\
Linear (scGPT) & 0.9868\,\std{0.0006} & 0.9732\,\std{0.0013} & 0.4584\,\std{0.0175} & 0.3607\,\std{0.0145} & 0.4688\,\std{0.0174} \\
\midrule
CellFlow & 0.9706\,\std{0.0012} & 0.9393\,\std{0.0024} & 0.5559\,\std{0.0187} & 0.4623\,\std{0.0170} & 0.5648\,\std{0.0185} \\
Squidiff & 0.9829\,\std{0.0005} & 0.9649\,\std{0.0009} & 0.4945\,\std{0.0170} & 0.4536\,\std{0.0141} & 0.5065\,\std{0.0169} \\
PerturbDiff & 0.9875\,\std{0.0008} & 0.9743\,\std{0.0015} & 0.5470\,\std{0.0174} & 0.5186\,\std{0.0137} & 0.5557\,\std{0.0173} \\
STATE & 0.9905\,\std{0.0005} & 0.9805\,\std{0.0010} & 0.5091\,\std{0.0206} & 0.4323\,\std{0.0183} & 0.5177\,\std{0.0205} \\
\midrule
scGen & 0.9883\,\std{0.0006} & 0.9760\,\std{0.0012} & 0.5381\,\std{0.0138} & 0.4900\,\std{0.0118} & 0.5411\,\std{0.0139} \\
CPA & 0.9815\,\std{0.0012} & 0.9614\,\std{0.0025} & 0.3039\,\std{0.0062} & 0.3040\,\std{0.0047} & 0.3101\,\std{0.0063} \\
stVAE & 0.9903\,\std{0.0005} & 0.9802\,\std{0.0010} & 0.5648\,\std{0.0158} & 0.5142\,\std{0.0132} & 0.5715\,\std{0.0158} \\
PerturbedVAE & 0.9902\,\std{0.0005} & 0.9795\,\std{0.0011} & 0.5578\,\std{0.0151} & 0.5019\,\std{0.0128} & 0.5600\,\std{0.0152} \\
\midrule
PerturbEnergy-Lang & \textbf{0.9915}\,\std{0.0004} & \textbf{0.9825}\,\std{0.0009} & \underline{0.6000}\,\std{0.0154} & \underline{0.5470}\,\std{0.0131} & \underline{0.6070}\,\std{0.0154} \\
PerturbEnergy-Mis & \underline{0.9913}\,\std{0.0004} & \underline{0.9821}\,\std{0.0008} & \textbf{0.6038}\,\std{0.0165} & \textbf{0.5492}\,\std{0.0138} & \textbf{0.6106}\,\std{0.0164} \\
\bottomrule
\end{tabular}}

\vspace{6pt}

\resizebox{\textwidth}{!}{%
\begin{tabular}{lcccc}
\toprule
Method & DirAgr\,$\uparrow$ & PDS-cos\,$\downarrow$ & PDS-$\ell_1$\,$\downarrow$ & PDS-$\ell_2$\,$\downarrow$ \\
\midrule
Linear & 0.5780\,\std{0.0043} & 0.2465 & 0.2992 & 0.2790 \\
Linear (scGPT) & 0.6146\,\std{0.0046} & 0.3852 & 0.3542 & 0.3512 \\
\midrule
CellFlow & 0.5430\,\std{0.0024} & 0.5042 & 0.4987 & 0.5012 \\
Squidiff & 0.6645\,\std{0.0048} & 0.3889 & 0.2657 & 0.2493 \\
PerturbDiff & 0.6903\,\std{0.0047} & 0.4208 & 0.4477 & 0.4437 \\
STATE & 0.6456\,\std{0.0064} & 0.3266 & 0.2150 & 0.2156 \\
\midrule
scGen & 0.6673\,\std{0.0042} & \underline{0.2336} & 0.3026 & 0.2868 \\
CPA & 0.6038\,\std{0.0015} & 0.3573 & 0.4960 & 0.4967 \\
stVAE & 0.6817\,\std{0.0047} & 0.3461 & 0.2521 & 0.2550 \\
PerturbedVAE & 0.6756\,\std{0.0046} & \textbf{0.1744} & 0.2383 & 0.2326 \\
\midrule
PerturbEnergy-Lang & \underline{0.6938}\,\std{0.0047} & 0.2633 & \textbf{0.1939} & \textbf{0.1902} \\
PerturbEnergy-Mis & \textbf{0.6973}\,\std{0.0049} & 0.3329 & \underline{0.1971} & \underline{0.1932} \\
\bottomrule
\end{tabular}}
\label{app_exp_tab:hvg_replogle_corr}
\end{table}

\begin{table}[h!]
\centering
\small
\caption{HVG-group DEG-related scores on Replogle, over all 2{,}000 highly variable
  genes. Values are the mean over held-out perturbation groups $\pm$ the s.e.m.\ across
  those groups. Best in \textbf{bold}, second best \underline{underlined}.}
\label{tab:hvg-replogle-deg}
\setlength{\tabcolsep}{4pt}
\resizebox{\textwidth}{!}{%
\begin{tabular}{lccccc}
\toprule
Method & DEPrec\,$\uparrow$ & DERec\,$\uparrow$ & DEOverlap\,$\uparrow$ & AUROC\,$\uparrow$ & AUPRC\,$\uparrow$ \\
\midrule
Linear & 0.2047\,\std{0.0075} & 0.2047\,\std{0.0075} & 0.1198\,\std{0.0049} & 0.7200\,\std{0.0076} & 0.1662\,\std{0.0073} \\
Linear (scGPT) & 0.2590\,\std{0.0106} & 0.2590\,\std{0.0106} & 0.1615\,\std{0.0073} & 0.7698\,\std{0.0092} & 0.2409\,\std{0.0115} \\
\midrule
CellFlow & 0.0967\,\std{0.0033} & 0.0967\,\std{0.0033} & 0.0518\,\std{0.0019} & 0.6708\,\std{0.0054} & 0.0806\,\std{0.0023} \\
Squidiff & 0.2566\,\std{0.0104} & 0.2566\,\std{0.0104} & 0.1592\,\std{0.0071} & 0.7981\,\std{0.0089} & 0.2475\,\std{0.0114} \\
PerturbDiff & 0.2982\,\std{0.0132} & 0.2982\,\std{0.0132} & 0.1973\,\std{0.0100} & 0.8096\,\std{0.0085} & 0.2938\,\std{0.0147} \\
STATE & 0.3049\,\std{0.0134} & 0.3049\,\std{0.0134} & 0.2024\,\std{0.0101} & 0.7946\,\std{0.0101} & 0.2997\,\std{0.0149} \\
\midrule
scGen & 0.2739\,\std{0.0098} & 0.2739\,\std{0.0098} & 0.1696\,\std{0.0068} & 0.8067\,\std{0.0079} & 0.2537\,\std{0.0103} \\
CPA & 0.1215\,\std{0.0041} & 0.1215\,\std{0.0041} & 0.0661\,\std{0.0024} & 0.6707\,\std{0.0045} & 0.0859\,\std{0.0028} \\
stVAE & 0.2837\,\std{0.0126} & 0.2837\,\std{0.0126} & 0.1848\,\std{0.0093} & 0.7987\,\std{0.0089} & 0.2713\,\std{0.0138} \\
PerturbedVAE & 0.2915\,\std{0.0114} & 0.2915\,\std{0.0114} & 0.1863\,\std{0.0082} & 0.8075\,\std{0.0084} & 0.2805\,\std{0.0128} \\
\midrule
PerturbEnergy-Lang & \textbf{0.3226}\,\std{0.0129} & \textbf{0.3226}\,\std{0.0129} & \textbf{0.2140}\,\std{0.0098} & \textbf{0.8233}\,\std{0.0086} & \textbf{0.3173}\,\std{0.0147} \\
PerturbEnergy-Mis & \underline{0.3219}\,\std{0.0131} & \underline{0.3219}\,\std{0.0131} & \underline{0.2139}\,\std{0.0100} & \underline{0.8202}\,\std{0.0089} & \underline{0.3164}\,\std{0.0149} \\
\bottomrule
\end{tabular}}
\label{app_exp_tab:hvg_replogle_deg}
\end{table}

\begin{table}[h!]
\centering
\small
\caption{HVG-group distribution-level distances on Replogle, over all 2{,}000 highly
  variable genes. Values are the mean over held-out perturbation groups $\pm$ the
  s.e.m.\ across those groups. Best in \textbf{bold}, second best \underline{underlined}.}
\label{tab:hvg-replogle-dist}
\setlength{\tabcolsep}{4pt}
\resizebox{\textwidth}{!}{%
\begin{tabular}{lccccc}
\toprule
Method & RMSE\,$\downarrow$ & MSE\,$\downarrow$ & MAE\,$\downarrow$ & EDist\,$\downarrow$ & Sinkhorn\,$\downarrow$ \\
\midrule
Linear & 0.1122\,\std{0.0025} & 0.0144\,\std{0.0007} & 0.0764\,\std{0.0016} & 1.4655\,\std{0.0592} & 1.0269\,\std{0.0047} \\
Linear (scGPT) & 0.0860\,\std{0.0017} & 0.0082\,\std{0.0004} & 0.0585\,\std{0.0010} & \underline{0.9934}\,\std{0.0352} & 0.9918\,\std{0.0035} \\
\midrule
CellFlow & 0.1303\,\std{0.0024} & 0.0186\,\std{0.0007} & 0.0866\,\std{0.0015} & 4.4345\,\std{0.0812} & 0.6073\,\std{0.0052} \\
Squidiff & 0.1021\,\std{0.0012} & 0.0108\,\std{0.0003} & 0.0656\,\std{0.0007} & 3.7791\,\std{0.0379} & 0.5664\,\std{0.0035} \\
PerturbDiff & 0.0817\,\std{0.0020} & 0.0079\,\std{0.0005} & 0.0527\,\std{0.0012} & 6.3453\,\std{0.0534} & 0.5637\,\std{0.0041} \\
STATE & 0.0725\,\std{0.0016} & 0.0060\,\std{0.0003} & 0.0491\,\std{0.0009} & \textbf{0.7007}\,\std{0.0274} & 0.9903\,\std{0.0038} \\
\midrule
scGen & 0.0795\,\std{0.0019} & 0.0073\,\std{0.0004} & 0.0525\,\std{0.0011} & 3.4475\,\std{0.0548} & \underline{0.5432}\,\std{0.0036} \\
CPA & 0.0953\,\std{0.0030} & 0.0118\,\std{0.0008} & 0.0612\,\std{0.0018} & 4.8973\,\std{0.0860} & 0.5575\,\std{0.0052} \\
stVAE & 0.0738\,\std{0.0015} & 0.0061\,\std{0.0003} & 0.0490\,\std{0.0008} & 5.8831\,\std{0.0569} & \textbf{0.5376}\,\std{0.0038} \\
PerturbedVAE & 0.0745\,\std{0.0016} & 0.0063\,\std{0.0003} & 0.0501\,\std{0.0009} & 5.5135\,\std{0.0444} & 0.5545\,\std{0.0038} \\
\midrule
PerturbEnergy-Lang & \textbf{0.0691}\,\std{0.0014} & \textbf{0.0054}\,\std{0.0003} & \textbf{0.0465}\,\std{0.0009} & 8.6506\,\std{0.0427} & 0.5577\,\std{0.0037} \\
PerturbEnergy-Mis & \underline{0.0704}\,\std{0.0014} & \underline{0.0055}\,\std{0.0002} & \underline{0.0473}\,\std{0.0008} & 9.0347\,\std{0.0442} & 0.5587\,\std{0.0036} \\
\bottomrule
\end{tabular}}
\label{app_exp_tab:hvg_replogle_dis}
\end{table}

\subsubsubsection{Results on PBMC Dataset}

The HVG group results on PBMC is shown in Table~\ref{app_exp_tab:hvg_pbmc_corr}, Table~\ref{app_exp_tab:hvg_pbmc_deg} and Table~\ref{app_exp_tab:hvg_pbmc_dis}. From the correlation and distance related measures in Table~\ref{app_exp_tab:hvg_pbmc_corr} and Table~\ref{app_exp_tab:hvg_pbmc_dis} respectively, we have the same observations with Replogle, where PerturbEnergy outperforms or remains competitive to other baselines, showing its ability to capture the general expression patterns. From the DE-recoginition results in Table~\ref{app_exp_tab:hvg_pbmc_deg}, we also find that PerturbEnergy outperform others especially VAE baselines (0.2473 vs 0.2208). This benefit comes from modeling perturbation-specific information.

\begin{table}[h!]
\centering
\small
\caption{HVG-group correlation-based measures on PBMC, over all 2{,}000 highly variable
  genes. Values are the mean over held-out perturbation groups $\pm$ the s.e.m.\ across
  those groups; the perturbation-discrimination scores are computed once over the whole
  test set and carry no such spread. Best in \textbf{bold}, second best
  \underline{underlined}.}
\label{tab:hvg-pbmc-corr}
\setlength{\tabcolsep}{4pt}
\resizebox{\textwidth}{!}{%
\begin{tabular}{lccccc}
\toprule
Method & PCC\,$\uparrow$ & $R^2$\,$\uparrow$ & PDCorr\,$\uparrow$ & LFCSpear\,$\uparrow$ & LogFC-cos\,$\uparrow$ \\
\midrule
Linear & 0.9705\,\std{0.0009} & 0.9372\,\std{0.0020} & 0.2015\,\std{0.0123} & 0.0708\,\std{0.0060} & 0.1960\,\std{0.0123} \\
Linear (scGPT) & 0.9826\,\std{0.0007} & 0.9636\,\std{0.0014} & 0.2922\,\std{0.0119} & 0.0881\,\std{0.0049} & 0.2934\,\std{0.0119} \\
\midrule
CellFlow & 0.9784\,\std{0.0013} & 0.9555\,\std{0.0025} & 0.2301\,\std{0.0067} & 0.2610\,\std{0.0051} & 0.2294\,\std{0.0069} \\
Squidiff & 0.9605\,\std{0.0009} & 0.8978\,\std{0.0016} & 0.3061\,\std{0.0102} & 0.2693\,\std{0.0044} & 0.3003\,\std{0.0103} \\
PerturbDiff & 0.9805\,\std{0.0007} & 0.9521\,\std{0.0021} & 0.4435\,\std{0.0125} & 0.3566\,\std{0.0057} & 0.4429\,\std{0.0126} \\
STATE & 0.9868\,\std{0.0006} & 0.9724\,\std{0.0012} & 0.3074\,\std{0.0140} & 0.0808\,\std{0.0061} & 0.3078\,\std{0.0140} \\
\midrule
scGen & 0.9857\,\std{0.0007} & 0.9706\,\std{0.0013} & 0.4126\,\std{0.0118} & \underline{0.3594}\,\std{0.0047} & 0.4121\,\std{0.0118} \\
CPA & 0.9425\,\std{0.0016} & 0.8794\,\std{0.0033} & 0.1275\,\std{0.0075} & 0.2351\,\std{0.0029} & 0.1207\,\std{0.0078} \\
stVAE & \underline{0.9887}\,\std{0.0003} & \underline{0.9760}\,\std{0.0006} & 0.4886\,\std{0.0119} & 0.3531\,\std{0.0050} & 0.4859\,\std{0.0119} \\
PerturbedVAE & 0.9654\,\std{0.0009} & 0.9232\,\std{0.0015} & 0.3244\,\std{0.0151} & 0.3069\,\std{0.0056} & 0.3241\,\std{0.0151} \\
\midrule
PerturbEnergy-Lang & \textbf{0.9903}\,\std{0.0003} & \textbf{0.9781}\,\std{0.0006} & \textbf{0.5577}\,\std{0.0103} & \textbf{0.3689}\,\std{0.0045} & \textbf{0.5562}\,\std{0.0103} \\
PerturbEnergy-Mis & 0.9884\,\std{0.0003} & 0.9746\,\std{0.0007} & \underline{0.5345}\,\std{0.0108} & 0.3499\,\std{0.0048} & \underline{0.5311}\,\std{0.0109} \\
\bottomrule
\end{tabular}}

\vspace{6pt}

\resizebox{\textwidth}{!}{%
\begin{tabular}{lcccc}
\toprule
Method & DirAgr\,$\uparrow$ & PDS-cos\,$\downarrow$ & PDS-$\ell_1$\,$\downarrow$ & PDS-$\ell_2$\,$\downarrow$ \\
\midrule
Linear & 0.5181\,\std{0.0021} & 0.2185 & 0.3569 & 0.3141 \\
Linear (scGPT) & 0.5265\,\std{0.0016} & 0.1139 & 0.3067 & 0.2689 \\
\midrule
CellFlow & 0.6100\,\std{0.0021} & 0.1190 & 0.2487 & 0.3699 \\
Squidiff & 0.6191\,\std{0.0016} & 0.1116 & 0.1344 & 0.1836 \\
PerturbDiff & \underline{0.6552}\,\std{0.0020} & 0.0920 & 0.0929 & 0.1461 \\
STATE & 0.5247\,\std{0.0018} & 0.1591 & 0.3306 & 0.2800 \\
\midrule
scGen & \textbf{0.6556}\,\std{0.0019} & 0.0689 & 0.1641 & 0.2173 \\
CPA & 0.5959\,\std{0.0012} & 0.2560 & 0.1856 & 0.3423 \\
stVAE & 0.6418\,\std{0.0017} & 0.0432 & 0.0777 & 0.1181 \\
PerturbedVAE & 0.6293\,\std{0.0021} & 0.1709 & 0.1135 & 0.1980 \\
\midrule
PerturbEnergy-Lang & 0.6376\,\std{0.0016} & \textbf{0.0289} & \textbf{0.0463} & \textbf{0.0681} \\
PerturbEnergy-Mis & 0.6325\,\std{0.0016} & \underline{0.0317} & \underline{0.0541} & \underline{0.0706} \\
\bottomrule
\end{tabular}}
\label{app_exp_tab:hvg_pbmc_corr}
\end{table}

\begin{table}[h!]
\centering
\small
\caption{HVG-group DEG-related scores on PBMC, over all 2{,}000 highly variable genes.
  Values are the mean over held-out perturbation groups $\pm$ the s.e.m.\ across those
  groups. Best in \textbf{bold}, second best \underline{underlined}.}
\label{tab:hvg-pbmc-deg}
\setlength{\tabcolsep}{4pt}
\resizebox{\textwidth}{!}{%
\begin{tabular}{lccccc}
\toprule
Method & DEPrec\,$\uparrow$ & DERec\,$\uparrow$ & DEOverlap\,$\uparrow$ & AUROC\,$\uparrow$ & AUPRC\,$\uparrow$ \\
\midrule
Linear & 0.1404\,\std{0.0055} & 0.1404\,\std{0.0055} & 0.0786\,\std{0.0035} & 0.7303\,\std{0.0045} & 0.1169\,\std{0.0059} \\
Linear (scGPT) & 0.1693\,\std{0.0067} & 0.1693\,\std{0.0067} & 0.0975\,\std{0.0046} & 0.7147\,\std{0.0052} & 0.1469\,\std{0.0070} \\
\midrule
CellFlow & 0.1349\,\std{0.0035} & 0.1349\,\std{0.0035} & 0.0735\,\std{0.0021} & 0.7689\,\std{0.0030} & 0.1062\,\std{0.0024} \\
Squidiff & 0.1501\,\std{0.0063} & 0.1501\,\std{0.0063} & 0.0853\,\std{0.0041} & 0.8073\,\std{0.0036} & 0.1363\,\std{0.0058} \\
PerturbDiff & 0.1995\,\std{0.0076} & 0.1995\,\std{0.0076} & 0.1176\,\std{0.0054} & \textbf{0.8215}\,\std{0.0039} & 0.1701\,\std{0.0080} \\
STATE & 0.1929\,\std{0.0083} & 0.1929\,\std{0.0083} & 0.1149\,\std{0.0060} & 0.7200\,\std{0.0060} & 0.1724\,\std{0.0089} \\
\midrule
scGen & 0.1993\,\std{0.0074} & 0.1993\,\std{0.0074} & 0.1170\,\std{0.0052} & 0.8189\,\std{0.0035} & 0.1750\,\std{0.0077} \\
CPA & 0.1392\,\std{0.0046} & 0.1392\,\std{0.0046} & 0.0769\,\std{0.0028} & 0.7787\,\std{0.0039} & 0.1124\,\std{0.0036} \\
stVAE & 0.2208\,\std{0.0092} & 0.2208\,\std{0.0092} & 0.1347\,\std{0.0069} & 0.8173\,\std{0.0044} & 0.2077\,\std{0.0105} \\
PerturbedVAE & 0.1709\,\std{0.0073} & 0.1709\,\std{0.0073} & 0.0994\,\std{0.0050} & 0.8155\,\std{0.0037} & 0.1562\,\std{0.0076} \\
\midrule
PerturbEnergy-Lang & \textbf{0.2473}\,\std{0.0090} & \textbf{0.2473}\,\std{0.0090} & \textbf{0.1515}\,\std{0.0068} & \underline{0.8205}\,\std{0.0044} & \textbf{0.2216}\,\std{0.0105} \\
PerturbEnergy-Mis & \underline{0.2366}\,\std{0.0093} & \underline{0.2366}\,\std{0.0093} & \underline{0.1451}\,\std{0.0070} & 0.8148\,\std{0.0044} & \underline{0.2142}\,\std{0.0107} \\
\bottomrule
\end{tabular}}
\label{app_exp_tab:hvg_pbmc_deg}
\end{table}

\begin{table}[h!]
\centering
\small
\caption{HVG-group distribution-level distances on PBMC, over all 2{,}000 highly variable
  genes. Values are the mean over held-out perturbation groups $\pm$ the s.e.m.\ across
  those groups. Best in \textbf{bold}, second best \underline{underlined}.}
\label{tab:hvg-pbmc-dist}
\setlength{\tabcolsep}{4pt}
\resizebox{\textwidth}{!}{%
\begin{tabular}{lccccc}
\toprule
Method & RMSE\,$\downarrow$ & MSE\,$\downarrow$ & MAE\,$\downarrow$ & EDist\,$\downarrow$ & Sinkhorn\,$\downarrow$ \\
\midrule
Linear & 0.0376\,\std{0.0009} & 0.0016\,\std{0.0001} & 0.0195\,\std{0.0004} & 0.2725\,\std{0.0137} & 0.4927\,\std{0.0113} \\
Linear (scGPT) & 0.0299\,\std{0.0010} & 0.0012\,\std{0.0001} & 0.0167\,\std{0.0004} & \underline{0.2035}\,\std{0.0137} & 0.4902\,\std{0.0111} \\
\midrule
CellFlow & 0.0338\,\std{0.0014} & 0.0017\,\std{0.0002} & 0.0159\,\std{0.0006} & 1.0656\,\std{0.0340} & 0.2766\,\std{0.0072} \\
Squidiff & 0.0529\,\std{0.0013} & 0.0033\,\std{0.0002} & 0.0217\,\std{0.0006} & 1.6522\,\std{0.0339} & 0.2937\,\std{0.0077} \\
PerturbDiff & 0.0340\,\std{0.0008} & 0.0014\,\std{0.0001} & 0.0146\,\std{0.0004} & 2.4800\,\std{0.0287} & 0.2905\,\std{0.0071} \\
STATE & 0.0264\,\std{0.0009} & 0.0009\,\std{0.0001} & 0.0138\,\std{0.0005} & \textbf{0.1658}\,\std{0.0112} & 0.4879\,\std{0.0112} \\
\midrule
scGen & 0.0275\,\std{0.0010} & 0.0011\,\std{0.0001} & \textbf{0.0123}\,\std{0.0005} & 0.8381\,\std{0.0238} & 0.2793\,\std{0.0072} \\
CPA & 0.0545\,\std{0.0012} & 0.0034\,\std{0.0002} & 0.0202\,\std{0.0006} & 4.0187\,\std{0.0435} & 0.2913\,\std{0.0074} \\
stVAE & 0.0251\,\std{0.0006} & \underline{0.0007}\,\std{0.0000} & \underline{0.0125}\,\std{0.0003} & 1.8705\,\std{0.0282} & \textbf{0.2630}\,\std{0.0063} \\
PerturbedVAE & 0.0443\,\std{0.0008} & 0.0022\,\std{0.0001} & 0.0183\,\std{0.0004} & 2.0305\,\std{0.0369} & 0.2886\,\std{0.0075} \\
\midrule
PerturbEnergy-Lang & \textbf{0.0239}\,\std{0.0006} & \underline{0.0007}\,\std{0.0000} & 0.0127\,\std{0.0003} & 2.6597\,\std{0.0327} & \underline{0.2698}\,\std{0.0067} \\
PerturbEnergy-Mis & \underline{0.0250}\,\std{0.0005} & \textbf{0.0007}\,\std{0.0000} & 0.0132\,\std{0.0003} & 2.5102\,\std{0.0304} & 0.2701\,\std{0.0066} \\
\bottomrule
\end{tabular}}
\label{app_exp_tab:hvg_pbmc_dis}
\end{table}

\subsubsubsection{Results on Norman2019 Dataset}

The HVG group results on Norman2019 is shown in Table~\ref{app_exp_tab:hvg_norman_corr}, Table~\ref{app_exp_tab:hvg_norman_deg} and Table~\ref{app_exp_tab:hvg_norman_dis}. On this dataset we find a relative different pattern from Replogle and PBMC where diffusion-based methods like PerturbDiff and CellFlow ahieve the best performance on correlation or distance based measures (Table~\ref{app_exp_tab:hvg_norman_corr} and Table~\ref{app_exp_tab:hvg_norman_dis}). Norman2019 generally contains duel-genetic perturbations and this diffusion process is good at capturing the complex pattern co-influenced by duel-genetic perturbations. Moreover, increasing perturbation information remains effective as we find that in DE related measures PerturbEnergy remains the best compared to other baselines (Table~\ref{app_exp_tab:hvg_norman_deg}).

\begin{table}[h!]
\centering
\small
\caption{HVG-group correlation-based measures on Norman2019, over all 2{,}000 highly
  variable genes. Values are the mean over held-out perturbation groups $\pm$ the
  s.e.m.\ across those groups; the perturbation-discrimination scores are computed once
  over the whole test set and carry no such spread. Best in \textbf{bold}, second best
  \underline{underlined}.}
\label{tab:hvg-norman-corr}
\setlength{\tabcolsep}{4pt}
\resizebox{\textwidth}{!}{%
\begin{tabular}{lccccc}
\toprule
Method & PCC\,$\uparrow$ & $R^2$\,$\uparrow$ & PDCorr\,$\uparrow$ & LFCSpear\,$\uparrow$ & LogFC-cos\,$\uparrow$ \\
\midrule
Linear & 0.9837\,\std{0.0022} & 0.9674\,\std{0.0043} & 0.4668\,\std{0.0426} & 0.2613\,\std{0.0231} & 0.4710\,\std{0.0422} \\
Linear (scGPT) & \underline{0.9971}\,\std{0.0003} & \underline{0.9941}\,\std{0.0007} & 0.9051\,\std{0.0160} & 0.3811\,\std{0.0145} & 0.9055\,\std{0.0160} \\
\midrule
CellFlow & 0.9772\,\std{0.0027} & 0.9545\,\std{0.0053} & \textbf{0.9137}\,\std{0.0270} & 0.3883\,\std{0.0141} & \textbf{0.9141}\,\std{0.0146} \\
Squidiff & 0.9921\,\std{0.0006} & 0.9839\,\std{0.0012} & 0.8296\,\std{0.0220} & \underline{0.4656}\,\std{0.0150} & 0.8307\,\std{0.0219} \\
PerturbDiff & \textbf{0.9975}\,\std{0.0003} & \textbf{0.9944}\,\std{0.0006} & 0.8985\,\std{0.0205} & \textbf{0.5092}\,\std{0.0132} & 0.8984\,\std{0.0205} \\
STATE & 0.9969\,\std{0.0004} & 0.9936\,\std{0.0008} & 0.9078\,\std{0.0153} & 0.3700\,\std{0.0164} & 0.9078\,\std{0.0153} \\
\midrule
scGen & 0.9957\,\std{0.0004} & 0.9911\,\std{0.0010} & 0.8800\,\std{0.0155} & 0.4622\,\std{0.0117} & 0.8802\,\std{0.0154} \\
CPA & 0.9818\,\std{0.0023} & 0.9636\,\std{0.0045} & 0.3727\,\std{0.0216} & 0.2036\,\std{0.0122} & 0.3769\,\std{0.0216} \\
stVAE & 0.9948\,\std{0.0007} & 0.9894\,\std{0.0013} & 0.8381\,\std{0.0253} & 0.3913\,\std{0.0122} & 0.8381\,\std{0.0254} \\
PerturbedVAE & 0.9964\,\std{0.0006} & 0.9926\,\std{0.0012} & 0.8980\,\std{0.0134} & 0.3736\,\std{0.0098} & 0.8984\,\std{0.0134} \\
\midrule
PerturbEnergy-Lang & 0.9970\,\std{0.0004} & 0.9936\,\std{0.0007} & 0.9063\,\std{0.0150} & 0.4280\,\std{0.0123} & 0.9066\,\std{0.0150} \\
PerturbEnergy-Mis & 0.9971\,\std{0.0003} & 0.9939\,\std{0.0007} & \underline{0.9112}\,\std{0.0141} & 0.4350\,\std{0.0122} & \underline{0.9116}\,\std{0.0141} \\
\bottomrule
\end{tabular}}

\vspace{6pt}

\resizebox{\textwidth}{!}{%
\begin{tabular}{lcccc}
\toprule
Method & DirAgr\,$\uparrow$ & PDS-cos\,$\downarrow$ & PDS-$\ell_1$\,$\downarrow$ & PDS-$\ell_2$\,$\downarrow$ \\
\midrule
Linear & 0.6395\,\std{0.0120} & 0.5000 & 0.5000 & 0.5000 \\
Linear (scGPT) & 0.6759\,\std{0.0075} & \textbf{0.0027} & \underline{0.0054} & \underline{0.0034} \\
\midrule
CellFlow & 0.6832\,\std{0.0072} & 0.0040 & 0.0297 & 0.0216 \\
Squidiff & \underline{0.7509}\,\std{0.0078} & 0.0061 & 0.0074 & 0.0054 \\
PerturbDiff & \textbf{0.7630}\,\std{0.0064} & 0.0202 & \textbf{0.0047} & \textbf{0.0027} \\
STATE & 0.6693\,\std{0.0090} & 0.0047 & 0.0142 & 0.0101 \\
\midrule
scGen & 0.7342\,\std{0.0061} & 0.0088 & 0.0202 & 0.0094 \\
CPA & 0.6043\,\std{0.0061} & 0.4784 & 0.4946 & 0.4926 \\
stVAE & 0.6925\,\std{0.0059} & 0.0263 & 0.0897 & 0.0830 \\
PerturbedVAE & 0.6800\,\std{0.0050} & 0.0040 & 0.0155 & 0.0115 \\
\midrule
PerturbEnergy-Lang & 0.7115\,\std{0.0061} & 0.0040 & 0.0108 & 0.0094 \\
PerturbEnergy-Mis & 0.7140\,\std{0.0059} & \underline{0.0027} & 0.0088 & 0.0081 \\
\bottomrule
\end{tabular}}
\label{app_exp_tab:hvg_norman_corr}
\end{table}

\begin{table}[h!]
\centering
\small
\caption{HVG-group DEG-related scores on Norman2019, over all 2{,}000 highly variable
  genes. Values are the mean over held-out perturbation groups $\pm$ the s.e.m.\ across
  those groups. Best in \textbf{bold}, second best \underline{underlined}.}
\label{tab:hvg-norman-deg}
\setlength{\tabcolsep}{4pt}
\resizebox{\textwidth}{!}{%
\begin{tabular}{lccccc}
\toprule
Method & DEPrec\,$\uparrow$ & DERec\,$\uparrow$ & DEOverlap\,$\uparrow$ & AUROC\,$\uparrow$ & AUPRC\,$\uparrow$ \\
\midrule
Linear & 0.3251\,\std{0.0119} & 0.3251\,\std{0.0119} & 0.1964\,\std{0.0086} & 0.9423\,\std{0.0032} & 0.3358\,\std{0.0161} \\
Linear (scGPT) & 0.6528\,\std{0.0203} & 0.6528\,\std{0.0203} & 0.4960\,\std{0.0201} & 0.9777\,\std{0.0053} & \underline{0.7284}\,\std{0.0251} \\
\midrule
CellFlow & 0.2190\,\std{0.0097} & 0.2190\,\std{0.0097} & 0.1242\,\std{0.0062} & 0.8502\,\std{0.0071} & 0.2118\,\std{0.0089} \\
Squidiff & 0.5682\,\std{0.0213} & 0.5682\,\std{0.0213} & 0.4079\,\std{0.0197} & 0.9735\,\std{0.0035} & 0.6335\,\std{0.0252} \\
PerturbDiff & 0.6467\,\std{0.0217} & 0.6467\,\std{0.0217} & 0.4909\,\std{0.0216} & \textbf{0.9809}\,\std{0.0037} & 0.7091\,\std{0.0268} \\
STATE & 0.6523\,\std{0.0210} & 0.6523\,\std{0.0210} & 0.4963\,\std{0.0209} & 0.9757\,\std{0.0064} & 0.7227\,\std{0.0260} \\
\midrule
scGen & 0.5969\,\std{0.0213} & 0.5969\,\std{0.0213} & 0.4369\,\std{0.0200} & 0.9732\,\std{0.0046} & 0.6618\,\std{0.0266} \\
CPA & 0.2759\,\std{0.0135} & 0.2759\,\std{0.0135} & 0.1628\,\std{0.0094} & 0.8958\,\std{0.0055} & 0.2699\,\std{0.0131} \\
stVAE & 0.5774\,\std{0.0228} & 0.5774\,\std{0.0228} & 0.4184\,\std{0.0207} & 0.9691\,\std{0.0051} & 0.6296\,\std{0.0277} \\
PerturbedVAE & 0.6513\,\std{0.0203} & 0.6513\,\std{0.0203} & 0.4942\,\std{0.0199} & 0.9747\,\std{0.0058} & 0.7118\,\std{0.0254} \\
\midrule
PerturbEnergy-Lang & \underline{0.6554}\,\std{0.0202} & \underline{0.6554}\,\std{0.0202} & \underline{0.4988}\,\std{0.0201} & 0.9776\,\std{0.0053} & 0.7215\,\std{0.0245} \\
PerturbEnergy-Mis & \textbf{0.6636}\,\std{0.0203} & \textbf{0.6636}\,\std{0.0203} & \textbf{0.5083}\,\std{0.0205} &\underline{0.9804}\,\std{0.0043} & \textbf{0.7309}\,\std{0.0234} \\
\bottomrule
\end{tabular}}
\label{app_exp_tab:hvg_norman_deg}
\end{table}

\begin{table}[h!]
\centering
\small
\caption{HVG-group distribution-level distances on Norman2019, over all 2{,}000 highly
  variable genes. Values are the mean over held-out perturbation groups $\pm$ the
  s.e.m.\ across those groups. Best in \textbf{bold}, second best \underline{underlined}.}
\label{tab:hvg-norman-dist}
\setlength{\tabcolsep}{4pt}
\resizebox{\textwidth}{!}{%
\begin{tabular}{lccccc}
\toprule
Method & RMSE\,$\downarrow$ & MSE\,$\downarrow$ & MAE\,$\downarrow$ & EDist\,$\downarrow$ & Sinkhorn\,$\downarrow$ \\
\midrule
Linear & 0.0762\,\std{0.0048} & 0.0067\,\std{0.0009} & 0.0203\,\std{0.0014} & 1.2427\,\std{0.1400} & 1.3307\,\std{0.0529} \\
Linear (scGPT) & \underline{0.0325}\,\std{0.0019} & \underline{0.0012}\,\std{0.0001} & 0.0132\,\std{0.0006} & 0.5972\,\std{0.0622} & 1.1605\,\std{0.0364} \\
\midrule
CellFlow & 0.0352\,\std{0.0029} & 0.0016\,\std{0.0003} & \underline{0.0113}\,\std{0.0010} & \textbf{0.2976}\,\std{0.0452} & 1.2613\,\std{0.0464} \\
Squidiff & 0.0551\,\std{0.0022} & 0.0032\,\std{0.0002} & 0.0174\,\std{0.0007} & 0.9530\,\std{0.0515} & 1.0328\,\std{0.0339} \\
PerturbDiff & \textbf{0.0318}\,\std{0.0017} & \textbf{0.0011}\,\std{0.0001} & \textbf{0.0105}\,\std{0.0006} & 2.8086\,\std{0.1059} & \textbf{0.8198}\,\std{0.0293} \\
STATE & 0.0337\,\std{0.0021} & 0.0013\,\std{0.0002} & 0.0127\,\std{0.0007} & \underline{0.3919}\,\std{0.0490} & 1.2014\,\std{0.0364} \\
\midrule
scGen & 0.0400\,\std{0.0022} & 0.0018\,\std{0.0002} & 0.0129\,\std{0.0008} & 0.9353\,\std{0.0814} & 0.9480\,\std{0.0348} \\
CPA & 0.0803\,\std{0.0051} & 0.0074\,\std{0.0010} & 0.0229\,\std{0.0015} & 5.0856\,\std{0.2118} & 0.9943\,\std{0.0549} \\
stVAE & 0.0434\,\std{0.0026} & 0.0021\,\std{0.0003} & 0.0142\,\std{0.0008} & 2.8448\,\std{0.1561} & \underline{0.8282}\,\std{0.0415} \\
PerturbedVAE & 0.0356\,\std{0.0025} & 0.0015\,\std{0.0003} & 0.0148\,\std{0.0009} & 3.5358\,\std{0.1081} & 0.8481\,\std{0.0394} \\
\midrule
PerturbEnergy-Lang & 0.0339\,\std{0.0019} & 0.0013\,\std{0.0002} & 0.0126\,\std{0.0006} & 3.6007\,\std{0.1238} & 0.8649\,\std{0.0376} \\
PerturbEnergy-Mis & 0.0330\,\std{0.0019} & 0.0012\,\std{0.0001} & 0.0123\,\std{0.0006} & 4.1930\,\std{0.1199} & 0.8724\,\std{0.0384} \\
\bottomrule
\end{tabular}}
\label{app_exp_tab:hvg_norman_dis}
\end{table}

\subsubsubsection{Results on Tahoe100M Dataset}

The HVG group results on Tahoe100M is shown in Table~\ref{app_exp_tab:hvg_tahoe_corr}, Table~\ref{app_exp_tab:hvg_tahoe_deg} and Table~\ref{app_exp_tab:hvg_tahoe_dis}. The results show general the same pattern as in Replogle and PBMC.

\begin{table}[h!]
\centering
\small
\caption{HVG-group correlation-based measures on Tahoe100M, over all 2{,}000 highly
  variable genes. Values are the mean over held-out perturbation groups $\pm$ the
  s.e.m.\ across those groups; the perturbation-discrimination scores are computed once
  over the whole test set and carry no such spread. Best in \textbf{bold}, second best
  \underline{underlined}.}
\label{tab:hvg-tahoe-corr}
\setlength{\tabcolsep}{4pt}
\resizebox{\textwidth}{!}{%
\begin{tabular}{lccccc}
\toprule
Method & PCC\,$\uparrow$ & $R^2$\,$\uparrow$ & PDCorr\,$\uparrow$ & LFCSpear\,$\uparrow$ & LogFC-cos\,$\uparrow$ \\
\midrule
Linear & 0.9761\,\std{0.0006} & 0.9451\,\std{0.0018} & 0.3223\,\std{0.0057} & 0.1999\,\std{0.0040} & 0.3344\,\std{0.0065} \\
Linear (scGPT) & 0.9783\,\std{0.0005} & 0.9532\,\std{0.0012} & 0.4708\,\std{0.0076} & 0.2459\,\std{0.0038} & 0.4803\,\std{0.0076} \\
\midrule
CellFlow & 0.9804\,\std{0.0005} & 0.9583\,\std{0.0012} & 0.5560\,\std{0.0055} & 0.2803\,\std{0.0033} & 0.5630\,\std{0.0056} \\
Squidiff & 0.9343\,\std{0.0009} & 0.8561\,\std{0.0023} & 0.3047\,\std{0.0050} & 0.1878\,\std{0.0027} & 0.3067\,\std{0.0049} \\
PerturbDiff & 0.9798\,\std{0.0007} & 0.9529\,\std{0.0013} & 0.5307\,\std{0.0067} & 0.3664\,\std{0.0033} & 0.5192\,\std{0.0074} \\
STATE & 0.9815\,\std{0.0004} & \textbf{0.9613}\,\std{0.0008} & 0.5694\,\std{0.0063} & 0.2884\,\std{0.0037} & 0.5755\,\std{0.0065} \\
\midrule
scGen & 0.9726\,\std{0.0007} & 0.9380\,\std{0.0021} & 0.2986\,\std{0.0042} & 0.2375\,\std{0.0030} & 0.3022\,\std{0.0043} \\
CPA & 0.8190\,\std{0.0017} & 0.6587\,\std{0.0028} & 0.2654\,\std{0.0096} & 0.2175\,\std{0.0042} & 0.2588\,\std{0.0103} \\
stVAE & 0.9794\,\std{0.0005} & 0.9564\,\std{0.0010} & 0.5456\,\std{0.0056} & 0.3420\,\std{0.0028} & 0.5506\,\std{0.0059} \\
PerturbedVAE & 0.9707\,\std{0.0006} & 0.9341\,\std{0.0017} & 0.3686\,\std{0.0042} & 0.2777\,\std{0.0023} & 0.3720\,\std{0.0044} \\
\midrule
PerturbEnergy-Lang & \underline{0.9818}\,\std{0.0004} & 0.9605\,\std{0.0010} & \underline{0.5944}\,\std{0.0056} & \underline{0.3680}\,\std{0.0030} & \underline{0.6003}\,\std{0.0058} \\
PerturbEnergy-Mis & \textbf{0.9820}\,\std{0.0004} & \underline{0.9610}\,\std{0.0010} & \textbf{0.5974}\,\std{0.0055} & \textbf{0.3774}\,\std{0.0029} & \textbf{0.6031}\,\std{0.0057} \\
\bottomrule
\end{tabular}}

\vspace{6pt}

\resizebox{\textwidth}{!}{%
\begin{tabular}{lcccc}
\toprule
Method & DirAgr\,$\uparrow$ & PDS-cos\,$\downarrow$ & PDS-$\ell_1$\,$\downarrow$ & PDS-$\ell_2$\,$\downarrow$ \\
\midrule
Linear & 0.5783\,\std{0.0020} & 0.0730 & 0.2357 & 0.2619 \\
Linear (scGPT) & 0.5879\,\std{0.0017} & 0.0851 & 0.1053 & 0.1340 \\
\midrule
CellFlow & 0.5952\,\std{0.0015} & \underline{0.0404} & 0.1137 & 0.1306 \\
Squidiff & 0.5719\,\std{0.0014} & 0.1401 & 0.1434 & 0.1548 \\
PerturbDiff & \textbf{0.6568}\,\std{0.0015} & 0.1439 & 0.1699 & 0.1771 \\
STATE & 0.6036\,\std{0.0016} & \textbf{0.0380} & \textbf{0.0378} & \textbf{0.0389} \\
\midrule
scGen & 0.6021\,\std{0.0013} & 0.1002 & 0.2124 & 0.2616 \\
CPA & 0.6106\,\std{0.0016} & 0.2641 & 0.2268 & 0.2642 \\
stVAE & 0.6354\,\std{0.0011} & 0.0591 & 0.0831 & 0.0816 \\
PerturbedVAE & 0.6141\,\std{0.0010} & 0.0887 & 0.1507 & 0.1798 \\
\midrule
PerturbEnergy-Lang & 0.6457\,\std{0.0013} & 0.0454 & \underline{0.0536} & \underline{0.0587} \\
PerturbEnergy-Mis & \underline{0.6485}\,\std{0.0013} & 0.0469 & 0.0540 & 0.0608 \\
\bottomrule
\end{tabular}}
\label{app_exp_tab:hvg_tahoe_corr}
\end{table}

\begin{table}[h!]
\centering
\small
\caption{HVG-group DEG-related scores on Tahoe100M, over all 2{,}000 highly variable
  genes. Values are the mean over held-out perturbation groups $\pm$ the s.e.m.\ across
  those groups. Best in \textbf{bold}, second best \underline{underlined}.}
\label{tab:hvg-tahoe-deg}
\setlength{\tabcolsep}{4pt}
\resizebox{\textwidth}{!}{%
\begin{tabular}{lccccc}
\toprule
Method & DEPrec\,$\uparrow$ & DERec\,$\uparrow$ & DEOverlap\,$\uparrow$ & AUROC\,$\uparrow$ & AUPRC\,$\uparrow$ \\
\midrule
Linear & 0.1318\,\std{0.0038} & 0.1318\,\std{0.0038} & 0.0720\,\std{0.0023} & 0.7426\,\std{0.0034} & 0.1015\,\std{0.0030} \\
Linear (scGPT) & 0.2070\,\std{0.0059} & 0.2070\,\std{0.0059} & 0.1192\,\std{0.0039} & 0.7713\,\std{0.0040} & 0.1730\,\std{0.0057} \\
\midrule
CellFlow & 0.2339\,\std{0.0056} & 0.2339\,\std{0.0056} & 0.1361\,\std{0.0038} & 0.7950\,\std{0.0036} & 0.1972\,\std{0.0059} \\
Squidiff & 0.1341\,\std{0.0035} & 0.1341\,\std{0.0035} & 0.0731\,\std{0.0021} & 0.7502\,\std{0.0029} & 0.0954\,\std{0.0026} \\
PerturbDiff & 0.2054\,\std{0.0054} & 0.2054\,\std{0.0054} & 0.1176\,\std{0.0035} & 0.7751\,\std{0.0039} & 0.1659\,\std{0.0051} \\
STATE & 0.2407\,\std{0.0061} & 0.2407\,\std{0.0061} & \underline{0.1411}\,\std{0.0042} & 0.7951\,\std{0.0039} & 0.2042\,\std{0.0065} \\
\midrule
scGen & 0.1311\,\std{0.0028} & 0.1311\,\std{0.0028} & 0.0709\,\std{0.0016} & 0.7374\,\std{0.0032} & 0.0922\,\std{0.0019} \\
CPA & 0.1353\,\std{0.0044} & 0.1353\,\std{0.0044} & 0.0744\,\std{0.0026} & 0.7500\,\std{0.0039} & 0.1094\,\std{0.0030} \\
stVAE & 0.2209\,\std{0.0053} & 0.2209\,\std{0.0053} & 0.1274\,\std{0.0036} & 0.7790\,\std{0.0040} & 0.1786\,\std{0.0057} \\
PerturbedVAE & 0.1513\,\std{0.0039} & 0.1513\,\std{0.0039} & 0.0833\,\std{0.0024} & 0.7571\,\std{0.0036} & 0.1087\,\std{0.0029} \\
\midrule
PerturbEnergy-Lang & \underline{0.2408}\,\std{0.0057} & \underline{0.2408}\,\std{0.0057} & 0.1406\,\std{0.0039} & \underline{0.7952}\,\std{0.0040} & \underline{0.2050}\,\std{0.0060} \\
PerturbEnergy-Mis & \textbf{0.2452}\,\std{0.0058} & \textbf{0.2452}\,\std{0.0058} & \textbf{0.1437}\,\std{0.0040} & \textbf{0.8004}\,\std{0.0039} & \textbf{0.2072}\,\std{0.0061} \\
\bottomrule
\end{tabular}}
\label{app_exp_tab:hvg_tahoe_deg}
\end{table}

\begin{table}[h!]
\centering
\small
\caption{HVG-group distribution-level distances on Tahoe100M, over all 2{,}000 highly
  variable genes. Values are the mean over held-out perturbation groups $\pm$ the
  s.e.m.\ across those groups. Best in \textbf{bold}, second best \underline{underlined}.}
\label{tab:hvg-tahoe-dist}
\setlength{\tabcolsep}{4pt}
\resizebox{\textwidth}{!}{%
\begin{tabular}{lccccc}
\toprule
Method & RMSE\,$\downarrow$ & MSE\,$\downarrow$ & MAE\,$\downarrow$ & EDist\,$\downarrow$ & Sinkhorn\,$\downarrow$ \\
\midrule
Linear & 0.1243\,\std{0.0012} & 0.0159\,\std{0.0004} & 0.0632\,\std{0.0005} & 0.9646\,\std{0.0147} & 0.7265\,\std{0.0008} \\
Linear (scGPT) & 0.1166\,\std{0.0010} & 0.0139\,\std{0.0003} & 0.0629\,\std{0.0004} & 0.9021\,\std{0.0115} & 0.7206\,\std{0.0008} \\
\midrule
CellFlow & 0.1093\,\std{0.0010} & 0.0122\,\std{0.0003} & \textbf{0.0560}\,\std{0.0004} & \underline{0.8195}\,\std{0.0103} & 0.7319\,\std{0.0013} \\
Squidiff & 0.2058\,\std{0.0010} & 0.0427\,\std{0.0004} & 0.1101\,\std{0.0005} & 6.7723\,\std{0.0579} & 0.4257\,\std{0.0011} \\
PerturbDiff & 0.1166\,\std{0.0012} & 0.0140\,\std{0.0003} & 0.0581\,\std{0.0005} & 21.6085\,\std{0.0591} & 0.3951\,\std{0.0008} \\
STATE & \underline{0.1062}\,\std{0.0007} & \textbf{0.0114}\,\std{0.0002} & \underline{0.0560}\,\std{0.0004} & \textbf{0.7920}\,\std{0.0073} & 0.7239\,\std{0.0008} \\
\midrule
scGen & 0.1324\,\std{0.0013} & 0.0181\,\std{0.0004} & 0.0677\,\std{0.0005} & 3.7313\,\std{0.0280} & 0.4343\,\std{0.0010} \\
CPA & 0.3211\,\std{0.0021} & 0.1044\,\std{0.0013} & 0.1305\,\std{0.0007} & 20.7955\,\std{0.0703} & 0.4451\,\std{0.0010} \\
stVAE & 0.1126\,\std{0.0009} & 0.0129\,\std{0.0002} & 0.0613\,\std{0.0004} & 14.2524\,\std{0.0406} & \underline{0.3899}\,\std{0.0007} \\
PerturbedVAE & 0.1372\,\std{0.0011} & 0.0192\,\std{0.0003} & 0.0704\,\std{0.0004} & 7.2237\,\std{0.0251} & \textbf{0.3887}\,\std{0.0007} \\
\midrule
PerturbEnergy-Lang & 0.1067\,\std{0.0009} & 0.0116\,\std{0.0002} & 0.0587\,\std{0.0004} & 21.1279\,\std{0.0236} & 0.3943\,\std{0.0007} \\
PerturbEnergy-Mis & \textbf{0.1060}\,\std{0.0009} & \underline{0.0115}\,\std{0.0002} & 0.0579\,\std{0.0004} & 21.2050\,\std{0.0241} & 0.3942\,\std{0.0007} \\
\bottomrule
\end{tabular}}
\label{app_exp_tab:hvg_tahoe_dis}
\end{table}

\end{document}